%% file: main.tex
\documentclass{article}

\usepackage{amsmath}
\input{math_commands}

\usepackage{iclr2024_conference,times}

\usepackage[utf8]{inputenc} %
\usepackage[T1]{fontenc}    %
\usepackage{hyperref}       %
\usepackage{url}            %
\usepackage{booktabs}       %
\usepackage{amsfonts}       %
\usepackage{nicefrac}       %
\usepackage{microtype}      %
\usepackage{multirow}
\usepackage{amsthm}
\usepackage{wrapfig}
\usepackage{cleveref}
\newtheorem{assumption}{Assumption}
\Crefname{assumption}{Assumption}{Assumptions}
\definecolor{mypurple}{RGB}{128, 0, 128}
\definecolor{microsoftcyan}{RGB}{48, 192, 180}
\definecolor{darkyellow}{rgb}{0.85, 0.65, 0.13}
\iclrfinalcopy

\title{\center{
Understanding Warmup-Stable-Decay Learning Rates: A River Valley Loss Landscape Perspective}}

\author{Kaiyue Wen \\
Stanford University \\
\texttt{kaiyuew@stanford.edu}
\And
Zhiyuan Li \\
Toyota Technological Institute at Chicago \\
\texttt{zhiyuanli@ttic.edu}\\
\And
Jason  Wang \\
Stanford University \\
\texttt{jsywang@stanford.edu}
\And
David Hall \\
Stanford University \\
\texttt{dlwh@cs.stanford.edu}
\AND
Percy Liang \\
Stanford University \\
\texttt{pliang@cs.stanford.edu}
\And
Tengyu Ma \\
Stanford University \\
\texttt{tengyuma@cs.stanford.edu}
}

\newtheorem{theorem}{Theorem}[section]

\newtheorem{lemma}[theorem]{Lemma}

\newtheorem{definition}[theorem]{Definition}

\Crefname{lemma}{Lemma}{Lemmas}
\usepackage{graphicx}
\usepackage{subcaption}
\usepackage{enumitem}

\begin{document}

\maketitle
\begin{center}
\vspace{0.2in}
    \begin{minipage}{0.28\textwidth}
        \centering
        \textbf{Kaiyue Wen} \\
        Stanford University \\
        \texttt{kaiyuew@stanford.edu}
    \end{minipage}
    \hspace{0.05\textwidth} %
    \begin{minipage}{0.28\textwidth}
        \centering
        \textbf{Zhiyuan Li} \\
TTIC \\
\texttt{zhiyuanli@ttic.edu}
    \end{minipage}
    \hspace{0.05\textwidth}
    \begin{minipage}{0.28\textwidth}
        \centering
        \textbf{Jason Wang} \\
        Stanford University \\
        \texttt{jsywang@stanford.edu}
    \end{minipage} \\
    \vspace{0.5cm} %
    \begin{minipage}{0.28\textwidth}
        \centering
        \textbf{David Hall} \\
        Stanford University \\
        \texttt{dlwh@cs.stanford.edu}
    \end{minipage}
    \hspace{0.05\textwidth}
    \begin{minipage}{0.28\textwidth}
        \centering
        \textbf{Percy Liang} \\
        Stanford University \\
        \texttt{pliang@cs.stanford.edu}
    \end{minipage}
    \hspace{0.05\textwidth}
    \begin{minipage}{0.28\textwidth}
        \centering
        \textbf{Tengyu Ma} \\
        Stanford University \\
        \texttt{tengyuma@cs.stanford.edu}
    \end{minipage}
\end{center}
\vskip 0.3in minus 0.1in

\begin{abstract}

Training language models currently requires pre-determining a fixed compute budget because the typical cosine learning rate schedule depends on the total number of steps. In contrast, the Warmup-Stable-Decay ($\wsd$) schedule uses a constant learning rate to produce a main branch of iterates that can in principle continue indefinitely without a pre-specified compute budget. Then, given any compute budget, one can branch out from the main branch at a proper time with a rapidly decaying learning rate to produce a strong model.
Empirically, $\wsd$ generates an intriguing, non-traditional loss curve: the loss remains elevated during the stable phase but sharply declines during the decay phase. 
Towards explaining this phenomenon, we conjecture that pretraining loss exhibits a \emph{river valley landscape}, which resembles a deep valley with a river at its bottom. Under this assumption, we show that during the stable phase, the iterate undergoes large oscillations due to the high learning rate, yet it progresses swiftly along the river. During the decay phase, the rapidly dropping learning rate minimizes the iterate's oscillations, moving it closer to the river and revealing true optimization progress. Therefore, the sustained high learning rate phase and fast decaying phase are responsible for progress in the river and the mountain directions, respectively, and are both critical. Our analysis predicts phenomenons consistent with empirical observations and shows that this landscape can naturally emerge from pretraining on a simple bi-gram dataset.
Inspired by the theory, we introduce $\wsds$, a variant of $\wsd$ that reuses previous checkpoints' decay phases and keeps only one main branch, where we resume from a decayed checkpoint. $\wsds$ empirically outperforms $\wsd$ and $\cycliccosine$ in obtaining multiple pretrained language model checkpoints across various compute budgets in a single run for parameters scaling from 0.1B to 1.2B. 
\end{abstract}

\input{main/intro_new}
\input{main/related}

\input{main/theory}

\input{main/method}

\input{main/experiments}

\input{main/conclusion}

\newpage
\bibliographystyle{iclr2024_conference}
\bibliography{our}

\newpage
\appendix
\input{appendix/proof}
\input{appendix/experiments}

\end{document}

%% file: math_commands.tex
\newcommand{\eigen}[2]{\lambda_{#1}\left(#2\right)}
\newcommand{\eigenv}[2]{v_{#1}\left(#2\right)}

\newcommand{\R}{\mathbb{R}}
\newcommand{\one}{\mathbf{1}}
\newcommand{\mineta}{\eta_{\mathrm{min}}}
\newcommand{\maxeta}{\eta_{\mathrm{max}}}
\newcommand{\E}{\mathbb{E}}

\newcommand{\river}{\mathcal{M}}

\newcommand{\err}{\kappa}
\newcommand{\maxgamma}{\gamma_{\mathrm{max}}}
\newcommand{\flatgamma}{\gamma_{\mathrm{flat}}}
\newcommand{\normbound}{\Delta}
\newcommand{\noise}{{g}}
\newcommand{\normal}[2]{\mathcal{N} \left(#1, #2\right)}
\newcommand{\identity}{\mathcal{I}}
\newcommand{\sharpproj}{{P}_S}
\newcommand{\flatproj}{{P}_F}
\newcommand{\gradphi}{\nabla L(\phi(w,t))}
\newcommand{\sharpphi}{\sharpproj(\phi(w,t))}

\newcommand{\sharpgrad}{\sharpproj(w(t)) \nabla L(w(t)) }
\newcommand{\flatgrad}{\flatproj(w(t)) \nabla L(w(t)) }
\newcommand{\riverproj}{{P}_{\river}}
\newcommand{\sharpgradtau}{\sharpproj(w_{k,\tau}) \nabla L(w_{k,\tau})}
\newcommand{\flatgradtau}{\flatproj(w_{k,\tau}) \nabla L(w_{k,\tau})}
\newcommand{\prob}{ \mathbb{P}}
\newcommand{\Normal}{\mathcal{N}}
\newcommand{\maxT}{T_{\mathrm{max}}}

\newcommand{\convergeT}{T_{\mathrm{converge}}}
\newcommand{\cosineoracle}{\textit{Cosine-Oracle}}
\newcommand{\cycliccosine}{\textit{Cyclic-Cosine}}
\newcommand{\maintheta}{{\theta^{\mathrm{main}}}}
\newcommand{\wsd}{\textit{WSD}}
\newcommand{\wsds}{\textit{WSD-S}}

%% file: main/intro_new.tex
\section{Introduction}

Pre-training large language models (LLMs) typically involves following a learning rate schedule that decreases over a pre-determined number of steps, such as a cosine schedule \citep{loshchilov2017sgdr,touvron2023llama}, where the learning rate starts high and gradually decreases in a smooth curve following the shape of a cosine function. 
This inflexible approach makes it difficult to adapt to additional compute or data, as the learning rate schedule for all the data is not a natural continuation of the schedule used with past data. Additionally, fitting scaling laws is costly because each compute budget requires retraining to adjust the learning rate schedule \citep{hoffmann2022training}.

In contrast to the cosine learning rate, recent work~\cite{hu2024minicpm} introduces the warmup-stable-decay ($\wsd$) schedule, which does not require committing to a pre-specified total compute budget.
After a standard warm-up period, the $\wsd$ schedule maintains a main ``branch'' using  a constant learning rate indefinitely and branch off using a fast-decaying learning rate schedule to obtain intermediate checkpoints (see the second row of Figure~\ref{fig:learningrate}). Using the $\wsd$ schedule, one can continue training from a checkpoint in the main branch by resuming with the same constant learning rate and can obtain training losses for multiple compute budgets with a single run.

Empirically, the $\wsd$ schedule produces a non-traditional loss curve (see~\Cref{fig:losscurve}): during the constant learning rate phase, the loss remains higher than the loss using other schedules like the cosine schedule; but during the decay phase, it drops sharply, often leading to a better final performance compared to the cosine schedule. This raises the main question the paper aims to address: 

\begin{wrapfigure}{r}{0.4\textwidth}
    \centering
    \includegraphics[width=0.4\textwidth]{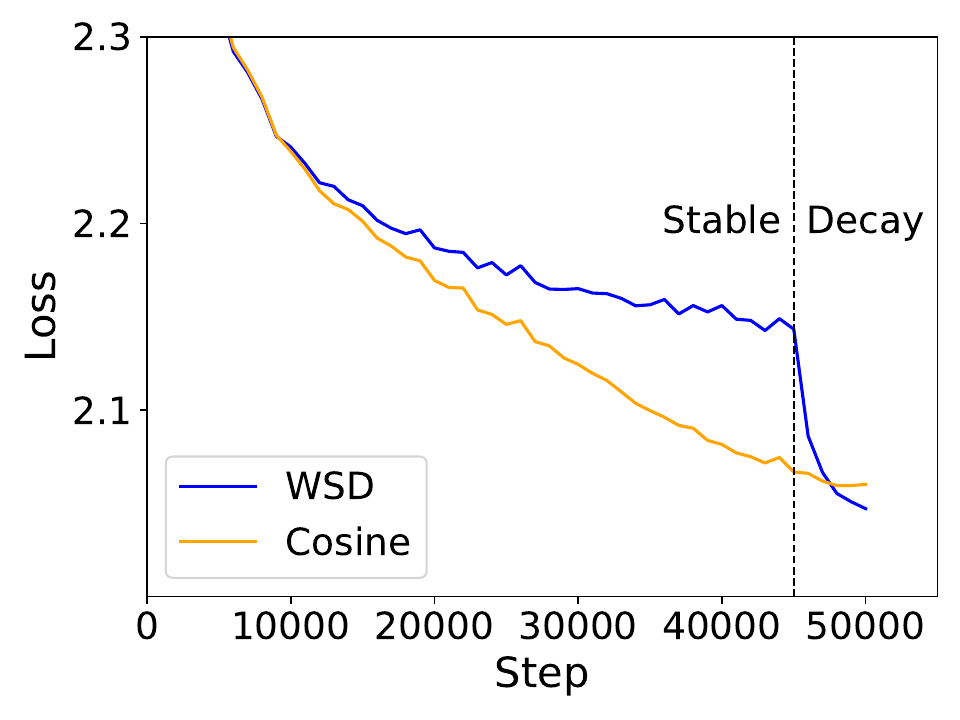}
    \caption{\textbf{The Distinctive Loss Curve produced by $\wsd$.} 
A constant learning rate phase, characterized by slow loss improvements, eventually leads to better validation loss after learning rate decay.}
    \label{fig:losscurve}
    \vspace{-0.3in}
\end{wrapfigure}

\emph{Why does $\wsd$ work, especially with such a non-traditional loss curve? Specifically, why does a constant learning rate phase, characterized by slow loss improvements, eventually lead to superior performance?}

\begin{figure}[t]
    \centering
    \begin{subfigure}[b]{0.36\textwidth}
        \centering\includegraphics[width=\textwidth]{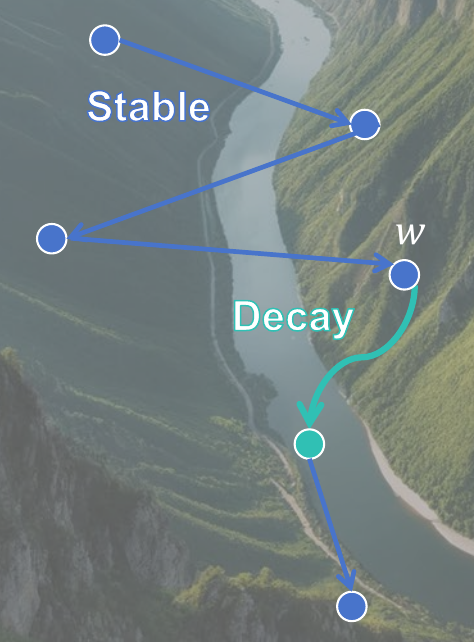}
        \caption{River Valley Landscape}
        \label{fig:landscape}
    \end{subfigure}
    \hfill
    \begin{subfigure}[b]{0.62\textwidth}
        \centering
        \includegraphics[width=\textwidth]{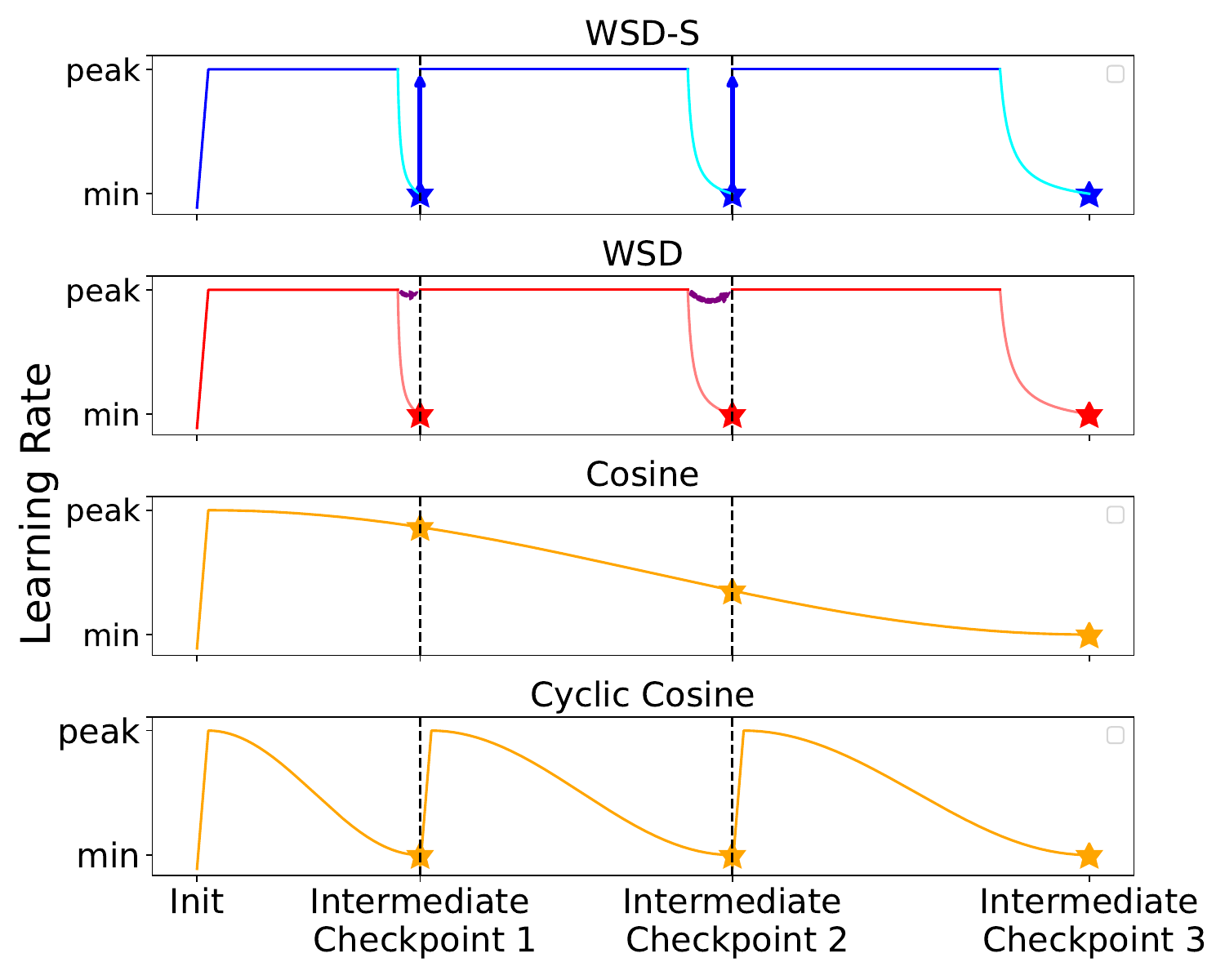}
        \caption{Learning rate schedules}
        \label{fig:learningrate}
    \end{subfigure}
    \caption{
        We theoretically analyze the Warmup-Stable-Decay ($\wsd$) schedule and demonstrate a \textbf{river valley} loss lanscape model to explain its effectiveness (demonstrated in~\Cref{fig:landscape}). The stable phase adopts a large learning rate and the iterate will progress along the river while oscillating between the sharp hillsides. Due to the large oscillation caused by the large learning rate, the run will potentially show a higher loss compared to a run using smaller learning rate in this phase. During the decay phase, the learning rate is dropped rapidly to ease the oscillation of the iterates, driving it closer to the river, revealing the optimization progress.
        Based on our theory, we propose $\wsds$implified ($\wsds$), an effective simplification of the $\wsd$ schedule in continual learning, where we start directly using a high learning rate from previous intermediate checkpoints. We visualize the learning rate schedule in~\Cref{fig:learningrate}.
        The {\color{mypurple}{arrow}} in the second row of~\Cref{fig:learningrate} indicates $\wsd$ reinitializes the checkpoint from the last checkpoint from the constant learning rate phase instead. 
    }
    \vspace{-0.2in}
    \label{fig:mot}
\end{figure}

The first contribution of this paper is a theoretical framework to explain the underlying mechanism of $\wsd$. We characterize a type of loss landscape, called {the} river valley landscape (\Cref{def:river-valley}), and theoretically show that $\wsd$ has superior performance on such loss landscapes. 
We show that the river valley landscape can provide multiple theoretical prediction matching the empirical observations and hence can serve as a useful conceptual picture for understanding the pretraining optimization process.

As the name suggests, a river valley landscape intuitively features steeply sloping hillsides with a river winding through the bottom of the gorge (see \Cref{fig:landscape}). During the stochastic gradient-based optimization process, the iterate bounces between the hillsides as it slowly and implicitly progresses along the river direction. The loss in this landscape can be decomposed into two components: the \emph{river component}, which represents the primary loss along the river at the bottom of the hills, and the \emph{hill component}, which accounts for the additional loss caused by deviations in height from the river's course. Progress is determined primarily by the river component in the long run.
We demonstrate that when the loss function exhibits this type of landscape, a learning rate schedule should satisfy the following two key properties to effectively minimize the loss.
\vspace{-0.1in}
\begin{enumerate}[leftmargin=*]
\setlength{\itemsep}{0pt}
    \item \textbf{Sustained high learning rate.} It is advantageous to maintain a large learning rate for as long as possible during training, even at the cost of less reduction in the loss. A large learning rate yields larger  bouncing due to the stochasticity of the gradient, increasing the hill component of the loss, but it also makes faster progress in the river direction. In contrast, a small learning rate results in less bouncing, keeping the iterate close to the river, but progress along the river direction is slower. Therefore, a larger learning rate leads to faster fundamental progress minimizing the river component, which is obscured by the oscillation in the hill component. This progress will be revealed by the decay phase discussed below.
    \item \textbf{Final low learning rate.} As training nears completion, it becomes essential to reduce the learning rate. This decay minimizes the oscillations in the mountain direction to decrease the hill component and ensures that the iterates converge to a point close to the river, which has a lower loss than any nearby points up the mountain.
\end{enumerate}
\vspace{-0.1in}
In Section~\ref{sec:optimization}, we provide formal theoretical statements analyzing the trajectories of (stochastic) gradient descent on the river valley landscape, fleshing out the intuitions above. 
Among our synthetic and real-world studies supporting the river valley landscape hypothesis, an intriguing observation in language model pretraining is that the loss on the linear interpolation of two checkpoints in the stable phase exhibits a convex and unimodal shape, resembling a valley, whereas between two checkpoints in the decay phase, the loss shows a smooth monotone decay.

All the theoretical results above assume a river valley landscape. How likely does the next-token prediction loss follow this pattern, and why?
We hypothesize the river valley landscape can naturally arise from the heterogeneity in the stochasticity of different tokens: highly deterministic tokens (which often involve facts and knowledge) contribute to the "river" direction, while uncertain tokens (which often involve flexibility and ambiguity in the language) create the steep hillsides. We demonstrate this insight by showing in~\Cref{sec:data} that under a bigram toy model, indeed the loss has a river valley landscape, and empirically, most properties of the loss curves under various learning rates on the real datasets are still seen in this toy model. We further show that the stable learning rate phase learns the deterministic tokens whereas the decay phase learns better the stochastic tokens.

Finally, we discuss the limitations of $\wsd$ within our framework and how to go beyond it. In $\wsd$, once an intermediate checkpoint is reached, the model and optimizer are rolled back to the end of the stable phase before resuming training with a constant learning rate. However, our theory suggests that the decay phase also contributes to progress along the river direction, so discarding this part of the progress is unnecessary. Moreover, when an organization pretrains an LLM with $\wsd$, it needs to release the last checkpoint in a stable-phase to facilitate continual pretraining for another party, adding extra complexity.

To overcome this limitation, we propose a simplified and improved version of $\wsd$, called \textbf{$\wsds$}, for continual learning. Specifically, $\wsds$ resumes training immediately from the intermediate checkpoint using a high constant learning rate, avoiding the need to roll back to a previous checkpoint.

We evaluate the effectiveness of $\wsds$ with extensive experiments on LLMs from 0.1B to 1.2B parameters in a continual learning setting with 50B, 100B, and 200B tokens as the three target compute budgets. We empirically show that $\wsds$ has performance comparable with independent oracle runs with cosine learning rate schedules optimally tuned on each of the three budgets. Furthermore, $\wsds$ leads to a better validation loss than $\wsd$ under the same compute budgets due to the re-use of the decay period. We also show through ablation studies that the performance is relatively insensitive to the precise fraction of time spent decaying as long as it is near $10\%$ and the decay does not start shortly after a coincidental loss spike.

%% file: main/related.tex
\vspace{-0.1in}
\section{Related Work}
\vspace{-0.1in}

\textbf{Learning Rate Schedules.} Learning rate schedules are crucial in deep learning, with previous studies exploring various options. \citet{smith2017cyclical} was the first to propose a cyclic triangular learning rate schedule that interleaves decreasing and increasing learning rates. \citet{loshchilov2017sgdr} extended the idea to a cyclic cosine learning rate schedule. \citet{he2015deep} introduced the notion of warmup, which gradually increases the learning rate in the earlier training phase. \citet{goyal2018accurate, hoffer2018train, you2020large} concluded that the learning rate should scale linearly with the batch size, which is further theoretically examined in~\citet{smith2020generalization, li2021validity, malladi2023sdes}. 
\citet{you2019does} performed an analysis on why learning rate schedules are helpful and suspected that the large learning rate at the beginning phase is mostly useful for avoiding memorization of noisy data, which is consistent with our analysis in~\Cref{sec:data}.

In the LLM era, works including~\citet{hoffmann2022training, deepseekai2024deepseek, hu2024minicpm} examined how to choose learning rate schedules for pretraining. In particular, \citet{hu2024minicpm} introduced a learning rate schedule called Warmup-Stable-Decay ($\wsd$) that remains constant for the majority of the runs before decaying in language model pretraining, which were studied independently in~\citet{zhai2022scaling, ibrahim2024simple, hagele2024scalinglawscomputeoptimaltraining}. \citet{raffel2023exploring, ibrahim2024simple} explored another possibility of using an inverse square root schedule to pretrain the language models. \citet{defazio2023whenmuchadaptivelearning} proposes to use linear decay for the entire training run. \citet{defazio2024roadscheduled} shows that with appropriate iterate averaging, a constant learning rate schedule can reach better performance than the cosine learning rate schedule. \citet{rae2022scaling, gupta2023continual, hu2024minicpm, ibrahim2024simple} examined how to choose a learning rate schedule in a continual learning setting and verified that rewarming-up cosine learning rate brings performance drops that are costly to recover. A common belief is that the performance drop is due to the sudden increase in learning rate during rewarming-up. However, our work shows that increasing the learning rate after a short decay in $\wsd$ does not cause a similar performance drop as seen with the cosine learning rate, challenging the previous hypothesis. Instead, we suggest that the performance loss associated with rewarming-up cosine learning rate is due to the implicit damage it causes to the model, making it unsuitable for continual training. On the contrast, $\wsd$ avoids such damage by maintaining a high learning rate during the stable phase, hence the sudden increase in learning rate does not lead to performance drops in continual training.

\textbf{Continual Learning.}  Continual learning, the process of updating the model with newly collected data, can improve the models' knowledge and capability.  Previous continual learning research \citep{aljundi2019online, veniat2021efficient, cossu2022continual, 51163, harun2023siesta, mehta2023empirical} assumed  significant domain shift and aimed to avoid forgetting old knowledge while learning new knowledge. 
Recent works including \citet{hernandez2021scaling, lesort2023challenging} suggested that optimizers including SGD and Adam have a knowledge accumulation effect and the effect of catastrophic forgetting may be less significant than expected, especially when replay is applied. Our work mainly focuses on continual pre-training without necessarily a strong domain shift and hence does not touch upon the effect of covariance shift. Continual learning is also extensively employed in large language models such as LLaMA to extend their capabilities, such as handling longer contexts (e.g., see \citet{tworkowski2023focused, peng2023yarnefficientcontextwindow, chen2023extendingcontextwindowlarge, dubey2024llama3herdmodels} and references therein) or dealing with new languages and domains (e.g., see \citet{azerbayev2024llemmaopenlanguagemodel, rozière2024codellamaopenfoundation, cui2024efficienteffectivetextencoding} and references therein).

\textbf{Theoretical Understanding on Loss Landscape.} A long line of research aims to better understand the loss landscape in deep learning (e.g., see \citet{freeman2017topologygeometryhalfrectifiednetwork, garipov2018losssurfacesmodeconnectivity, li2020explaining} and references therein). We will highlight several phenomena that are related to our findings. 

(1) Ill-conditioned directional sharpness and heavy-tailed noise: \citet{zhang2020gradient, NEURIPS2020_b05b57f6} examined the gradient noise in language modeling and observed that the noise is heavy-tailed in multiple dimensions. \citet{pan2023understanding, liu2024sophia} showed that the loss has vastly different curvatures in different dimensions. \citet{pan2022eigencurveoptimallearningrate} analyzes optimizing a quadratic function with skewed curvature theoretically.
Our river valley landscape is consistent with these findings. 

(2) Benefit of large learning rates: Large learning rates have a provable regularizing effect in finding flatter minima~\citep{kong2020stochasticity, wang2022large}, and flatter minima typically have a better generalization effect, even in the pretraining setting~\citep{jiang2019fantastic, blanc2020implicit, liu2022pretraining, li2022happens, ma2022quadratic, lyu2023understanding, andriushchenko2023sgdlargestepsizes}. 

(3) Connecting loss landscape with feature learning: Some recent works~\citep{nakkiran2019sgd, rosenfeld2023outliers} tried to understand how the loss landscape is formed through the lens of feature learning. \citet{rosenfeld2023outliers} showed that a large learning rate will cause oscillation in learning subtle classification rules while continuing to learn other more deterministic features.~\citet{wang2024improvinggeneralizationconvergenceenhancing} studied how to improve generalization and convergence by amplifying the update provided by the optimizer in the flat direction of the loss landscape. ~\citet{wu2024largestepsizegradientdescent, cai2024largestepsizegradientdescent} studied gradient descent dynamics on logistic regression, showing that a large learning rate will cause oscillation in the earlier phase but will lead to higher progress later in training. \citet{pagliardini2024ademamixoptimizerbetterfaster} developed a modification of the Adam optimizer based on optimization analysis on the Rosenbrock function, which is a special case of the river valley landscape. \citet{song2024doessgdreallyhappen} shows that when SGD update is projected to the dominant subspace of the Hessian, the model's optimization progress slows down and they conjecture the existence of \emph{ill-conditioned valley} in the landscape, which can be viewed as a similar and simpler version of the river valley landscape discussed in this paper. 

(4) Ravines in the Loss Landscape. Concurrently with our work, \citet{davis2024gradientdescentadaptivestepsize} identified the existence of a ravine in the loss landscape—a manifold where every point has a vanishing gradient within the sharp eigenspace of the Hessian. This feature appears in any smooth loss function exhibiting fourth-order growth near minimizers. They also demonstrate the advantages of using adaptive step sizes in this context. The concept of a ravine aligns closely with the river structure described in our paper and can be considered a specific instance of it.

%% file: main/theory.tex
\input{main/optimization_new}

\input{main/data}

%% file: main/optimization_new.tex
\section{Theoretical Analysis with River Valley Loss Landscapes}
\label{sec:optimization}

\begin{wrapfigure}{r}{0.32\textwidth}
    \centering
    \vspace{-0.2in}
    \includegraphics[width=0.3\textwidth]{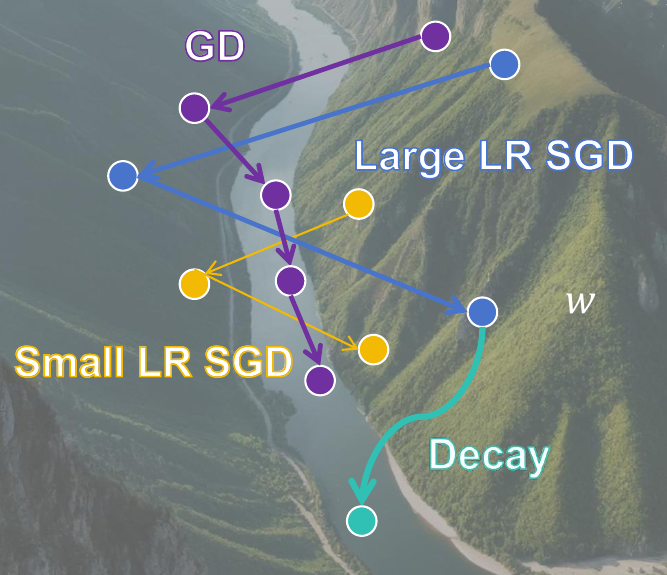}
    \caption{A River Valley Loss Landscape and the Optimization Dynamics with Various Learning Rates.}
    \label{fig:demo}
    \vspace{-0.2in}
\end{wrapfigure}

\textbf{Intuition and Outline.} As illustrated in Figure \ref{fig:demo}, we will analyze the dynamics of different optimization algorithm on the river valley loss landscape. (1) The \textcolor{purple}{gradient descent (GD)} iterates will quickly converge to the river and track the river afterward (\Cref{thm:gdtracksriver}). This holds for the continuous limit of GD, gradient flow, as well (\Cref{thm:gftracksriver}). (2) Stochastic gradient descent (SGD) algorithm on a river valley loss landscape will bounce between the hillsides due to the noise instead of converging to the river, while still gradually progressing along the river over time (\Cref{thm:sgd-main}). A 
large learning rate results in significant oscillations but also facilitates rapid advancement in the river's direction (the \textcolor{blue}{blue} trajectory). Conversely, a small learning rate keeps the iterates closer to the river's course but leads to slower movement along it (the \textcolor{darkyellow}{yellow} trajectory). While the oscillations in the mountain direction can cause a higher loss for a large learning rate in the short term, the large learning rate leads to faster fundamental progress along the river.
(3) Finally, reducing the learning rate before termination (the \textcolor{microsoftcyan}{cyan} trajectory) helps mitigate these oscillations and reveal the underlying fundamental progress in the river's direction (\Cref{thm:decay-main}). In this scenario, the $\wsd$ algorithm is preferred: the stable phase allows for substantial progress along the river, and the final decay phase reduces oscillations, ensuring that the iterates terminate near the river's course.

\subsection{Setting and Assumptions}

We will now formally present our theory. We use $w \in \R^d$ to denote the parameters and $L$ to denote the loss. Further, we use $\eigen{k}{H}$ and $\eigenv{k}{H}$ to denote the $k$-th largest eigenvalue and eigenvector of a matrix $H$, respectively. The ``river'' in the river valley is a 1-dimensional manifold $\river$ formalized below.

\begin{assumption}
\label{assum:river}
    We assume the existence of a ``river'', which is a 1-dimensional manifold $\river$ such that any point $w \in \river$ has a gradient $\nabla L(w)$ that is in the same direction as the minimal eigenvector direction of the Hessian, $\eigenv{d}{\nabla^2 L(w)}$. 
\end{assumption}
\begin{wrapfigure}{r}{0.3\textwidth}
    \centering
    \vspace{-0.5in}
    \includegraphics[width=0.3\textwidth]{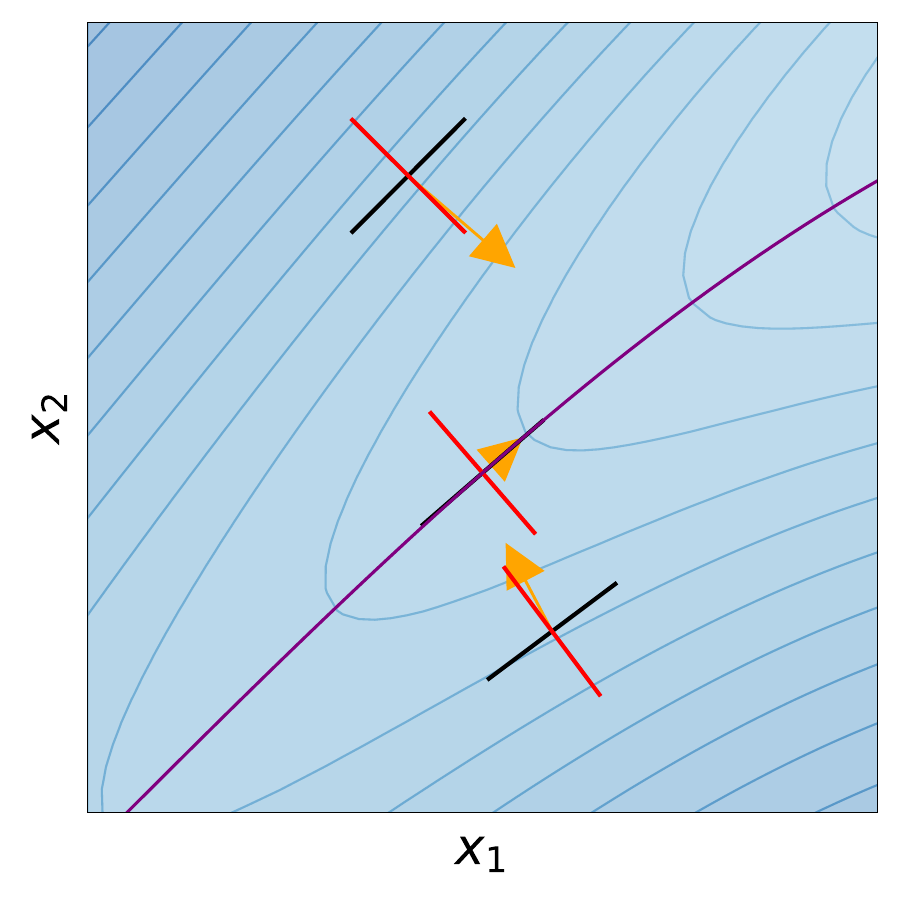}
    \caption{Illustration of the Definition of the River.}
    \label{fig:river}
    \vspace{-0.9in}
\end{wrapfigure}
Under this assumption, at every point on the river, the gradient $\nabla L(w)$ will align with the locally flattest direction, $\eigenv{d}{\nabla^2 L(w)}$, which we refer to as the \emph{river direction}. All other directions orthogonal to the river direction are considered as the \emph{mountain directions}, corresponding to the steep hillsides in our conceptual picture.
This definition is illustrated in the 2-dimensional loss landscape shown in \Cref{fig:river}. In this figure, the {\color{red}red} line represents the top eigenvector (the sharpest direction), while the black line represents the smallest eigenvector (the flattest direction) at each point. Along the \textcolor{blue}{blue} curve, which represents the river, the {\color{orange}orange} negative gradient aligns with the river direction.
For a point on the hillside, the negative gradient pulls the point downward along the hillside towards the river, while also while also moving it along the river. This intuition will be rigorously formalized in the theorems that follow.

To analyze the optimization dynamics, we will consider a neighborhood $U$ of the river $\river$. Inside this neighborhood, we will impose several technical assumptions. For example, we will assume that the Hessian has an eigengap between the smallest and the second smallest eigenvalues, so that the river direction is well-defined in $U$. We will also assume that the last eigenvector of the Hessian $v_d$ changes slowly, which implies that the river direction does not change rapidly during optimization. We summarize these assumptions below.

\begin{assumption}[Regularity Assumption]
\label{assum:technical}
    There exists an open set $U$ containing $\river$ satisfying the following assumptions:
    \begin{enumerate}[leftmargin=*]
        \item {Analyticity.} $L(w)$ is analytic with respect to $w$.
        \item {Bounded Hessian.}  There exists a constant $\maxgamma > 0$, such that $\forall w \in U, \| \nabla^2 L(w) \|_{\textup{op}} \le \maxgamma.$
        \item {Existence of Eigengap.} There exist constants $\flatgamma, \gamma > 0$, such that $$\forall w \in U, \eigen{d - 1}{\nabla^2 L(w)} > \gamma + 4 \flatgamma, |\eigen{d}{\nabla^2 L(w)}| < \flatgamma.$$
            \vspace{-0.2in}
        \item  {Slow Spinning of $v_d$.} There exist constants $\normbound > 0, \err \in [0, 0.01)$, such that $$\forall w \in U, \normbound_{\min} < \| \nabla L(w) \|_2 \le \normbound, \textup{ and } \| \nabla  \eigenv{d}{ \nabla^2 L(w)} \|_{\textup{op}}  \le \err \gamma/(2\normbound).$$
        This means that the river direction $v_d$ changes slowly during optimization.
        \item {Uniqueness of $\river$.} For any point $w \in U - \river$, the gradient $\nabla L(w)$ is not parallel to $v_d(\nabla^2 L(w))$.
        \item {Conservation of Gradient Flows.} There exists an open subset $V \subset U$ and  a constant $r > \frac{10 \normbound}{\gamma}$ for $\gamma$ defined in Assumption 2.3 such that $\forall w \in V$, the $r$-neighborhood of the gradient flow starting from $w$ stays in $U$ for continuous time $\maxT \ge {10 \log(2 \normbound / (\err \normbound_{\min}))}/{\gamma}$.
    \end{enumerate}
\end{assumption}

Throughout the analysis, $\kappa$ should be treated as a \emph{small} {dimensionless} constant, indicating the river spins slowly. We can now define the river valley landscape.
\begin{definition}[River Valley Landscape]
\label{def:river-valley}
If a loss function $L$ satisfies~\Cref{assum:river,assum:technical}, then we will claim that the loss function is a river valley.
\end{definition}

One simple example of a river valley landscape is the quadratic loss $L(x_1, x_2) = \frac{\gamma x_1^2}{2} - x_2$ with $\kappa$ equals to $0$. In this case, the river is simply the line $x_2 = 0$. However, the river valley landscape can also be more complex and non-convex, see \Cref{fig:demo} for an illustration.
We will demonstrate that when the loss function is a river valley, the optimization trajectory of (stochastic) gradient descent will track the river. In fact, we can even  prove that the iterates will follow the river with a predictable pace, which is characterized by the reference flow defined below.

\textbf{Reference Flow.} We introduce a Riemannian gradient flow constrained to the river $\river$, serving as a reference in the following theorems. This flow intuitively represents the dynamics of iterates during a gradient flow on the loss constrained by the river. We will denote the projection to the tangential space of the river as $\riverproj(w)$ for $w \in \river$ and choose an arbitrary starting point $x_0$ on the river. The reference flow is defined as
\begin{align}
\label{eq:refflow}
    d x(t) = -\riverproj \left(x\left(t\right)\right) \nabla L(x(t)) dt, \quad x(0) = x_0.  
\end{align}
Here, we use \( x \) to represent a point on the river, distinguishing it from \( w \), which denotes a weight in the original space. \( t \) refers to the continuous time variable.

\subsection{Gradient Flow Dynamics}

In this section, we will consider gradient flow in the river valley landscape starting from a point $w \in V$, with $V$ defined in \Cref{assum:technical}.6: 
\begin{align}
\label{eq:gfflow}
    d w(t) = - \nabla L(w(t)) dt, \quad w(0) = w \in V.  
\end{align}

We prove in the theorem below that the gradient flow starting from $w$ will eventually converge near the river and remain close to it. Subsequently, if we project the iterate $w(t)$ onto the river, the projection will move along the river at a pace similar to the reference flow $x(\cdot)$ (\cref{eq:refflow}). This phenomenon is visualized on \Cref{fig:gf}.

\begin{theorem}
\label{thm:gftracksriver}
If a loss $L$ is a river valley (\Cref{def:river-valley}), for the gradient flow $w(t)$ defined in~\Cref{eq:gfflow}, the iterate will obey the following dynamics:
\begin{enumerate}[leftmargin=*]
    \item The iterate will first converge to a neighborhood of the river. The distance betweeen the iterate and the river is bounded after a constant converge time $\convergeT = {2\log(2 \normbound / (\err \normbound_{\min}))}/{\gamma}$.
     \begin{align*}
        \mathrm{dist}(w(\convergeT), \river) =  \mathrm{min}_{t} \| x(t) - w(\convergeT) \|_2 \le 2\err \normbound/\gamma.
     \end{align*}
    \item After convergence, the iterate will track the river closely with the same pace as the reference flow. There exists a time shift $T_0$ depending on $w(\convergeT)$, such that for any $t \in [\convergeT, \maxT]$ for $\maxT$ defined in~\Cref{assum:technical}.6, there exists a $\tilde t \in [(1-\epsilon)T, (1+\epsilon)T]$ satisfying that,
    \begin{align*}
    \| x(T_0 + \tilde{t}) - w(t) \|_2 &\le   2\err\normbound/\gamma ,
    \end{align*}
    for $\epsilon = 30 \err$. 
\end{enumerate}
\end{theorem}

\begin{figure}[t]
    \centering
    \begin{subfigure}[b]{0.33\textwidth}
        \centering
        \includegraphics[width=\textwidth]{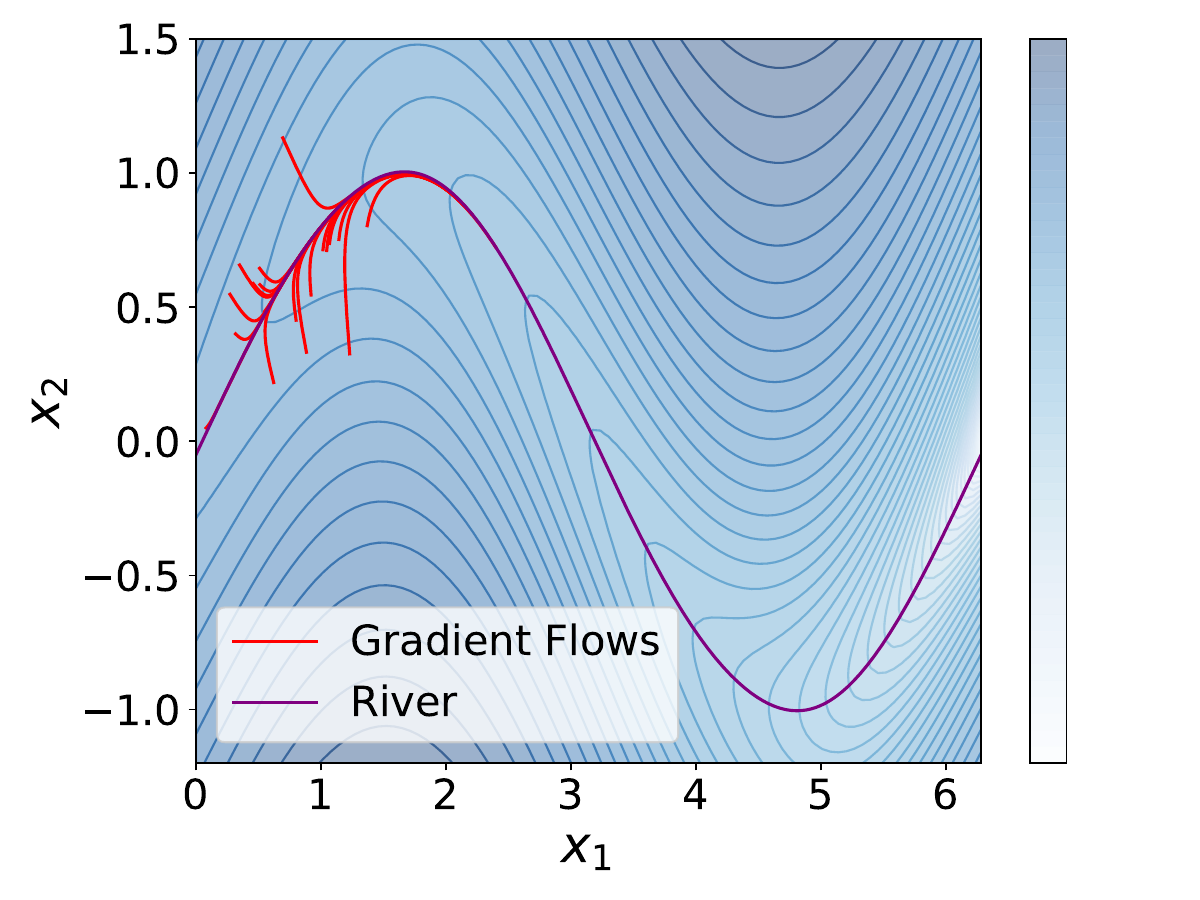}
        \caption{Gradient Flow Dynamics}
        \label{fig:gf}
    \end{subfigure}
    \begin{subfigure}[b]{0.32\textwidth}
        \centering
        \includegraphics[width=\textwidth]{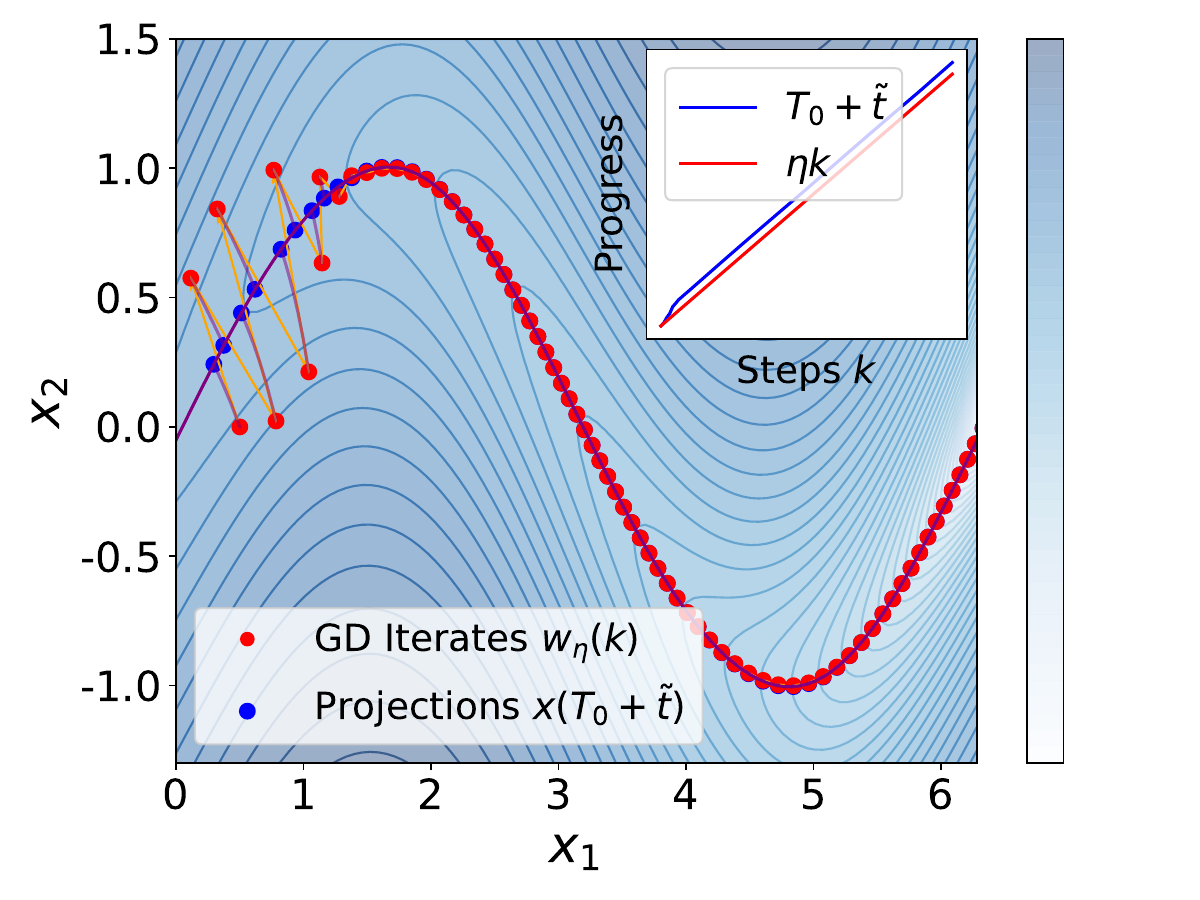}
        \caption{Gradient Descent Dynamics}
        \label{fig:gd}
    \end{subfigure}
    \begin{subfigure}[b]{0.33\textwidth}
        \centering
        \includegraphics[width=\textwidth]{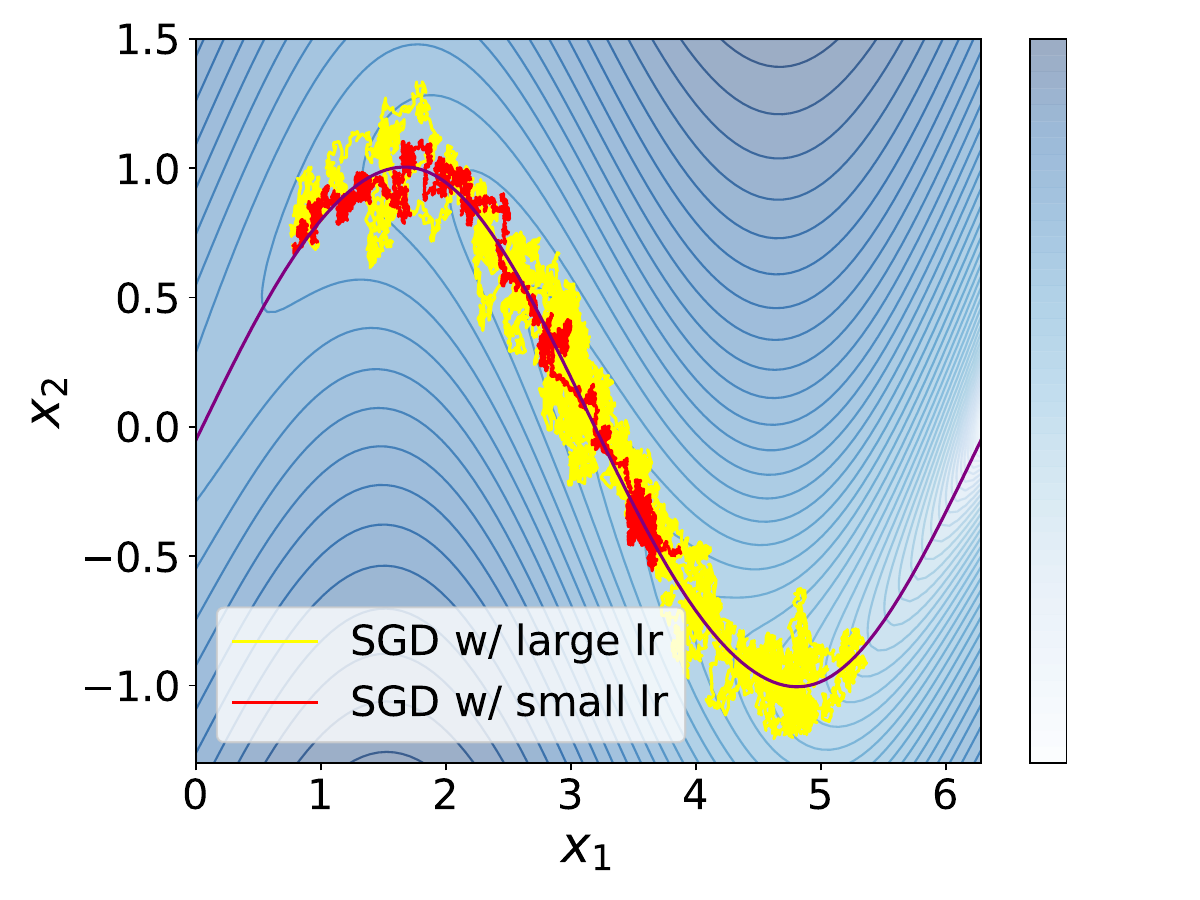}
        \caption{Stochastic Gradient Descent}
        \label{fig:sgd}
    \end{subfigure}
        \caption{
                \textbf{Illustration of Theory.} 
         We validate our theory using a 2D example. The \textcolor{blue}{blue curve} represents the "river", where the gradient aligns with the minimal eigenvector of the Hessian. (Left) Randomly initialized \textcolor{red}{gradient flows} converge near the river and follow it closely thereafter. (Middle) Discrete step-size gradient descent shows similar behavior: after initial oscillations, the \textcolor{red}{gradient descent iterates} align closely with their \textcolor{blue}{projections} on the river. (Right) Stochastic gradient descent (SGD) also tracks the river. In contrast to the discrete-step gradient descent, the iterates oscillate around the river rather than staying on it. The {trajectory} with a larger learning rate exhibits faster progress and greater oscillations than {trajectory} with a smaller learning rate.
        }
         \label{fig:toyloss}
        \vspace{-0.2in}
\end{figure}

The proof is deferred to~\Cref{app:gf}. In this theorem, the lower bound on $T$ represents the time required for the iterate to converge near the river. Here $x(T_0 + \tilde{t}) $ can be viewed as a projection of $w(t)$ onto the river. 
As both the geometric error ($2\err\normbound/\gamma$) and the time-alignment error ($\epsilon$) vanish when $\kappa$ is small, this projection is not only close to $w(t) $ but also moves at nearly the same rate as the reference flow. Here the term $T_0$ acts as a shift, reflecting the dependency on the initialization, as optimization trajectories starting from different initial points will enter the river at distinct locations. The term $\tilde t$ represents the progress made along the river. This interpretation will remain consistent in the subsequent sections.

\emph{Proof Sketch.} The proof of \Cref{thm:gftracksriver} relies heavily on the structure of the Hessian matrix. The eigengap assumption (\Cref{assum:technical}.3) ensures that the loss constrained to the mountain dimensions is steep and convex enough to guarantee convergence to the river. The slow spinning assumption (\Cref{assum:technical}.4) ensures that the river direction does not change rapidly, allowing the iterate to follow the river closely after convergence.

\subsection{Gradient Descent Dynamics}

We will now analyze gradient descent with a discrete learning rate. This process is inherently discrete, differentiating it from the continuous gradient flow analyzed above. Similar to the continuous case, an iterate far from the river will converge to the river (as visualized in the first few steps of \Cref{fig:gd}). 
To ease our analysis, we will skip the convergence analysis and assume the starting point $w$ lies on the course of the river. 

\begin{align}
\label{eq:gdflow}
    w_{\eta}(k + 1) - w_{\eta}(k) = - \eta \nabla L(w_{\eta}(k)), \quad w_{\eta}(0) = w \in \river.  
\end{align}

Here we use $k$ to denote the discrete time step, in contrast to the continuous time variable $t$ used in the previous section.
When the learning rate $\eta$ is small, despite the discretization error, gradient descent will still track the river closely at approximately the same pace as the reference flow similar to gradient flow.  In this case, the progress over $k$ steps will be approximately $\eta k$, as shown in the following theorem.

\begin{theorem}
\label{thm:gdtracksriver}
If a loss $L$ is a river valley (\Cref{def:river-valley}), when $\eta < \frac{\gamma}{2\maxgamma^2}$, for the gradient descent $w_{\eta}(k)$ defined in~\Cref{eq:gdflow} with initialization $w$ on the river, there exists a time shift $T_0$ depending on $w$ and $\eta$, satisfying that for any $\eta k \le \maxT$, there exists a $\tilde t \in [(1-\epsilon)\eta k, (1+\epsilon)\eta k]$  satisfying that,
\begin{align*}
\| x(T_0 + \tilde{t}) - w_{\eta}(k) \|_2 &\le  10 \err \normbound/\gamma\,,
\end{align*}
for $\epsilon = 30 \err + 4\eta \flatgamma$. 
\end{theorem}

The proof is deferred to~\Cref{app:gd}. We observe that the distance of the iterates from the river remains on the same order as in Theorem~\ref{thm:gftracksriver}. However, the learning rate will affect the pace of the iterates, introducing the additional error term $\eta \flatgamma$ in time-alignment error ($\epsilon$). This term is unavoidable, as the discretization inherent in gradient descent introduces a slight deviation from the continuous gradient flow, even in the absence of the hill component of the loss. Finally, the 
\Cref{thm:gdtracksriver} predicts that a larger learning rate $\eta$ will induce higher progress $\eta k$ down the river given the same number of steps $k$. 

We illustrate~\Cref{thm:gdtracksriver} in~\Cref{fig:gd}. Here the {\textcolor{red}{red}} points correspond to gradient descent iterates and {\textcolor[rgb]{0.0, 0.5, 0.0}{green}} points are the corresponding predicted projected points on the river. As the theory predicts, the distance between the iterates and projections diminishes after a few steps. Finally, the inset indicates that the projected points traverse the reference flow on the river at a rate proportional to the learning rate $\eta$.

\subsection{Stochastic Gradient Descent Dynamics}

The above analysis holds for deterministic dynamics and we will now proceed to model the stochasticity in the optimization process. This stochasticity will stop the iterate from fully converging to the river and lead to oscillation in the mountain direction. To simplify the analysis, we will consider a special case where the river direction is a constant and the river reduces to a straight line. We conjecture that the results below can be generalized to a more general setup and leave the detailed analysis to future work.

\begin{assumption}[Straight River]
\label{assum:straight}
For $U$ in~\Cref{assum:technical}, $\forall w \in U, \| \nabla  v_d(\nabla^2 L(w)) \|_2  = 0$. In this case, the river is a straight line parallel to the direction of $v_d(\nabla^2 L(w))$.
\end{assumption}

Under~\Cref{assum:straight}, $v_d(\nabla^2 L(w))$ is a constant vector for $w\in U$ and we will use $v_d$ to denote this vector. We will also assume that the update is deterministic in the direction of the river, which simplifies our proof while still capturing the essential dynamics of SGD. Consequently, we can express the SGD update as follows:
\begin{align}
    \tilde w(k + 1) = \tilde w_k - \eta_k \nabla L\left(\tilde w\left(k\right)\right)  + \eta_k \noise_k, \quad \noise_k \sim \normal{0}{\sigma^2 \left(\identity_d - v_d v_d^T \right)}, \quad \tilde w(0) = w \in \river.  \label{eq:sgd}
\end{align}
Here $\normal{\mu}{\Sigma}$ indicates the normal distribution with mean $\mu$ and covariance $\Sigma$. Compared to deterministic gradient descent, the introduced noise $\noise_k$ causes the iterates to deviate from the river instead of fully converging to it (see the difference between~\Cref{fig:gd} and~\Cref{fig:sgd}). Consequently, we need to impose additional assumptions on the loss landscape to analyze the dynamics of SGD.

\begin{assumption}[Regularity Assumption for SGD]
\label{assum:noise}
In the setting of Assumptions 2, we assume in addition the following:
\begin{enumerate}[leftmargin=*]
    \item {Bounded Hessian.} There exists a constant $\tau > 0$, such that for any weight $w \in U$, the nuclear norm of the Hessian is bounded. 
    $$\| \nabla^2 L(w) \|_{*}  =  \sum\limits_{i = 1}^{d} |\eigen{i}{\nabla^2 L(w)}| \le \tau. $$
    \item {Bounded Third Order Gradient.} There exist  constants $\rho > 0, \kappa' \in [0, 0.01]$, such that,
    \begin{align*}
        \| \nabla^3  L(w) \|_{\textup{op}} \le \rho, \quad \Delta \rho \le \kappa' \gamma^2.
    \end{align*}
    \item {Bounded Loss.} There exists a constant $M > 0$ such that $\forall w, L(w) < M$.
\end{enumerate}
\end{assumption}

In this assumption, we treat $\kappa'$ as a small constant, indicating that the influence of the third-order gradient is minimal. This suggests that the overall shape of the loss landscape is predominantly governed by the first and second-order information. 
We will now analyze the dynamics of SGD in two phases: the stable phase and the decay phase. In both phases, the iterates will track the reference flow with progress linear to the sum of the learning rates. At the same time, the iterates will also incur a loss due to the noise, which will be linear with respect to the learning rate on the last step.

\subsubsection{Stable Phase} 

We start with the stable phase, where the learning rate \(\eta_k = \eta\) remains constant. Similar to the deterministic case, we will demonstrate that the expected loss of the iterate \(\E[L(\tilde{w}(t))]\) closely follows the loss along the reference flow \(L(x(T))\), with \(T \approx \eta t\). However, the stochasticity in the updates introduces an additional term, \(\eta \sigma^2\), into the loss. This provides a formal basis for decomposing the loss into its river and hill components, that is, $L(x(T))$ and $\eta\sigma^2$, respectively.

\begin{theorem}
\label{thm:sgd-main}
Suppose a loss $L$ is a river valley (\Cref{def:river-valley}) and satisfies~\Cref{assum:straight,assum:noise}. Then, for any constants $\delta \in (0,1)$, for sufficiently small learning rate $\eta$ depending on the regularity constants, \footnote{Deferred to~\Cref{assum:regular-L} in Appendix.}, there exists a time shift $T_0$ depending on $w$ and $\eta$, the SGD iterates (defined in~\Cref{eq:sgd}) with $\eta_k = \eta$ satisfies that for any integer $k \in [1 /{\eta\gamma}, \maxT / \eta]$, there exists a $\tilde t \in [(1 - \epsilon_{t}) \eta k, (1 + \epsilon_t) \eta k]$ satisfying that,
\begin{align*}
\E[L\left( \tilde w (k)\right)] - L( x( T_0  + \tilde t))& =  {(d - 1) \eta \sigma^2}/{2} + \epsilon_{L} \end{align*}

where $\epsilon_{t} = 4 \eta \flatgamma$ and $|\epsilon_{L}| \le \tau \eta^2  \sigma^2+ \rho (C {d\eta \sigma^2}/{\gamma})^{3/2} + C \kappa' d \eta \sigma^2 +  \delta (2 M + \eta \sigma^2 d) \ll (d - 1) \eta \sigma^2$ with $C = 200 \log(64 \gamma T/\delta)$. 
\end{theorem}

The proof is deferred to~\Cref{app:sgd}. In~\Cref{thm:sgd-main}, the error term in the approximation of the pace of the projection remains the same as in the deterministic case (\Cref{thm:gdtracksriver}). However, the stochasticity introduces an additional hill component $(d - 1)\eta \sigma^2 / 2$ to the expected loss at the iterate. The hill component increases linearly with the learning rate. The error term in loss term $\epsilon_{L}$ can be decomposed into three parts: (1) $ \tau \eta^2  \sigma^2 + \rho (C {\eta \sigma^2}/{\gamma})^{3/2}$ are higher order discretization effects of learning rate $\eta$; (2) $C \kappa' d \eta \sigma^2$ is caused by the change of the Hessian in the valley dimensions and will diminish when $\kappa'$ is small; (3)  $\delta (2M + \eta \sigma^2 d)$ accounts for the small chances that the iterate will escape the neighborhood of the river due to the stochastic updates. While the theorem only considers the case where $v_d$ is a constant vector, we conjecture that the theorem can be extended to a general setting and verify this conjecture on a toy loss (see \Cref{fig:sgd}).

\subsubsection{Decay Phase} Finally, we will consider the decay phase in training and will show that a proper decaying schedule can reduce the hill component of the loss rapidly. We will first define our decaying schedule, starting from step $k_s = \left\lceil T/\eta \right\rceil$:
\begin{align}
\label{eq:decaylr}
    \eta_k = 
        \frac{\eta}{2 + (k - k_s)\eta \gamma}, \quad  & k_s \le t \le  1.1 k_s.
\end{align}
We choose this schedule to maximize the loss decrease rate on a quadratic function (see~\Cref{app:quadratic}) because we perform quadratic approximations of the loss near the river in our analysis. Our theorem predicts that the hill component of the loss will decrease linearly with the learning rate under this learning rate schedule, consistent with the empirical findings in~\cite{hu2024minicpm}.

\begin{theorem}
\label{thm:decay-main}
Under the setting of~\Cref{thm:sgd-main}, the SGD iterates (defined in~\Cref{eq:sgd}) with the learning rate schedule defined in~\Cref{eq:decaylr} satisfies that for any integer $k \in [k_s, 1.1 k_s]$, there exists a $\tilde t \in [(1 - \epsilon_{t}) T(k), (1 + \epsilon_t) T(k)]$ satisfying that, 
\begin{align*}
\E[L\left( \tilde w (k)\right)] - L( x( T_0 + \tilde t ))& \le  {(d - 1) {\color{red}{ \eta_k}} \sigma^2}/{2} + \epsilon_{L} 
\end{align*}
with $T(k) = T + \sum\limits_{i = k_s}^{k} \eta_{i}$.
\end{theorem}

\begin{wrapfigure}{r}{0.34\textwidth}
    \centering
    \vspace{-0.3in}
    \includegraphics[width=0.3\textwidth]{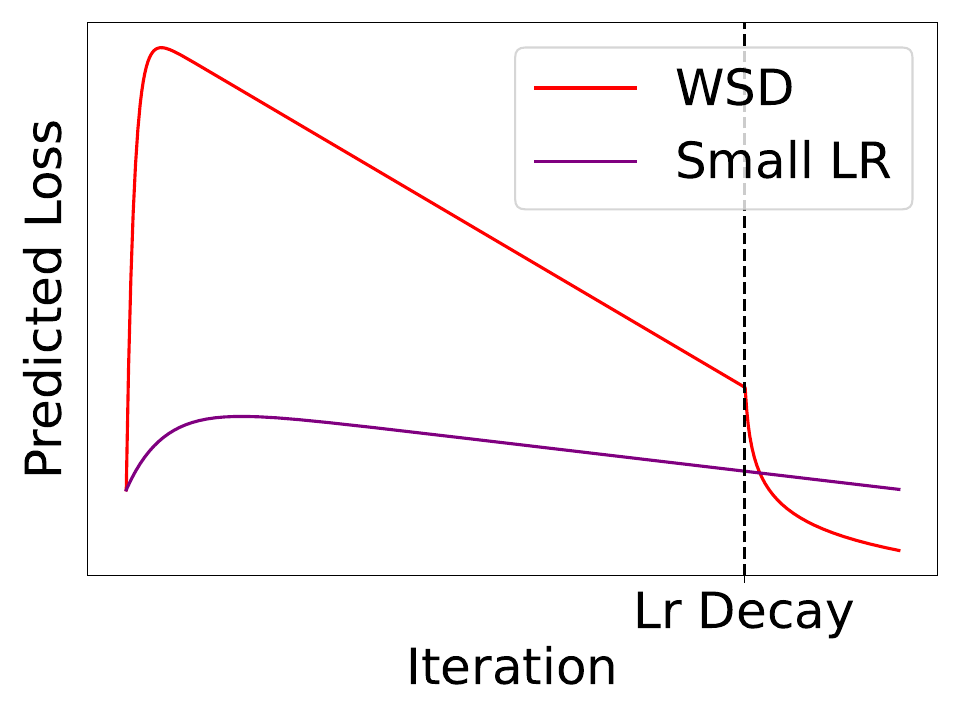}
    \caption{Predicted Loss Curve of SGD By~\Cref{thm:sgd-main,thm:decay-main} on Loss $L(x_1, x_2) = \gamma x_2^2/2 - x_1$.}
    \label{fig:thmsgd}
    \vspace{-0.4in}
\end{wrapfigure}
The formal proof is deferred to~\Cref{app:decay}. Compared with~\Cref{thm:sgd}, the hill component is now dominated by $(d - 1)\eta_k \sigma^2 / 2$, scaling linearly with the decaying learning rate. 
When the oscillation level $\sigma$ is large compared to the loss changes along the river, the loss decrease can then appear faster in the decay phase than in stable phases (see  \Cref{fig:thmsgd}). Further, the decaying phase also makes progress along the river, which corresponds to the term $\sum\limits_{i = k_s}^k \eta_i$ in the theorem. Finally, the terms used in our theorem match the scaling law formulation in the concurrent work~\citep{tissue2024scalinglawlearningrate}.

\subsection{Future Extensions}
We only consider a 1-dimensional river in the theoretical analysis above. However, it is possible to extend our assumptions to a generalized $k$-dimensional river, where the gradient lies in the flattest $k$-dimensional subspace for every point on the river.

\begin{assumption}
\label{assum:gen-river}
    We assume the existence of a ``generalized river'', which is a $p$-dimensional manifold $\river$ such that any point $w \in \river$ has a gradient $\nabla L(w)$ lies in the eigenspace spanned by the last $k$ eigenvectors' direction of the Hessian, $\{ \eigenv{i}{\nabla^2 L(w)} \mid i \in [d - p + 1, d] \}$. 
\end{assumption}

Based on this assumption, we can define \emph{generalized river valley Structure}. We believe that under some appropriate assumptions, there will be a similar coupling between the dynamics in the original river valley landscape and the dynamics constrained on the generalized river, and leave this for future work.

There are several other potential technical improvements to our theory that could be explored in future works. First, the analysis of the stochastic setting may be extended to include a river that is not a straight line. Second, the multiplicative constant of 4 in the eigengap condition (Assumption 2.3) might be eliminated through more refined analysis. Lastly, the upper bound on the learning rate in~\Cref{thm:gdtracksriver} (currently $\gamma / 2\maxgamma^2$) could potentially be improved to $\Theta(1/ \maxgamma)$ through more fine-grained analysis, which is the maximal learning rate in the quadratic case.

\subsection{Visualizing the River Valley} 
We use a direct probing method to verify our theory. Our theory suggests that when the learning rate is large, the model will bounce back and forth between the sharp valleys. However, in the decay phase, the model will move downwards the hillside to approach the river. This suggests that if we connect two checkpoints in the stable phase, we should expect to see a projection of the valley, and if we connect two checkpoints in the decay phase, we should expect to see smooth decreasing curves. To verify this, we pretrain a 124M GPT-2 model on OpenWebText. In the first run, we train the model with a constant learning rate for 25B tokens and interpolate between two checkpoints at 20B and 25B tokens (\Cref{fig:highconnect}). In the second run, we branch off from the first run at 20B tokens and decay the learning rate for 5B tokens, and we interpolate between two checkpoints at 20B and 25B tokens (\Cref{fig:rampconnect}).
The interpolation results closely resembles our theory. This observation is also consistent with~\cite{sanyal2023early} which shows weight averaging improves model performance in the earlier part of the cosine training runs, where the learning rates are higher. Additionally, the smooth decreasing curves we observed when connecting two checkpoints in the decay phase are consistent with the findings in~\cite{hagele2024scalinglawscomputeoptimaltraining}.
\begin{figure}[t]
    \centering
    \begin{subfigure}[b]{0.30\textwidth}
        \centering
        \includegraphics[width=\textwidth]{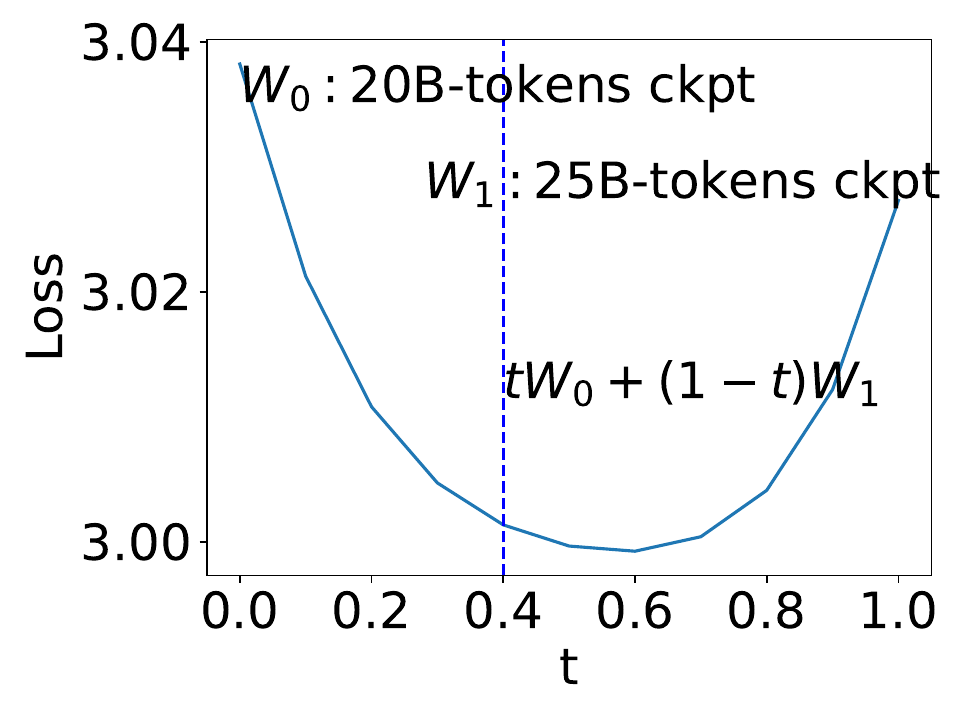}
        \caption{Stable Phase}
        \label{fig:highconnect}
    \end{subfigure}
    \begin{subfigure}[b]{0.30\textwidth}
        \centering
        \includegraphics[width=\textwidth]{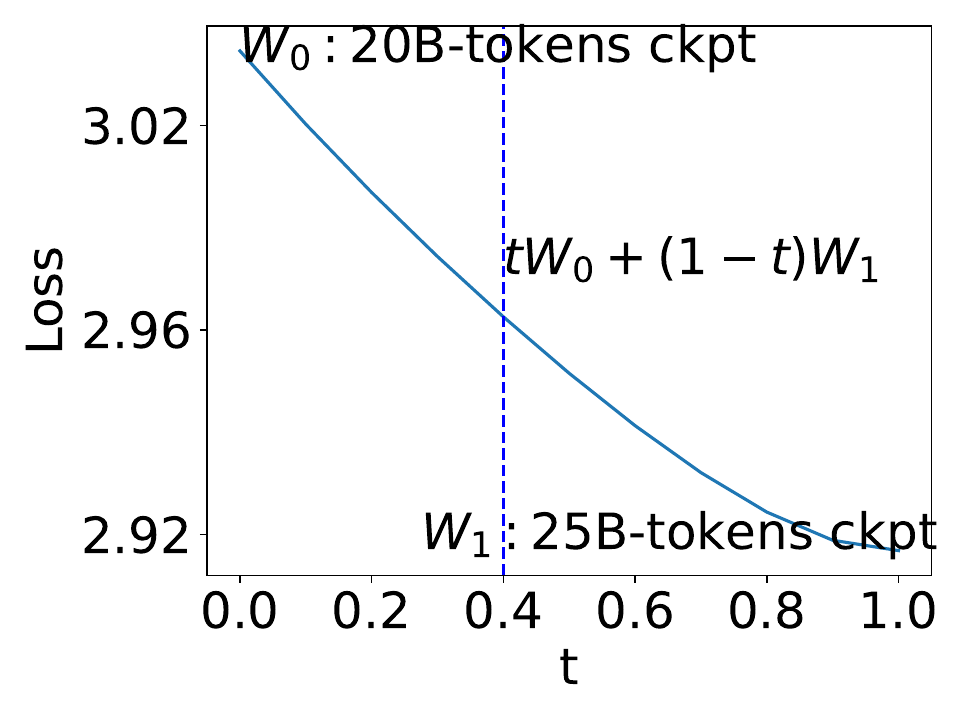}
        \caption{Decay Phase}
        \label{fig:rampconnect}
    \end{subfigure}
    \quad
    \begin{minipage}[b]{0.36\textwidth}%
        \caption{\textbf{Probing Loss Landscape.}  We validate the river valley analogy by interpolating stable and decay phases in GPT-2 pretraining experiments. We observe that loss resembles a valley when constrained on the segment connecting two models during the stable phase and smoothly decreases when connecting two models during the decay phase.}
        \label{fig:loss}
    \end{minipage}
\end{figure}

%% file: main/data.tex
\section{Uncertainty Variation in Data Distribution Shapes the River Valley Landscape}
\label{sec:data}

What causes the loss landscape to resemble a river valley structure? In this section, we propose and validate the hypothesis that variations in next-token uncertainty shape the loss landscape. When predicting a deterministic fact, a large learning rate can boost the model's confidence, accelerating learning. However, when the next token is inherently ambiguous—such as the continuation of a phrase like "I am"—the model must learn a calibrated distribution, which may necessitate a smaller step size. This variation in uncertainty leads to differences in sharpness across the loss landscape, resulting in the river valley structure.

\textbf{A Toy Bigram Language.} We formalize this intuition using a synthetic language composed of cities and names, where each city corresponds to a unique distribution of its citizens' names. For instance, one city might have a highly deterministic distribution, with most residents named "Ken``, while another city may have a more diverse {distribution of names}. This synthetic language follows the structure in~\cite{allenzhu2024physicslanguagemodels33}. The goal is to learn the distribution of names conditioned on each city. We show that cities with more deterministic name distributions align with flatter regions in the loss landscape (the "river"). In contrast, cities with more diverse name distributions correspond to sharper regions (the "hillsides").

\begin{wrapfigure}{r}{0.34\textwidth}
    \centering
    \vspace{-0.3in}
    \includegraphics[width=0.3\textwidth]{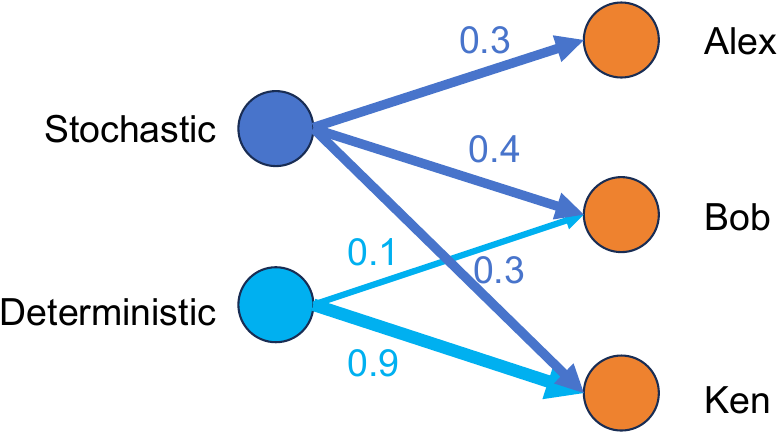}
    \caption{\textbf{Visualization of Toy Bigram Language.} We design a synthetic dataset where each city has a unique name distribution. The left shows the name distributions for two cities, one deterministic and one stochastic.}
    \label{fig:data}
    \vspace{-0.4in}
\end{wrapfigure}

Formally, let the set of cities be represented by \(\{1, \ldots, n\}\) and the set of names by \(\{1, \ldots, m\}\). Data is generated by first selecting a city \(i\) uniformly at random, then sampling a name \(j\) according to the city's name distribution. The \emph{name distribution} for city \(i\) is parameterized by a categorical distribution \(\mathrm{Categorical}([P_{i,j}]_{j=1}^m)\), where \(P_{i,j}\) represents the probability of selecting name \(j\) in city \(i\), and each \(P_{i,j} > 0\).
To quantify the uncertainty in each city’s name distribution, we compute the Gini impurity of the distribution as:
\[
U_i = \mathrm{I_G}(\mathrm{name} \mid \mathrm{city} = i) = 1 - \sum_{j = 1}^m P_{i,j}^2 \in \left[0, 1 - \frac{1}{m} \right).
\]

The value of $U_i$ reflects the uncertainty of city $i$'s name distribution. When the distribution is close to deterministic—i.e., there exists a $j$ such that $P_{i,j}$ is near 1—$U_i$ approaches its lower bound of 0. Conversely, for a nearly uniform distribution, $U_i$ approaches its upper bound of $1 - \frac{1}{m}$. 
Given this setup, we parameterize our model with $\Theta \in \mathbb{R}^{n \times m}$, where each row corresponds to a city and each column to a name. The model estimates the probability of name $j$ for city $i$ using the softmax function $\frac{\exp(\Theta_{i,j})}{\sum_{k=1}^m \exp(\Theta_{i,k})}$.

We use sampled data to train this model. Specifically, for each sampled city $i$, a name $j$ is drawn according to the true distribution $\mathcal{P}_{i,j}$. The model's task is to predict the correct name based on the sampled city, minimizing the negative log-likelihood (NLL) over the sampled data. 
The population loss is given by:
\begin{align}
\label{eq:toy}
    L(\Theta) = \frac{1}{n} \sum_{i = 1}^n \ell_i(\Theta_{i,:}) ,\quad
\ell_i(\Theta_{i,:}) = -\sum_{j = 1}^m \mathcal{P}_{i,j} \log \frac{\exp(\Theta_{i,j})}{\sum_{k=1}^m \exp(\Theta_{i, k})}.
\end{align}

This loss is separable across different cities, meaning that the contribution from each city is independent. The loss component $\ell_i(\Theta_{i,:})$ captures the contribution from city $i$, and different name distributions across cities lead to different forms of $\ell_i$.
Considering a parameter $\Theta^*$ that minimizes the loss $L$, we will show that cities with more stochastic name distributions correspond to sharper components in the loss landscape, as reflected by the average-direction sharpness of $\ell_i$ ($\mathrm{Tr}(\nabla^2 \ell_i(\theta)) \mid_{\theta = \Theta^*_{i,:}}$).

\begin{lemma}
\label{lem:sharp}
The average-direction sharpness of loss component $\ell_i$ at $\Theta^*$ equals the uncertainty of the name distribution ($U_i$).
$\mathrm{Tr}(\nabla^2 \ell_i(\theta)) \mid_{\theta = \Theta^*_{i,:}} = U_i.$
\end{lemma}
\Cref{lem:sharp} demonstrates that at the global minimum, the sharpness associated with a city decreases as the city's name distribution becomes more deterministic. This aligns with the intuition that a deterministic token corresponds to a flatter loss direction.

In a generalized river valley landscape, the hillsides represent the sharper components of the loss, which in this context correspond to cities with more stochastic name distributions. Conversely, the river corresponds to cities with more deterministic name distributions. We can further establish the existence of a generalized river (\Cref{assum:gen-river}) in this loss landscape under appropriate assumptions about $\mathcal{P}$ (see \Cref{thm:toy}). Along the river, the gradient remains nonzero only for the cities with more deterministic name distributions, reinforcing the connection between determinism and flatness in the loss landscape.

\begin{figure}[!h]
    \centering
    \begin{subfigure}[b]{0.45\textwidth}
        \centering
        \includegraphics[width=\textwidth]{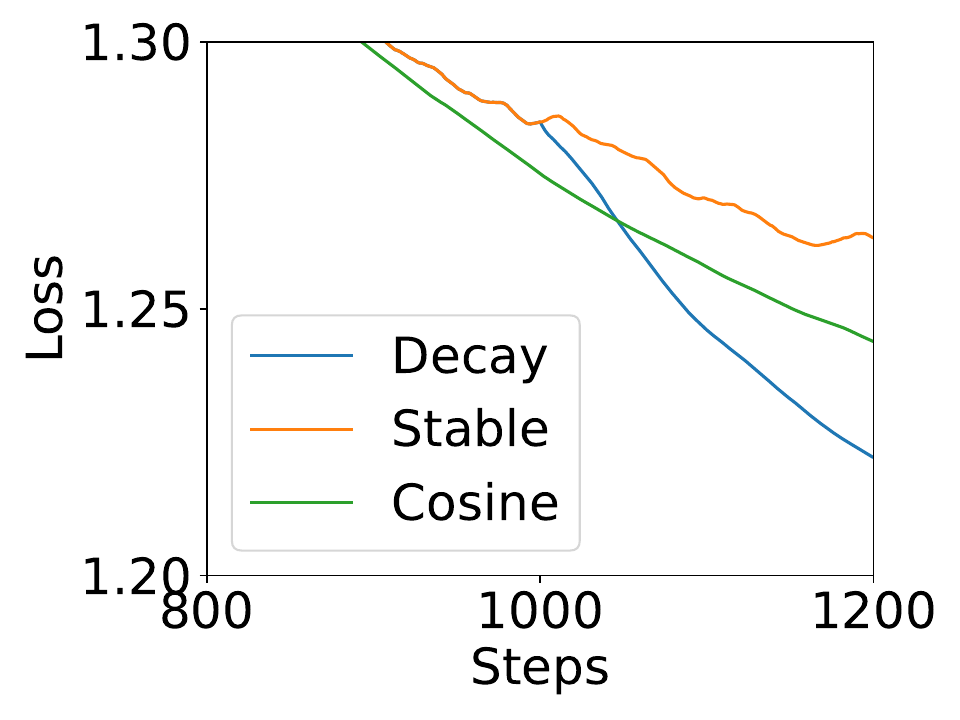}
    \end{subfigure}
    \hfill
    \begin{subfigure}[b]{0.54\textwidth}
        \centering
        \includegraphics[width=\textwidth]{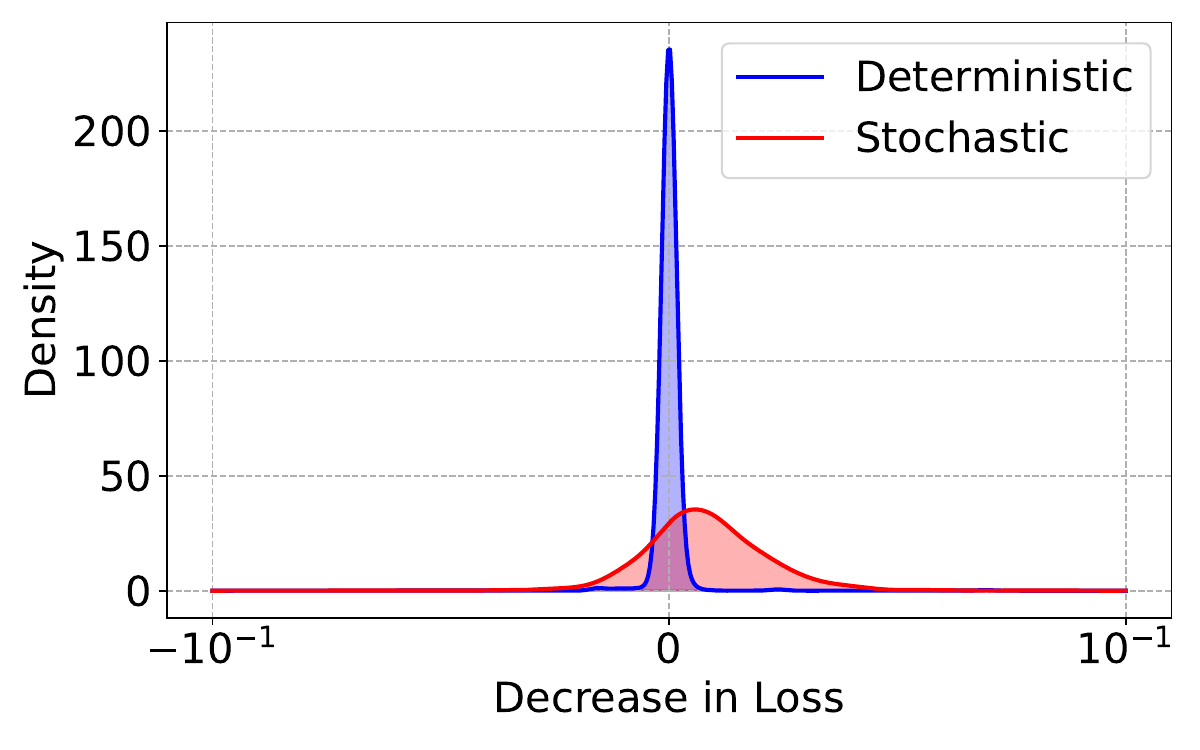}
    \end{subfigure}
    \caption{\textbf{Loss Curves and Relative Loss Decrease Distribution.} The left figure reproduces the non-traditional loss curve of $\wsd$ on the synthetic language, showing that the loss remains elevated during the stable phase but sharply declines during the decay phase. The right figure displays the distribution of the relative loss difference from the two different distributions comparing the stable phase and the decay phase. We can observe that the loss decrease over the stochastic distribution is much larger than the decrease in the deterministic distribution. }
    \label{fig:toy}
\end{figure}
\textbf{Empirical Verification.} 
We empirically verify that the loss curve of $\wsd$ can be reproduced in our synthetic setting. The dataset used contains two types of cities: (1) a deterministic type with name distribution's entropy less than 0.2, and (2) a stochastic type with name distribution's entropy greater than 1. Each type contains 1.8k cities and there are 10 possible names. We first train a toy model on this synthetic data and successfully replicate the non-traditional loss curve of $\wsd$ (\Cref{fig:toy},left).

Next, we convert the data into a synthetic language in the format "The resident of [CITY]: [NAME]" and fine-tune a GPT-2 model, pretrained on OpenWebText, using this synthetic data. We experiment with two different learning rate schedules: a constant schedule (stable) and a decaying schedule (decay).
We then calculate the difference in loss between the two models' predictions for the first token of "[NAME]". A significant Spearman correlation of 0.388 (\Cref{fig:toy},right) is observed between the loss difference and the ground truth entropy per city. This correlation indicates that the loss decrease is greater during the decay phase for more stochastic populations. Furthermore, although the decay phase achieves a lower overall loss, the mean loss for the deterministic sub-population is slightly higher than in the stable run, suggesting that the stable run better learns the deterministic sub-population.

%% file: main/method.tex
\section{$\wsd$-S: A Simplification of the $\wsd$ Schedule}
\begin{figure}[t]
    \centering
        \begin{subfigure}[b]{0.3\textwidth}
            \centering
            \includegraphics[width=\textwidth]{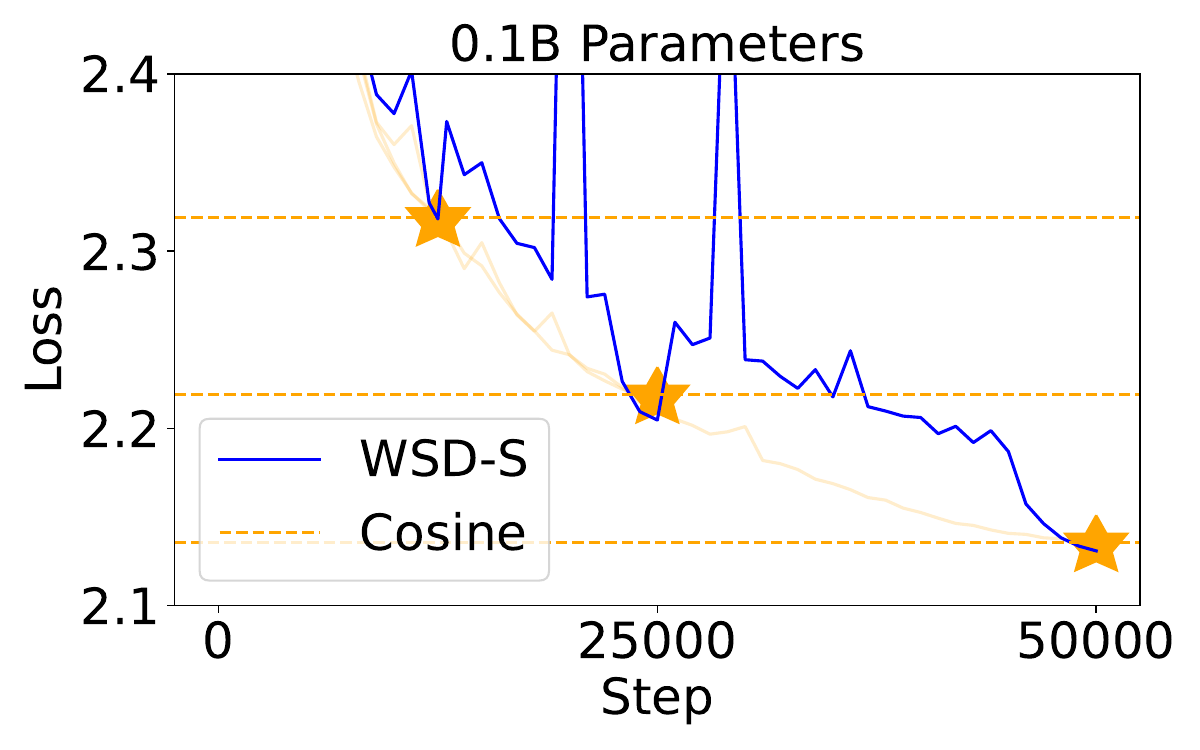}
        \end{subfigure}
        \centering
        \begin{subfigure}[b]{0.3\textwidth}
            \centering
            \includegraphics[width=\textwidth]{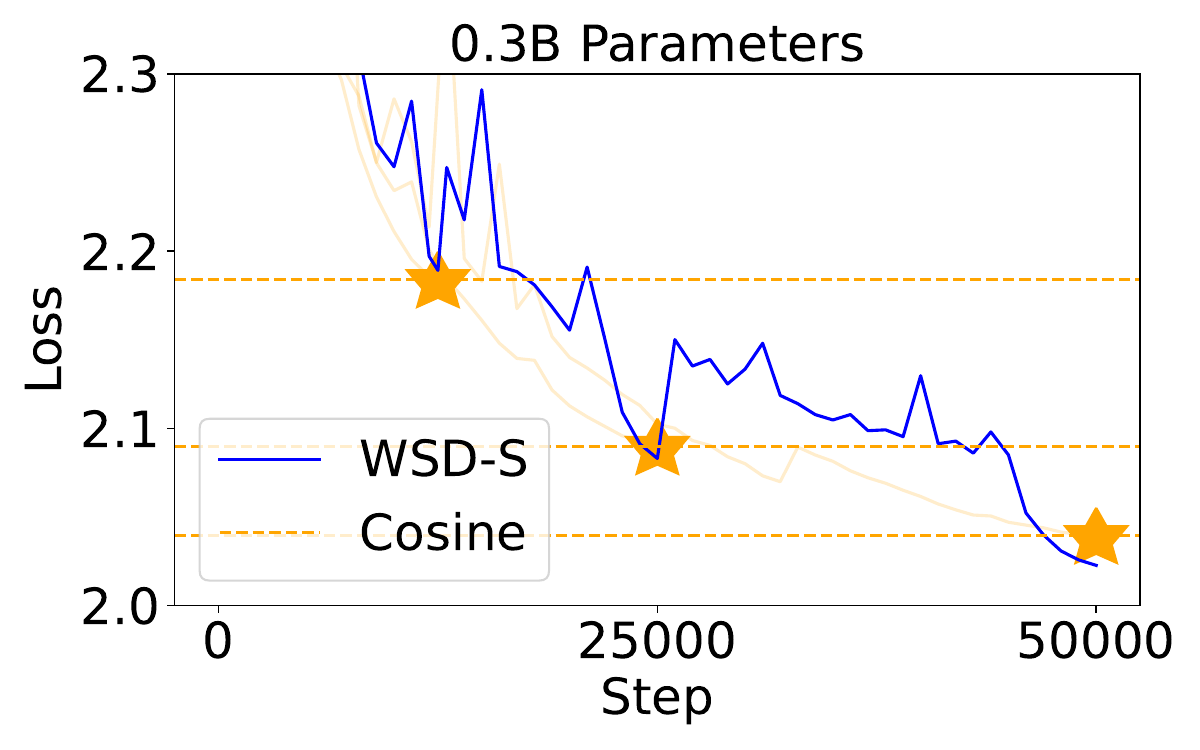}
        \end{subfigure}
        \begin{subfigure}[b]{0.3\textwidth}
            \centering
            \includegraphics[width=\textwidth]{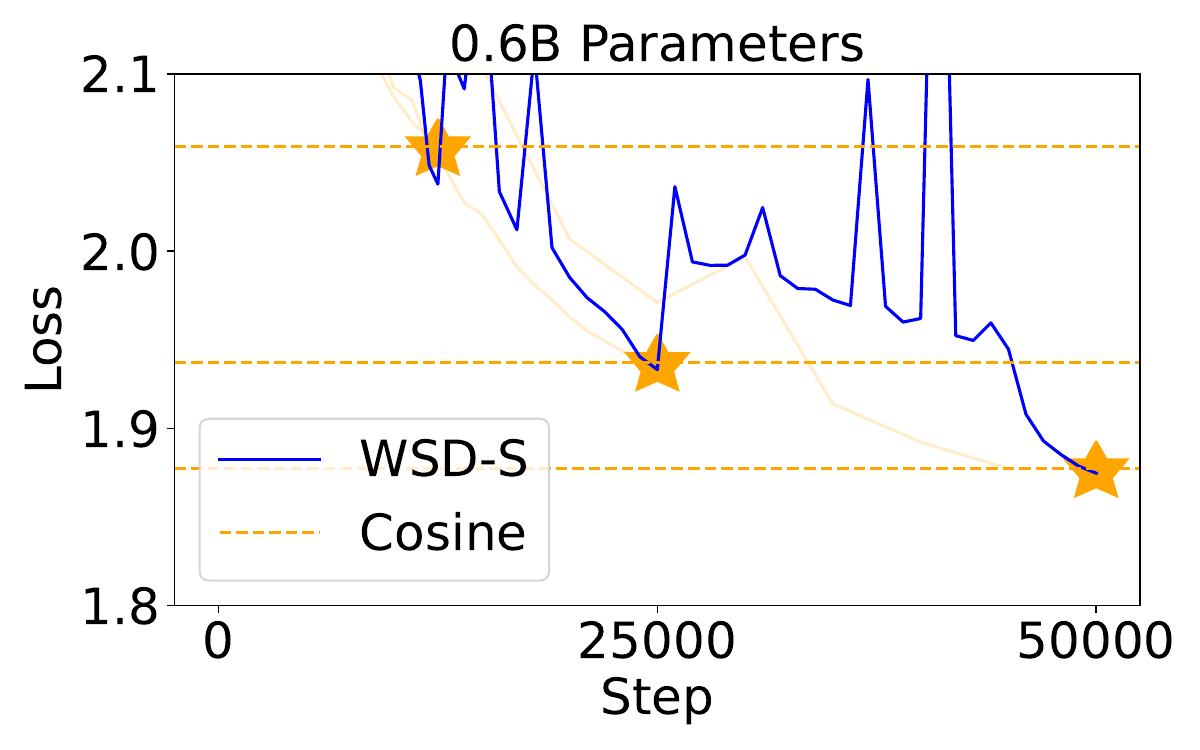}
        \end{subfigure}
        \begin{subfigure}[b]{0.3\textwidth}
            \centering
            \includegraphics[width=\textwidth]{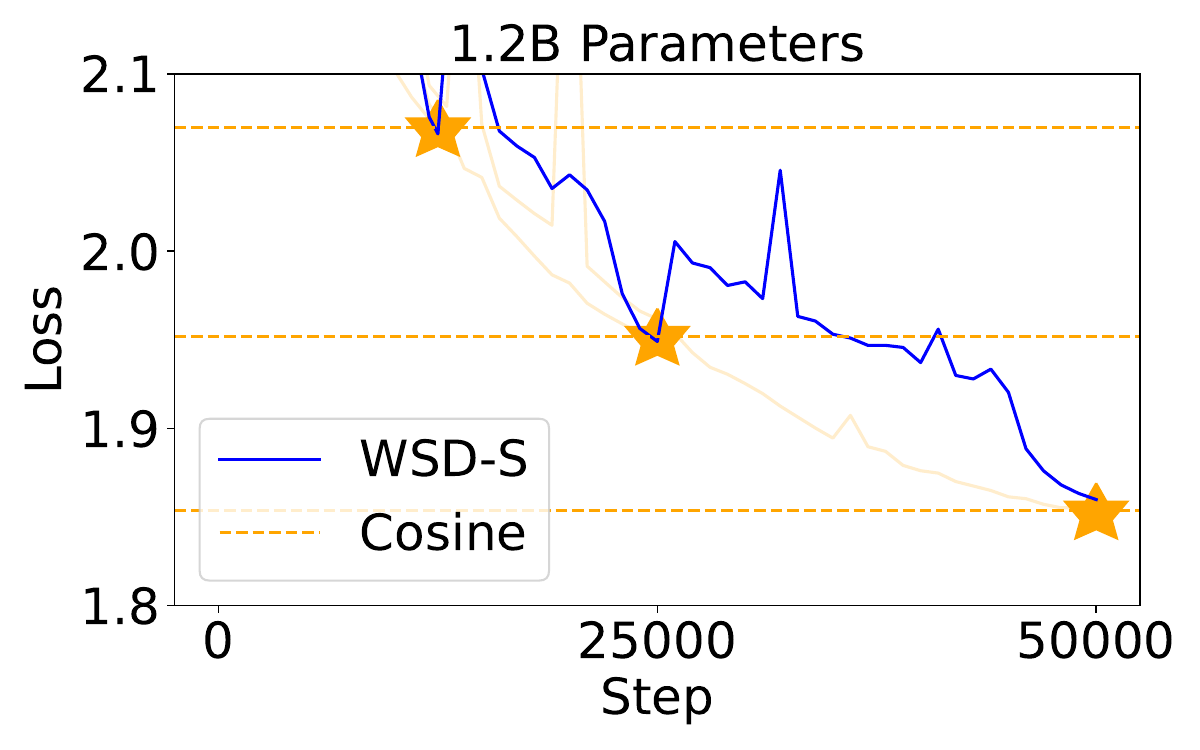}
        \end{subfigure}
        \begin{subfigure}[b]{0.58\textwidth}
            \centering
            \includegraphics[width=\textwidth]{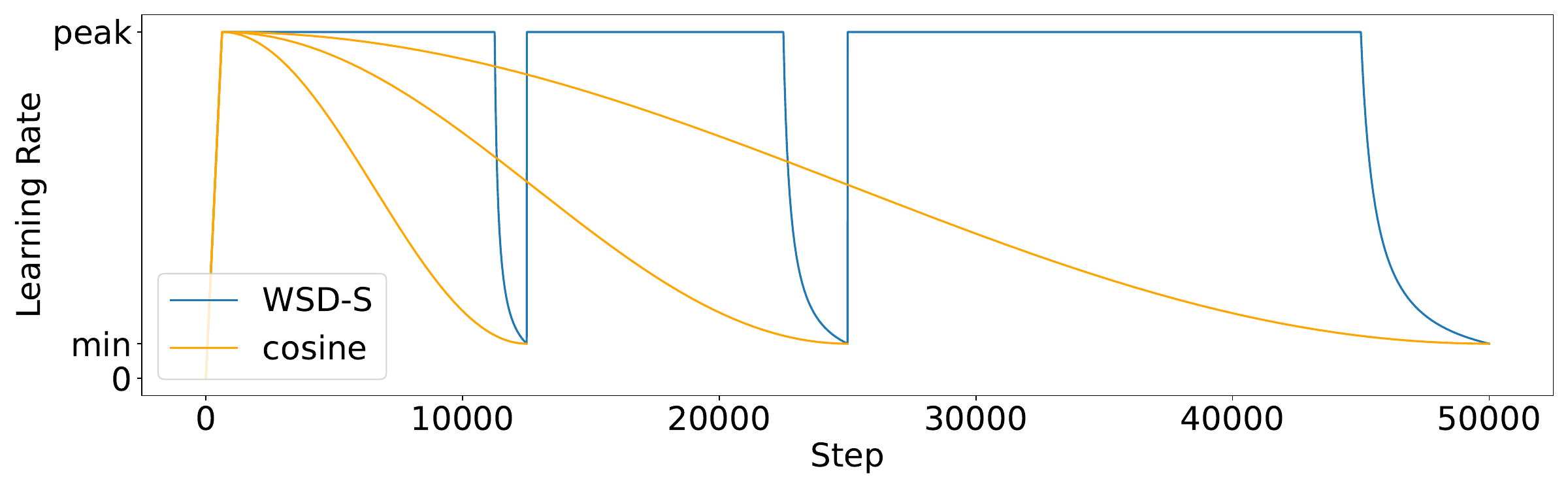}
        \end{subfigure}
    \caption{\textbf{Comparison with the Cosine Oracles.} We show that the $\wsds$ schedule can perform similarly to the Cosine schedules in a single run. The {\color{orange}{$\star$}} in the graphs visualize the terminating validation loss of different Cosine runs. The largest validation loss gap between the $\wsds$ and the Cosine schedules is 6e-3. The lower right figure plots the learning rate curves used in this experiment.}
    \label{fig:main}
\end{figure}

\subsection{Method}

The goal of continual pretraining is to generate checkpoints that exhibit good performance at multiple compute budgets in one run. Formally, our goal is to achieve multiple intermediate checkpoints $\theta_{T_k}$, each corresponding to a computing budget (number of steps) $T_k$ for $k \in \{1, \ldots, K\}$.

A strong baseline to measure the performance of $\theta_{T_k}$ would be running cosine learning rates (\Cref{fig:main}, lower right) for each budget $T_k$ separately,  decays the learning rate linearly to the cosine function between $[0, \pi]$.  We will dub this \emph{oracle} method as \textbf{$\cosineoracle$}. However, $\cosineoracle$ can't be done in a single run and will incur a high total compute budgets $\sum_{k} T_k$. A simple modification to $\cosineoracle$ is to use multiple consecutive cosine learning rates between $T_{k-1}$ and $T_k$ (\Cref{fig:learningrate}, last row), which we will dub as \textbf{$\cycliccosine$}. $\cycliccosine$ only requires a total compute budgets $T_K$. However, it leads to non-negligible performance lost compared to $\cosineoracle$ (\cite{hu2024minicpm}).

Warmup-Stable-Decay (\textbf{$\wsd$}) address this issue by maintaining a {main branch} that keep using a constant learning rate after warmup process and branch off using a decaying learning rate to achieve intermediate checkpoints. One can then continue pretrain from a checkpoint in the main branch by resuming with the same constant learning rate. Formally, $\wsd$ introduce decay starting points $D_1, \ldots, D_k$ such that $T_{i - 1} < D_{i} < T_{i}$. $\wsd$ will then correspond to the following process (\Cref{fig:learningrate}, second row):
\begin{enumerate}[leftmargin=*]
    \item Get a main branch of checkpoints $\maintheta$ by running a constant learning rate schedule for $D_K$ steps.
    \item For each $k$, run a decaying learning rate schedule for $T_k - D_k$ steps starting from $\theta^{\mathrm{main}}_{D_k}$ to get $\theta_{T_k}$.
\end{enumerate}

The above process reutilizes the main branch of checkpoints $\maintheta$ for each $T_k$ and hence reduces the total compute budget to $T_K + \sum_{k} (T_k - D_k)$. 
Recall that in the river valley landscape model, the Warmup-Stable-Decay ($\wsd$) algorithm can be viewed as a combination of a large learning rate phase to speed up progress down the river and a rapid learning rate drop at the end to reduce the oscillation. Because the decay phase also makes progress along the river (see~\Cref{thm:decay-main}), we proprose a simplified version of $\wsd$, called Warmup-Stable-Decay-Simplified (\textbf{$\wsds$}), that continues with another stable phase leaving off of the end of the previous decay phase (see the first row of \Cref{{fig:learningrate}}) without separating the training process into two branches. Formally, the $\wsds$ learning rate schedule is defined as follows:
\begin{align}
\label{eq:lrs}
\eta_k = \begin{cases} 
\text{decay}(T_{i} - D_i, \maxeta, \mineta)[t - D_i] & \text{if } \exists i, D_i < t \leq T_i; \\
\maxeta & \text{otherwise.} \\
\end{cases}
\end{align}

The key difference from our methods is the choice of initialization point when retraining starts. In $\wsd$, the second stable phase uses the model before the decay phase, whereas we use the model after it.
This process is more convenient to implement because it does not require rolling back to the main branch after each decay phase. Here the learning rate decay function $\mathrm{decay}$ can take many forms that decay the learning rate from $\maxeta$ to $\mineta$ over $T_i - D_i$ steps. In this paper, we will use the inverse proportional decay function as defined in~\Cref{eq:decay} for all experiments (visualized in~\Cref{fig:learningrate}, first two rows).

\begin{align}
\label{eq:decay}
\frac{1}{\text{decay}(T, \mineta, \maxeta)} &= \left[ \frac{t}{T} \frac{1}{\mineta} + \left(1 - \frac{t}{T}\right) \frac{1}{\maxeta} \ \Bigg| \ t \in \{0, 1, \ldots, T\} \right]. 
\end{align}

This function is motivated by the analysis on quadratic functions in~\Cref{thm:decay-main}. The reciprocal
 of the learning rate linearly interpolates from the reciprocal
 of maximal to the reciprocal
 of minimal learning rate.

%% file: main/experiments.tex
\subsection{Experiments}

\subsubsection{Setup}

\textbf{Architecture and data.}
We adopt the LLaMA architecture from \citet{touvron2023llama}, adjusting the hyperparameters to create four model sizes: 0.1B, 0.3B, 0.6B, and 1.2B. The exact hyperparameters are deferred to~\Cref{app:exp}. These models are trained on the Pile dataset \citep{gao2020pile} with a context length of 4096 and a batch size of 4M tokens.

\begin{figure}[t]
    \centering
    \begin{subfigure}[b]{0.49\textwidth}
        \centering
         \includegraphics[width=\textwidth]{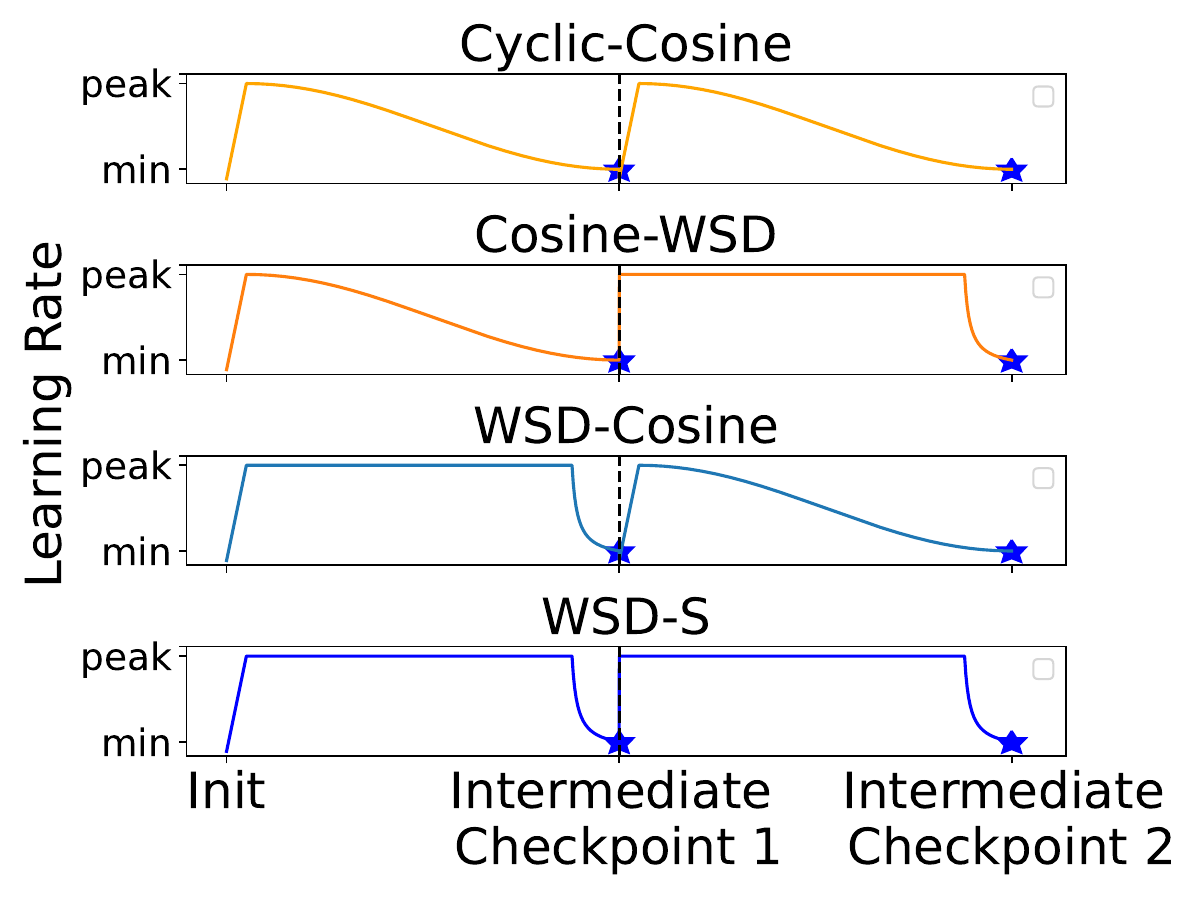}
    \end{subfigure}
    \begin{subfigure}[b]{0.49\textwidth}
        \centering
        \includegraphics[width=\textwidth]{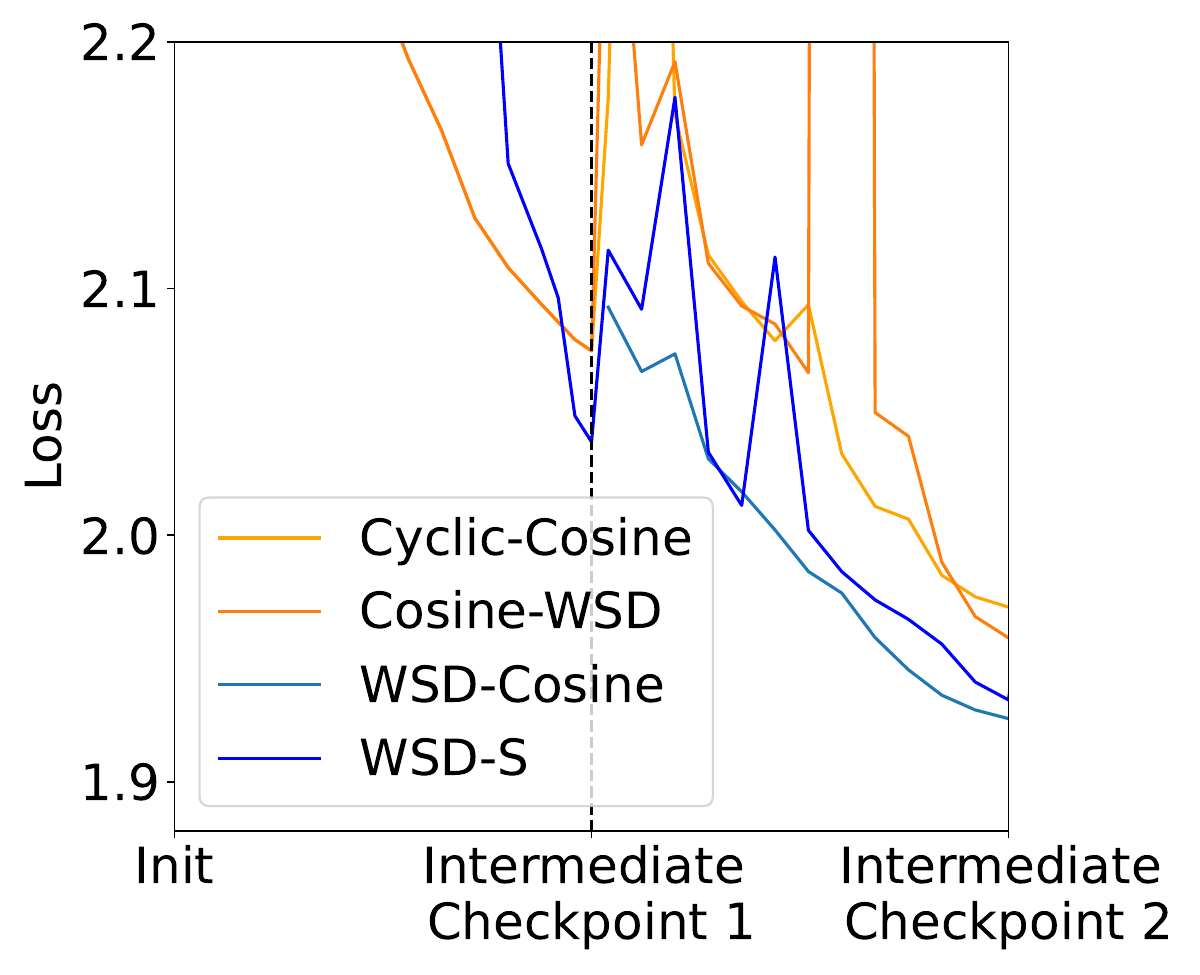}
    \end{subfigure}
    \caption{\textbf{Cosine Learning Rate Implicitly Hurts the Models for Future Continual Learning.} We show that while $\wsd$ and the cosine learning rate schedule may produce similar validation loss in a single run, a model trained with the cosine learning rate schedule is implicitly hurt compared to the model trained with $\wsd$ for future Continual learning. On the 0.6B models, after training the models for 50B tokens using both $\wsds$ and the cosine learning rate schedule, we continually train two models for another 50B tokens using both learning rates. We observe that the model trained with $\wsds$ consistently outperforms the model trained with the cosine learning rate when used as the starting point for further training.}
    \label{fig:comparison}
\end{figure}

\textbf{Implementation.} We set the batch size to 1024 and fixed the peak learning rate for the same model size for all the methods. For the 0.1B and 0.3B models, we use a peak learning rate of 6e-4, and for the 0.6B and 1.2B models, we use a peak learning rate of 4e-4. These values are chosen following current empirical practice (e.g. see \cite{groeneveld2024olmoacceleratingsciencelanguage}). We set the minimal learning rate to 0.1 of the peak learning rate. We use a TPU v3-256 model to train the model with the Levanter framework in Jax \citep{jax2018github, levanter}. The  fraction of time spent decaying is chosen to be $10\%$. The only exception is that when running $\wsd$ on the 0.3B models, we encounter a loss spike after training for 22.5B tokens and decay at the checkpoint trained for 22B tokens instead.
This change is in favor of $\wsd$ in our comparison between $\wsds$ and $\wsd$.
The detailed hyperparameters are deferred to~\Cref{app:exp}.

\subsubsection{Results}
\begin{figure}[t]
    \centering
        \begin{subfigure}[b]{0.24\textwidth}
            \centering
            \includegraphics[width=\textwidth]{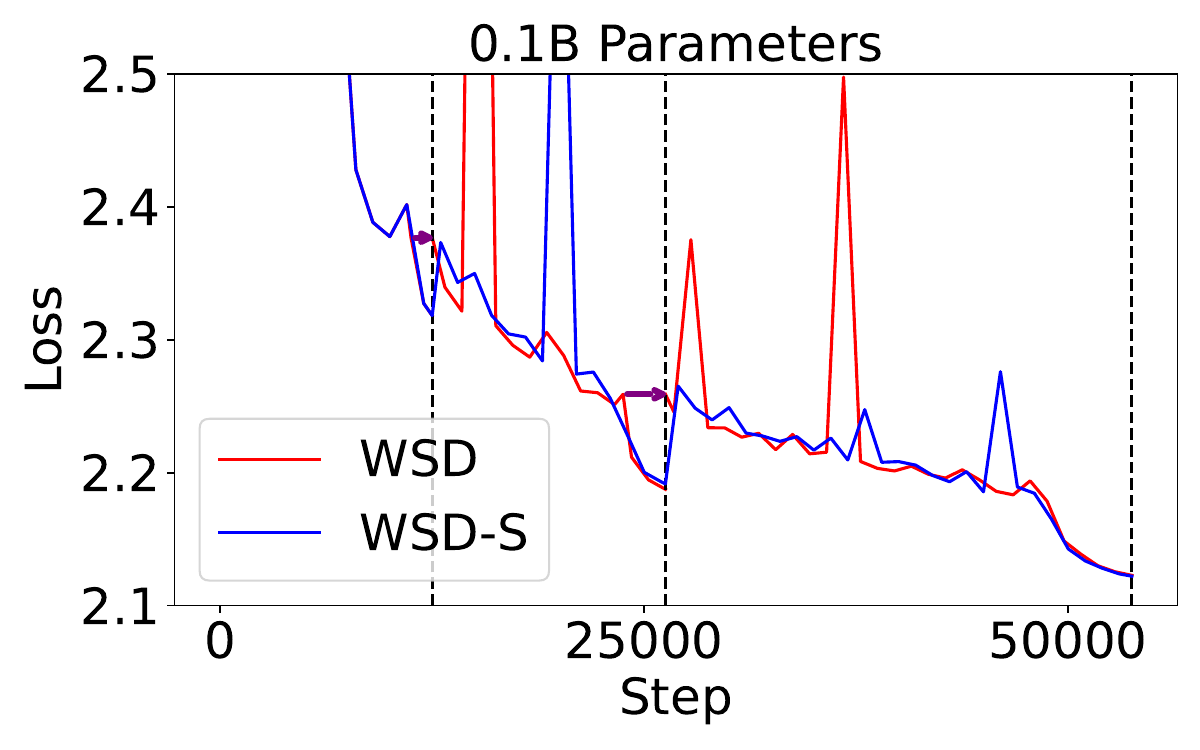}
        \end{subfigure}
        \begin{subfigure}[b]{0.24\textwidth}
            \centering
            \includegraphics[width=\textwidth]{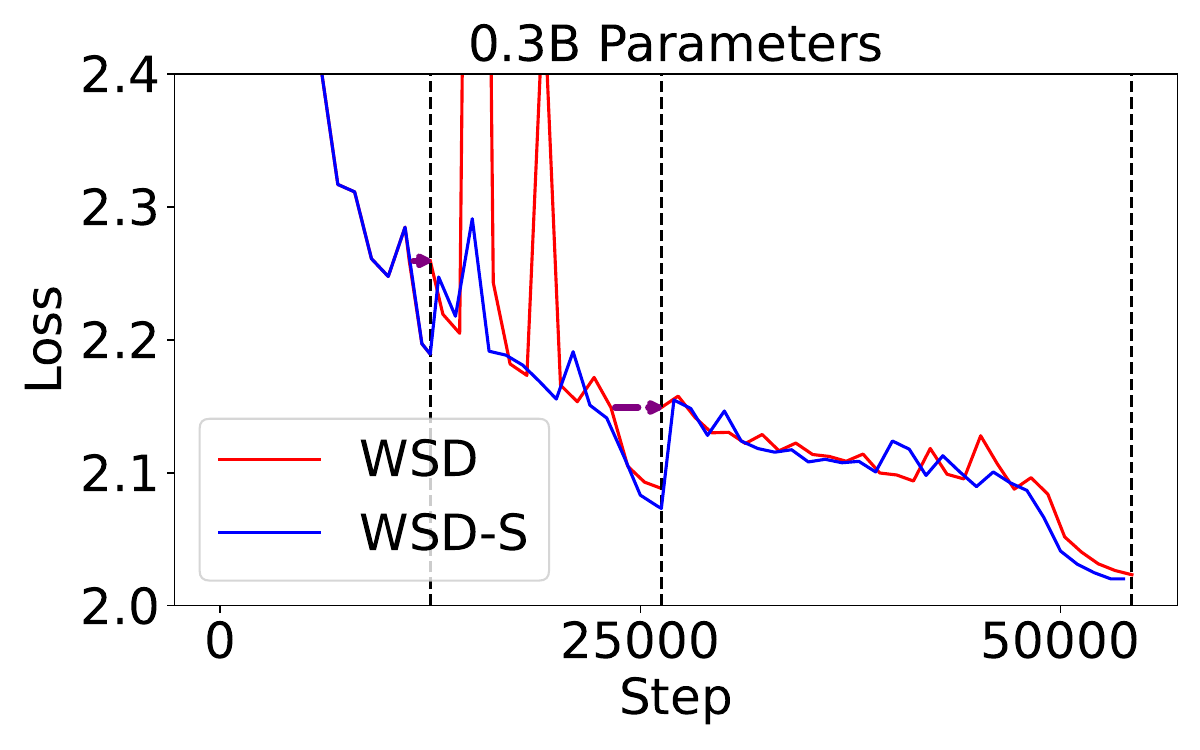}
        \end{subfigure}
        \begin{subfigure}[b]{0.24\textwidth}
            \centering
            \includegraphics[width=\textwidth]{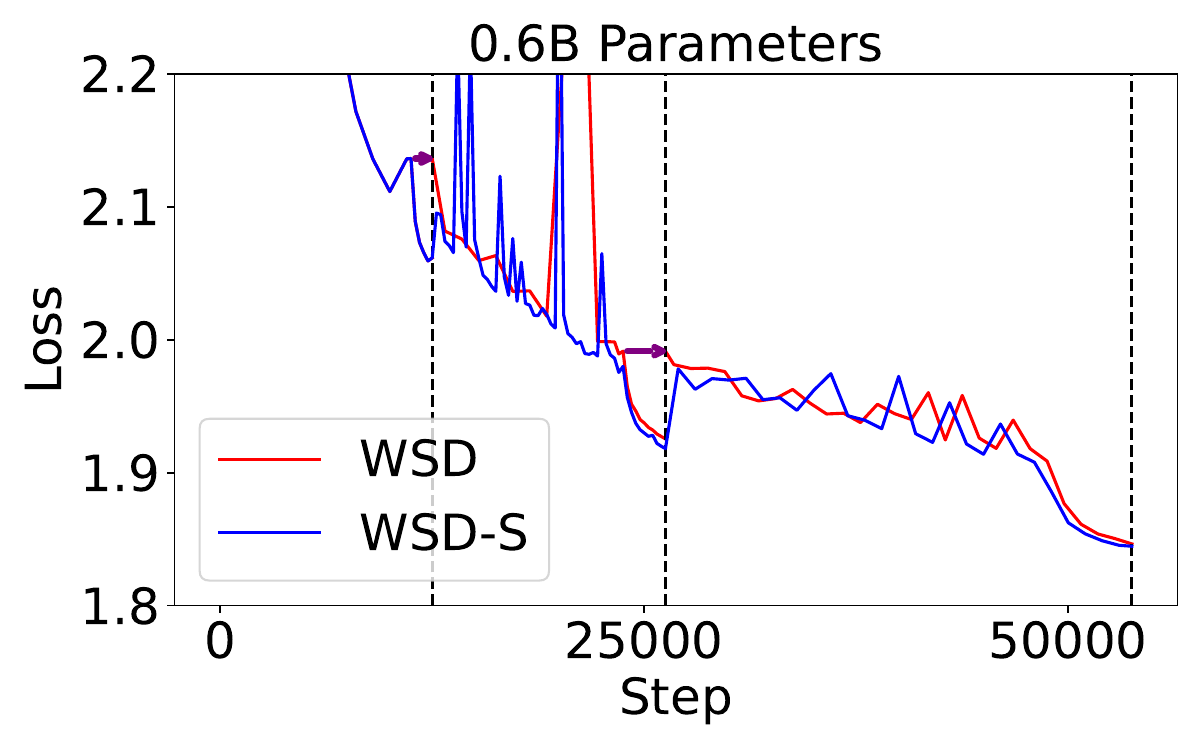}
        \end{subfigure}
        \begin{subfigure}[b]{0.24\textwidth}
            \centering
            \includegraphics[width=\textwidth]{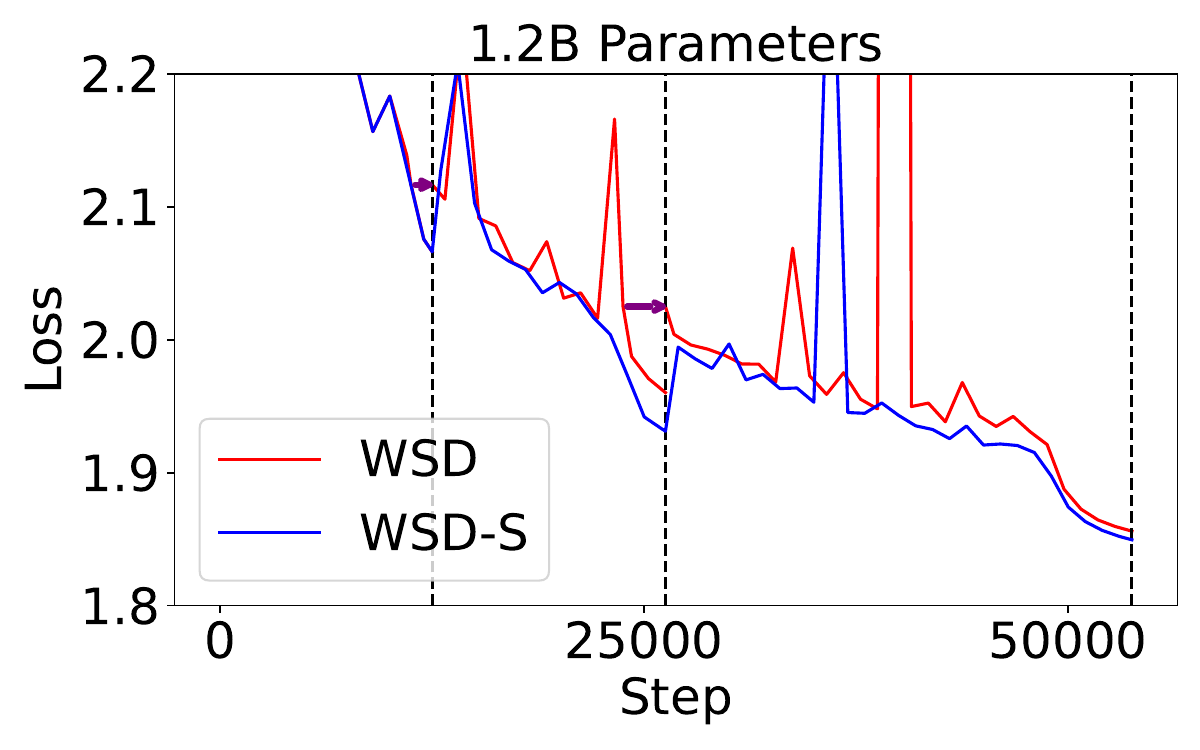}
        \end{subfigure}
        \quad
    \caption{\textbf{Comparison With $\wsd$.} We show that $\wsds$ performs favorably compared with $\wsd$ when the total computes is fixed, achieving a consistent improvement over $\wsd$ on all the model sizes when trained for approximately 200B tokens.}
    \label{fig:wsd}
    \vspace{-0.2in}
\end{figure}

\textbf{$\wsds$ performs competitively with $\cosineoracle$.} The three endpoints of $\wsds$ are set at 50B, 100B, and 200B tokens for all models. As shown in \Cref{fig:main}, $\wsds$ delivers competitive results compared to $\cosineoracle$ in a single run.

\textbf{$\wsds$ significantly outperforms $\cycliccosine$.} 
We compare the $\cycliccosine$ and the $\wsds$ on 0.6B models with a total token budget of 100B tokens. Both schedules reduce to a minimal learning rate at 50B tokens to obtain an intermediate checkpoint. Our results show that $\wsds$ outperforms $\cycliccosine$ with a significant performance gap of 4e-2 (\Cref{fig:comparison}). A common belief is that loss spiking after increasing the learning rate is the main cause of the performance loss in $\cycliccosine$. However, this belief does not explain the advantage of $\wsds$. We hypothesize that a model trained with a small learning rate for too long, as with Cosine, is implicitly hurt compared to a model trained with a large learning rate for the majority of the run, as with $\wsd$ or $\wsds$. 

To show that the model trained with $\wsd$ is more suitable for continual training, we conducted ablation studies by interchanging the schedules in the latter half of the runs to create two new learning rates (Cosine-$\wsd$ and $\wsd$-Cosine). Among the four runs, the model trained using $\wsd$ for the first half consistently achieved lower loss in continual learning, indicating that $\wsd$ produces models more suitable for continual learning, even after learning rate decay.

\textbf{$\wsds$ matches (and slightly outperforms) $\wsd$ given the same total compute.} For $\wsd$, we adopt the following comparison methodology: assuming a 10$\%$ decay portion, to get three checkpoints at 12.5k, 25k, and 50k steps, $\wsd$ then requires corresponding total steps of 12.5k, 26.25k, and 53.75k.  Hence, we examine whether $\wsds$ can output three models of matching or better performance in the same corresponding steps (see \Cref{fig:wsd}). Our results suggest that $\wsds$ consistently outperforms $\wsd$ when trained on 200B tokens and underperforms $\wsd$ only on the smallest scale experiments when we trained 0.1B models for 25k steps. As this is the smallest scale experiment, we conclude that $\wsds$ has a slight advantage over $\wsd$ when the total compute is fixed. This matches our intuition that $\wsds$ can reuse the decay phases of previous checkpoints, leading to a more efficient use of the total compute. As a simpler version of $\wsd$, $\wsds$ is more user-friendly for open-source pretrained models, allowing users to continue training the final checkpoint without needing intermediate ones given that the pretrained models are trained with $\wsd$ or $\wsds$.

\begin{figure}[t]
    \centering
    \begin{subfigure}[b]{0.28\textwidth}
        \centering
        \includegraphics[height=3cm,width=\textwidth]{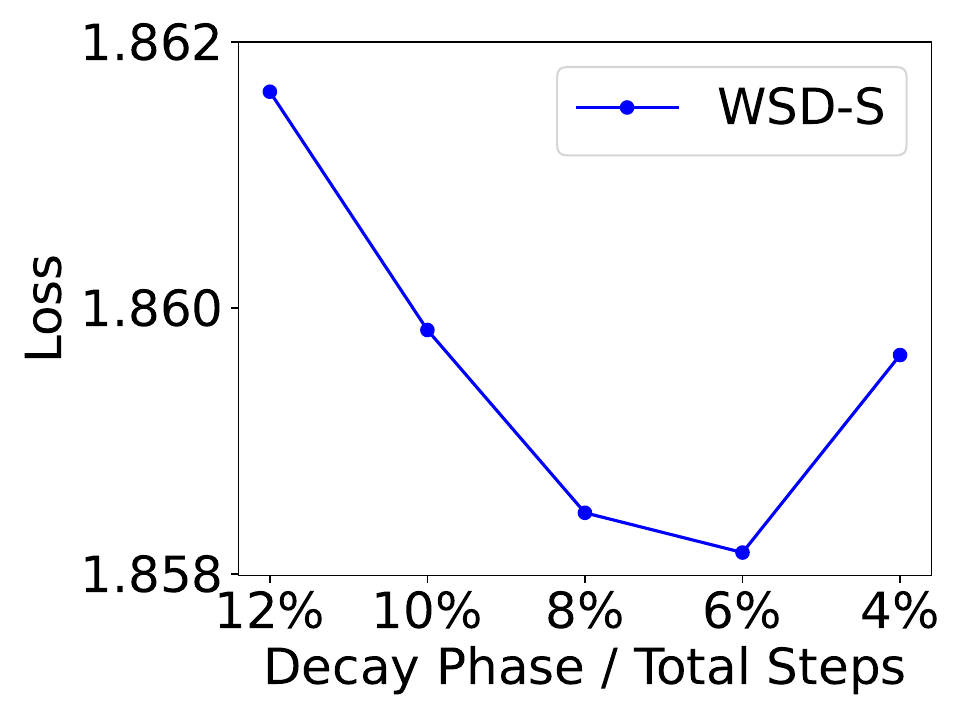}
        \caption{1.2B Models on 200B tokens}
        \label{fig:1B200B}
    \end{subfigure}
    \hfill
    \begin{subfigure}[b]{0.28\textwidth}
        \centering
        \includegraphics[height=3cm,width=\textwidth]{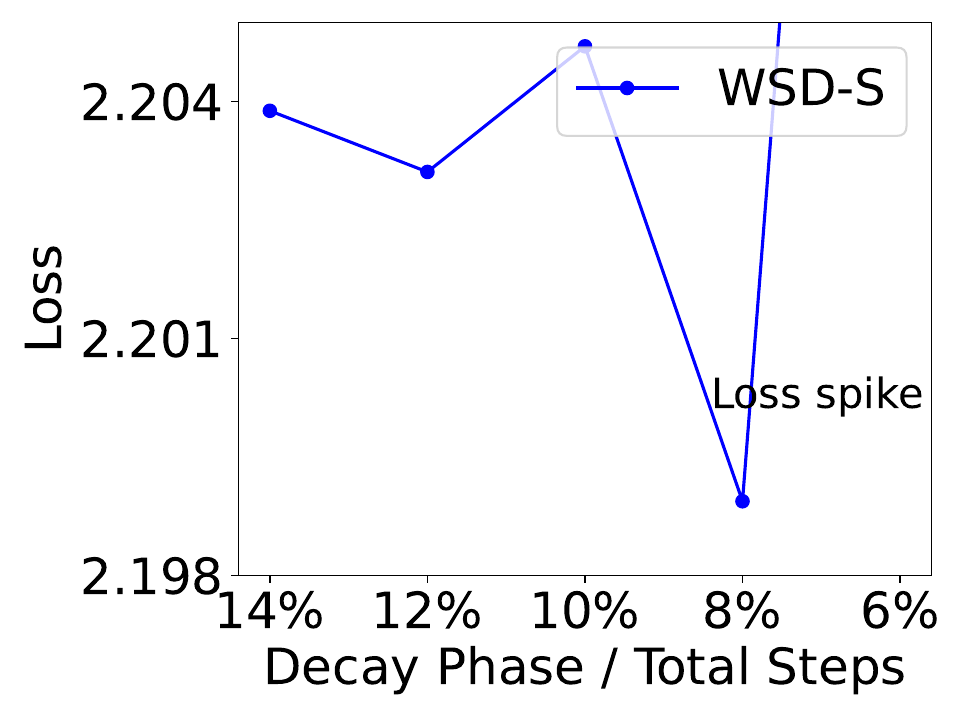}
        \caption{0.1B Models on 100B tokens}
        \label{fig:100M50B}
    \end{subfigure}
    \hfill
    \begin{subfigure}[b]{0.4\textwidth}
        \centering
        \includegraphics[height=3cm,width=\textwidth]{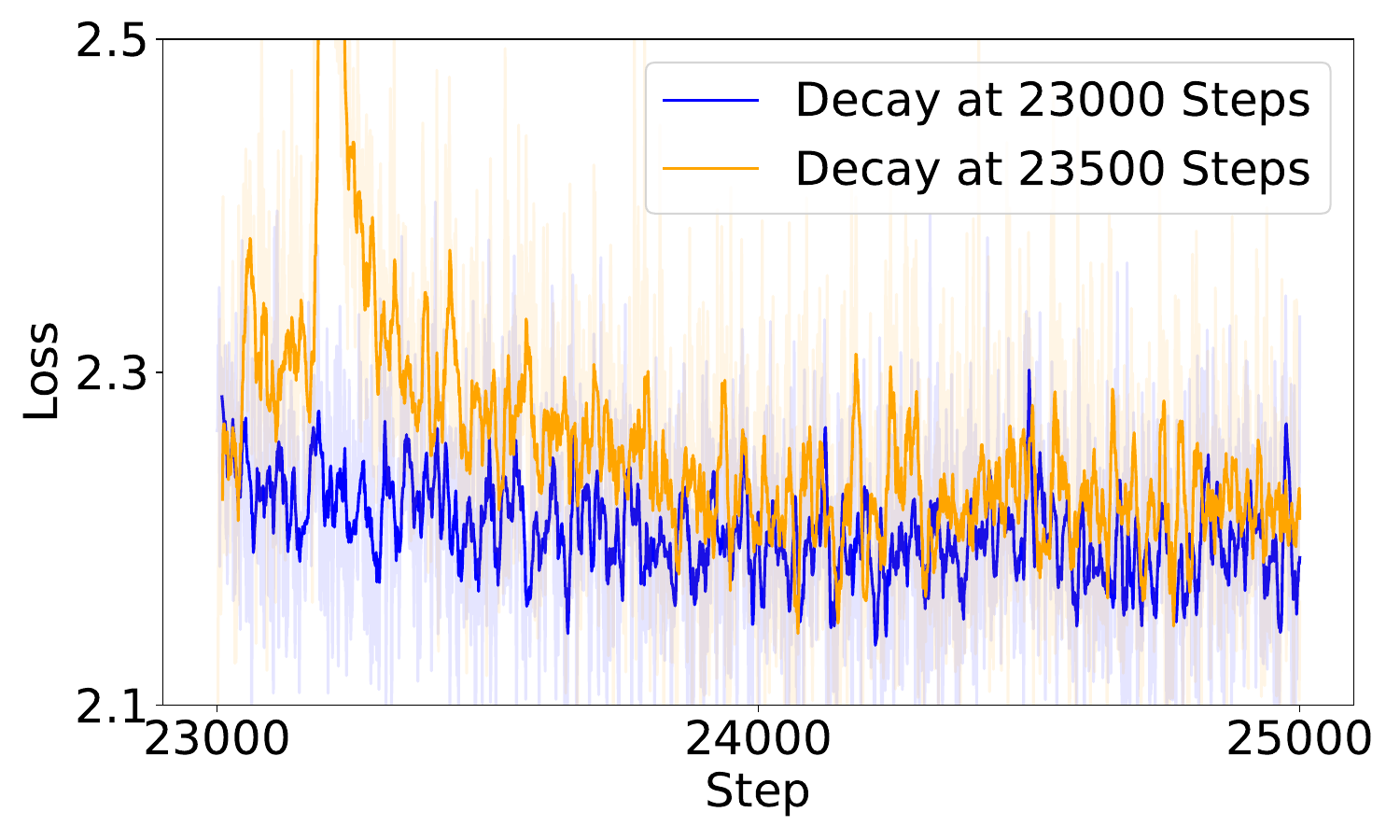}
        \caption{0.1B Models on 100B Tokens.
         Decay Near a Loss Spike (6\%)}
        \label{fig:spikeobscure}
    \end{subfigure}
    \caption{\textbf{Ablation Study on the Sensitivity of Fraction of Time Decaying}. This study examines two settings: a smaller scale with 0.1B parameters trained on 100B tokens (middle figure) and a larger scale with 0.6B parameters trained on 200B tokens (left figure). The results indicate that the final performance is similar when the decay phase is 8\%-12\% of the total training steps. However, the right figure demonstrates a significant performance loss when decaying near a loss spike. It compares two training loss curves with decay phases of 8\% and 6\% of the total compute on the 0.lB models, where the latter starts immediately after a loss spike, leading to a validation loss increase of 2e-2.}
    \label{fig:portion}
    \vspace{-0.2in}
\end{figure}

\textbf{$\wsd$ and $\wsds$ is not sensitive to the fraction of time spent decaying} We conclude with an ablation study on the fraction of time spent decaying, and the result is shown in~\Cref{fig:portion}.
The final performance matches tightly within the range of $8\%$ to $12\%$, showing small sensitivity to the choice of the decay portion. 

However, in our experiments, we observe that decaying near a loss spike can lead to a significant performance loss (\Cref{fig:portion}, right).With the large learning rate, the training runs tend to be very volatile and there are multiple loss spikes in the training (see~\Cref{fig:main}). If a decay happens closely \emph{after} a loss spike and the loss has not yet decreased to its original level, it is typical that the final validation loss will be worse by 1e-2 or even more. We observe the same phenomenon for $\wsd$, and when such a scenario happens, we suggest either running longer till the loss stabilizes or rolling back to a slightly earlier checkpoint before the loss spikes and decays from there.

%% file: main/conclusion.tex
\section{Conclusion}

In this work, we propose a river valley landscape metaphor to explain why large learning rates can make implicit progress in training that can only be revealed when the learning rate is decayed. Further, we propose a simple data model that can generate the river valley landscape, attributing the difference of sharpness in different directions to the difference in uncertainty between different tokens. Based on our theory, we propose a learning rate schedule called $\wsds$ that can produce multiple intermediate checkpoints in a single-consecutive run as good as if we prespecified the individual budgets and used a tuned cosine learning rate. The learning rate schedule involves rapidly decaying to a minimal learning rate level at these budgets and resuming to a high constant learning rate immediately afterward and keeping the learning rate constant in the rest of the runs. From training language models from 0.1B to 1.2B parameters, we show $\wsds$ performs competitively with cosine learning rate and outperforms other compute-agnostic learning rates including $\wsd$ and cyclic-cosine learning rate fixing the total compute budget.

\section*{Acknowledgement}

The authors would like to thank the support of NSF 2211780 and also would like to thank the Google TPU Research Cloud for the computing resources that enabled these experiments

%% file: appendix/proof.tex
\section{Omitted Proofs}

\subsection{Notation.}

To denote $a^Tb$ for two vectors, we will $\langle a, b \rangle$.
We will use the following function to denote the directional derivative of a mapping $F: \R^d \to \R^m$:
\begin{align*}
    \nabla F(x)[v] = \lim_{\alpha \to 0} \frac{F(x + \alpha v) - F(x)} {\alpha}.
\end{align*}

\subsection{A warmup on the quadratic function}
\label{app:quadratic}

We will first motivate the decaying function we choose~\Cref{eq:decay} using a simple example on quadratic function.

\begin{lemma}
\label{lem:fast}
Assuming that we are considering the following gradient descent
\begin{align*}
    y_{k + 1} =  y_k - \eta_k \nabla (\gamma y_k^2/2) - \eta \noise_k, \noise_k \in \Normal(0, \sigma^2 \identity).
\end{align*}
Suppose $\eta_0 = \maxeta$ and $y_0$ follows a normal distribution $\Normal(0, \maxeta \frac{\sigma^2}{2 \gamma - \maxeta \gamma^2})$. Then the following two statements hold, 
\begin{enumerate}
    \item If $\forall t, \eta_k = \eta_0$, $y_k$ will follow the same distribution as $y_0$.
    \item Consider all the learning rate schedule $\eta_k$, the following is the optimal
    \begin{align*}
       \forall t \ge 1, \eta_k^* = \frac{1}{\gamma (k - 1) + \frac{2}{\maxeta}}
    \end{align*}
    in the sense that it yields the fastest expected loss decrease. Suppose $\eta_k^*$ corresponds to iterates variables $y_k^*$, for any $\eta_k$ and its corresponding iterates variables $y_k$,
    \begin{align*}
        \E[\gamma y_k^2 / 2] \ge \E[\gamma (y_k^*)^2 / 2] = \frac{\sigma^2}{\gamma} \eta_k^*.
    \end{align*}
\end{enumerate}

\end{lemma}
\begin{proof}
We will denote $\sigma_k = \E[y_k^2]$ and assume WLOG we start decayping at step $0$. Then we will have
\begin{align*}
    \sigma_k = (1 - \eta_k \gamma)^2 \sigma_{k - 1} + \eta_k^2 \sigma^2.
\end{align*}

If we choose all $\eta_k = \maxeta$, we can directly verify that $\sigma_k = \sigma$.

If we choose $\eta_k = \frac{\sigma_{k - 1 }\gamma}{\sigma_{k - 1}\gamma^2 + \sigma^2}$ to minimize the right hand side, we will have that
\begin{align*}
&\sigma_k = \frac{\sigma_{k - 1} \sigma^2}{\sigma_{k- 1}\gamma^2 + \sigma^2}. \\
\iff     &\frac{1}{\sigma_k} = \frac{1}{\sigma_{k - 1}} + \frac{\gamma^2}{\sigma^2} = \frac{1}{\sigma_0} + \frac{\gamma^2 k}{\sigma^2}.
\end{align*}

This implies $\sigma_k = \frac{1}{\frac{1}{\sigma_0} + \frac{\gamma^2 k}{\sigma^2}}$ and plugging into $\eta_k = \frac{\sigma_{k - 1 }\gamma}{\sigma_{k - 1}\gamma^2 + \sigma^2}$ we have that
\begin{align*}
    \eta_k^* = \frac{\gamma}{\gamma^2 + \frac{\sigma^2}{\sigma_{k - 1}}}=  \frac{\gamma}{\gamma^2 k + \frac{\sigma^2}{\sigma_0}} = \frac{\gamma}{\sigma^2} \sigma_k.
\end{align*}

The optimality of $\eta_k^*$ can be easily inferred from the proof.
\end{proof}

\subsection{Landscape Analysis}
We will parameterize $\flatproj(x)$ as $v_d(x) v_d(x)^T$ and $\sharpproj(x)$ to denote $\identity - \flatproj(x)$. Throughout this section, we will assume $v_d(x)$ is continuous and pointing towards the direction of the gradient for all the $x$ on the river.
The following technical lemmas will be used repetitively in the proof.

\begin{lemma}
\label{lem:blockdiag}
Under~\Cref{assum:technical,assum:river}, the directional derivative of $\sharpproj(x)$ and $\flatproj(x)$ exist and satisfied that 
\begin{align*}
    \nabla \flatproj(x)[v] = -\nabla \sharpproj(x)[v] = \nabla v_d(x)[v] v_d(x)^T  + v_d(x) \nabla v_d(x)[v]^T.
\end{align*}
Further 
\begin{align*}
    \nabla \sharpproj(x)[v] \sharpproj a &= \langle \nabla v_d(x)[v], \sharpproj a \rangle  v_d(x), \\
    \nabla \sharpproj(x)[v] \flatproj a &= \langle v_d, \flatproj a \rangle  \nabla v_d(x)[v]. \\
    \| \nabla \sharpproj(x)[v] \|_2 &\le \frac{\gamma \epsilon}{\normbound} \|v\|_2.
\end{align*}
\end{lemma}
\begin{proof}

    As $\flatgamma$ is an unique eigenvalue of $\nabla^2 L(x + vt)$ and $\nabla^2 L(x + vt)$ is analytical with respect to $t$, by Theorem 6.1 of~\cite{kato1995perturbation}, we know that $v_d(x + vt)$ is analytical with respect to $t$. Hence, the directional derivative exists.

    The proof is by applying the chain rule and noticing that $\langle v_d(x), \nabla v_d(x)[v] \rangle = 0$ because $v_d(x)$ is always a unit vector.
\end{proof}

We will now define the projection of iterate to the river as the progress measure of the optimization dynamics.
\begin{definition}
\label{def:project}
For $U$ in~\Cref{assum:technical} and any $w \in U$, we define the following ODE as the projection flow:
\begin{align}
\label{eq:projection}
    \phi(w, 0) = w, d \phi(w, t) = -\sharpproj(w) \nabla L\left( \phi(w, t)\right) dt.
\end{align}
When $\lim_{t \to \infty} \phi(w, t)$ is well defined, we will define $\Phi(w) = \lim_{t \to \infty} \phi(w, t)$ as the projection of $w$ to the river.
\end{definition}

The following lemma ensures that the projection function is well-defined and is close to $w$:

\begin{lemma}
\label{lem:phiexists}
Under~\Cref{assum:river,assum:technical}, for any $w$ satisfying that $\mathcal{B}(w, \frac{2 \normbound}{\gamma}) \subset U$, $\Phi(w) \in \river$ exists and $\|w  - \Phi(w) \|_2 \le \frac{2 \| \sharpproj(w) \nabla L(w) \|_2}{\gamma + 2\flatgamma}$.
Moreover, movement along the projection flow decays exponentially,  $\|\sharpphi \gradphi \|_2 \le \exp(-\gamma t/2) \| \sharpproj(w) \nabla L(w) \|_2$.
\end{lemma}

\begin{proof}
We will track $\| \sharpphi \gradphi\|_2^2$ along the projection flow before $\phi(w, t)$ leaves $U$,
\begin{align*}
    &\frac{d \|\sharpphi \gradphi \|_2^2}{dt} \\=& 2 \langle \sharpphi \gradphi, \frac{d  \sharpphi}{dt} \gradphi + \sharpphi \frac{d  \gradphi}{dt}\rangle.
\end{align*}

By~\Cref{lem:blockdiag,assum:technical}, the first term can be bounded as
\begin{align*}
    & \langle \sharpphi \gradphi, \frac{d  \sharpphi}{dt} \gradphi \rangle  \\
    =&-\langle \sharpphi \gradphi,\nabla  \sharpphi[\sharpphi\gradphi ]\gradphi \rangle \\
    =&- \langle \sharpphi \gradphi, \nabla v_d (x)[\sharpphi\gradphi ] \rangle  \langle v_d, \flatproj(\phi(x,t)) \nabla L(\phi(x,t)) \rangle  \\
    \le& \| \gradphi \| \| \sharpphi \gradphi\|^2 \frac{\err \gamma}{\normbound} \le \err \gamma  \| \sharpphi \gradphi\|^2.
\end{align*}

The second term is always negative
\begin{align*}
    &\langle \sharpphi \gradphi,   \sharpphi \frac{d  \gradphi}{dt}\rangle \\
    =& -\langle\sharpphi \gradphi, \sharpphi \nabla^2 L(\phi(x,t))  \sharpphi \gradphi \rangle \\
    \le& - (\gamma + 4\flatgamma) \| \sharpphi \gradphi \|^2.
\end{align*}

Summing up the two terms,
\begin{align*}
    \frac{d \|\sharpphi \gradphi \|_2^2}{dt}  \le -(\gamma + 2 \flatgamma) \| \sharpphi \gradphi \|^2.
\end{align*}

By~\Cref{lem:expdecay}, we have $\|\sharpphi \gradphi \|_2^2 \le \exp(-(\gamma + 2 \flatgamma)t) \| \sharpproj(w) \nabla L(w) \|_2^2$.
Hence $\forall t > 0$,
\begin{align*}
\| \phi(w, t) - w \|_2 &\le \int_{0}^{\infty} \| \sharpproj(\phi(w,\tau)) \nabla L(\phi(w, \tau)) \|_2 d \tau \\
&\le \| \sharpproj(w) \nabla L(w) \|_2 \int_{0}^{\infty} \exp(-(\gamma + 2 \flatgamma)t/2)  \\
&\le \frac{2 \| \sharpproj(w) \nabla L(w) \|_2}{(\gamma + 2 \flatgamma)}.
\end{align*}
As $\mathcal{B}(w, \frac{\normbound}{2 \gamma}) \subset U$, the analysis hold along the trajectory and this shows that $\Phi(w) = \lim\limits_{t \to \infty} \phi(w, t)$ exists and that $\| \Phi(w) - w \|_2 \le  \frac{2 \| \sharpproj(w) \nabla L(w) \|_2}{(\gamma + 2 \flatgamma)}$.

Further $\Phi(w)$ satisfies that $\sharpproj(\Phi(w)) \nabla L(\Phi(w)) = 0$, and by~\Cref{assum:technical}, $\Phi(w) \in \river$.  
\end{proof}

The following lemmas focus on the properties of $\partial \Phi$.

\begin{lemma}
\label{lem:partialphiexists}
Under~\Cref{assum:river,assum:technical}, for any $w$ satisfying that $\mathcal{B}(w, \frac{2 \normbound}{\gamma}) \subset U$, $\partial \Phi(w)$ is well-defined. 
\end{lemma}

\begin{proof}
Recall that $\Phi(w) = \lim\limits_{n \to \infty} \underbrace{(\phi \circ \phi \circ \ldots \phi)}_{n \text{ times}}(w)$, as $\phi$ is differentiable (\Cref{lem:blockdiag}) and $\river$ is the fixed point of $\Phi$, by Theorem 5.1 of~\cite{Falconer1983}, we have that $\partial \Phi(w)$ is well-defined.
\end{proof}

\begin{lemma}
\label{lem:partialphi}
Under~\Cref{assum:river,assum:technical}, for any $w$ satisfying that $\mathcal{B}(w, \frac{2 \normbound}{\gamma}) \subset U$,  it holds that $\partial \Phi(w) \sharpproj(w) \nabla L(w) = 0$. Further for any $w = x(t) \in \river$, it holds that $\partial \Phi(w) \sharpproj(w) = 0, \partial \Phi(w) \frac{d x(t)}{dt} = \frac{d x(t)}{dt}$ and that for any $v$, $\partial \Phi(w) v$ aligns with $\frac{d x(t)}{dt}$.
\end{lemma}

\begin{proof}
According to~\Cref{lem:phiexists,lem:partialphi}, $\Phi, \partial \Phi$ is well-defined when $\mathcal{B}(w, \frac{2 \normbound}{\gamma}) \subset U$. Based on~\Cref{def:project}, we have that
\begin{align*}
    \forall t, \Phi( \phi(w, t)) = \Phi(w). 
\end{align*}
Hence,
\begin{align*}
     \frac{ d\Phi( \phi(w, t)) }{dt}\mid_{t = 0} = 0,
\end{align*}
Therefore,
\begin{align*}
    0 = \partial \Phi(w) \frac{d \phi(w,t)}{dt} \mid_{t = 0} = -\partial \Phi(w) \sharpproj(w) \nabla L(w).
\end{align*}

For any $w \in \river$ and any $v \in \R^d$, it holds that
\begin{align}
    0&= \frac{d \partial \Phi(w + \alpha v) \sharpproj(w + \alpha v) \nabla L(w + \alpha v)}{d \alpha}  \mid_{\alpha = 0} \notag \\
    &= \partial^2 \Phi(w)[v] \sharpproj(w) \nabla L(w) + \partial \Phi(w) \Big[ \partial \sharpproj(w)[v]\nabla L(w) + \sharpproj(w) \nabla^2 L(w) v\Big] \notag \\
    &=  \partial \Phi(w) \Big[ \partial \sharpproj(w)[v]\nabla L(w) + \sharpproj(w) \nabla^2 L(w) v\Big].\label{eq:onriverphi}
\end{align}

Define $J_w(v)$ as the projection from $v$ to $ \partial \sharpproj(w)[v]\nabla L(w) + \sharpproj(w) \nabla^2 L(w) v$. 

\begin{lemma}
\label{lem:range}
    $J_w(v)$ is a linear projection and the range of $J_w$ is the range of $\sharpproj$. 
\end{lemma}
\begin{proof}
Based on~\Cref{lem:blockdiag}, $\sharpproj \partial \sharpproj(w)[v]\nabla L(w) = \partial \sharpproj(w)[v]\nabla L(w)$. Hence the range of $J_w$ is a subspace of the range of $\sharpproj$. 
    When $v = \sharpproj(w) u \neq 0$, based on~\Cref{assum:technical},
\begin{align*}
    \| J_w(v) \|_2 \ge \| \sharpproj \nabla^2 L(w) \sharpproj(w) u \|_2 - \| \partial \sharpproj(w)[\sharpproj(w) u]\nabla L(w) \|_2 \ge \gamma \|u\|_2 - \gamma \err \| u \|_2  > 0.
\end{align*}
Hence the range of $J_w$ has a dimension no smaller than the dimension of the range of $\sharpproj(w)$. This concludes that the range of $J_w$ is the range of $\sharpproj(w)$.
\end{proof}

Hence by~\Cref{eq:onriverphi,lem:range}, it holds that for $w \in \river$, $ \partial \Phi(w) \sharpproj(w) = 0$. This shows that the range of $\partial \Phi(w)$ has dimension $1$.

Finally for any $w \in \river$, $\Phi(w) = w$. Hence,
\begin{align*}
    \frac{d \Phi(x (t)) - d x(t)}{dt} = 0. \\
    \partial \Phi(x(t)) \frac{d x(t)}{dt} = \frac{dx(t)}{dt}.
\end{align*}
Hence the range of $\partial \Phi(w)$ contains $ \frac{dx(t)}{dt}$, this concludes the proof.
\end{proof}

\begin{lemma}
\label{lem:boundpartialphi}
Under~\Cref{assum:river,assum:technical}, for any $w$ satisfying that $\mathcal{B}(w, \frac{2 \normbound}{\gamma}) \subset U$,  it holds that 
\begin{align*}
     \frac{\| \flatproj[\Phi(w)] \partial \Phi(w) \nabla L(w) -  \flatproj(w) \nabla L(w) \|_2}{ \| \flatproj(w) \nabla L(w)  \|_2} &\le 5\err,\\
     \frac{\| \sharpproj[\Phi(w)] \partial \Phi(w) \nabla L(w) \|_2}{ \| \flatproj(w) \nabla L(w)  \|_2} &\le 5\err.
\end{align*}
\end{lemma}

\begin{proof}

First, by~\Cref{lem:partialphi}, it holds that $\partial \Phi(w) \nabla L(w) = \partial \Phi(w)  \flatproj(w) \nabla L(w)$. Define
\begin{align*} 
    v &= \flatproj(w) \nabla L(w) / \|\flatproj(w) \nabla L(w)\|_2, \\
    s(t) &= \sharpproj(\phi(w,t)) \partial \phi(w,t)[v], \\
    f(t) &= \flatproj(\phi(w,t)) \partial \phi(w,t)[v],
\end{align*}
it holds that $s(0) = 0$ and $f(0) = v$ as $\phi(w,0) = w$.

We will bound the changes of $s(t)$ and $f(t)$. We will begin with calculating the time derivative of $\partial \phi(w,t)[v]$.
\begin{align}
    \frac{d \partial \phi(w,t)[v]}{dt} =& \partial\left( \frac{d\phi(w,t)}{dt} \right)[v]\notag \\
    =& -\partial\left(\sharpproj(\phi(w,t)) \nabla L(\phi(w,t)) \right)[v] \notag \\
    =& - \partial \sharpproj(\phi(w,t))[\partial \phi(w,t)[v]] \nabla L(\phi(w,t)) \notag\\
    &- \sharpproj(\phi(w,t)) \nabla^2 L(\phi(w,t)) \partial \phi(w,t)[v] \notag \\
    =& - \partial \sharpproj(\phi(w,t))[\partial \phi(w,t)[v]] \nabla L(\phi(w,t)) \notag\\
    &- \sharpproj(\phi(w,t)) \nabla^2 L(\phi(w,t)) s(t). \label{eq:doublepartialphi} 
\end{align}

We will now bound $\frac{d \| s(t) \|_2}{dt}$,
\begin{align*}
    \frac{d \| s(t) \|_2^2}{dt} =& 2\langle s(t), \frac{d s(t)}{dt} \rangle \\
    =& 2\langle s(t), \nabla \sharpproj(\phi(w,t)) [\frac{d \phi(w,t)}{dt}] \partial \phi(w,t)[v] 
    +  \sharpproj(\phi(w,t))  \frac{d \partial \phi(w,t)[v]}{dt}  \rangle \\
    =& -2 \langle s(t), \nabla \sharpproj(\phi(w,t)) [\sharpproj(\phi(w,t)) \nabla L(\phi(w,t))] \partial \phi(w,t)[v] \rangle \\
    & -2\langle s(t),  \sharpproj(\phi(w,t))  \frac{d \partial \phi(w,t)[v]}{dt}  \rangle.
\end{align*}

By~\Cref{lem:phiexists,assum:technical,lem:blockdiag}, the first term satisfies that,
\begin{align*}
    \langle s(t), \nabla \sharpproj(\phi(w,t)) [\sharpproj(\phi(w,t)) \nabla L(\phi(w,t))] \partial \phi(w,t)[v] \rangle  &\le \gamma \err \| s(t) \|_2 \| s(t) + f(t) \|_2.
\end{align*}

By~\Cref{eq:doublepartialphi,assum:technical}, the second term satisfies that,
\begin{align*}
    &\langle s(t),  \sharpproj(\phi(w,t))  \frac{d \partial \phi(w,t)[v]}{dt}  \rangle \\
    =& -\langle s(t),  \partial \sharpproj(\phi(w,t))[\partial \phi(w,t)[v]] \nabla L(\phi(w,t))  \rangle \\
    &-\langle s(t), \sharpproj(\phi(w,t)) \nabla^2 L(\phi(w,t)) s(t) \rangle \\
    \le& \gamma \err \| s(t) \|_2 \| s(t) + f(t) \|_2 - \gamma \| s(t) \|_2^2.
\end{align*}

Hence $2 \|s(t)\|_2 \frac{d \| s(t) \|_2}{dt} =\frac{d \| s(t) \|_2^2}{dt} \le - 2 \gamma \| s(t) \|_2^2  + 4\gamma \err \| s(t) \|_2 \| s(t) + f(t) \|_2$ and we can conclude that
\begin{align}
\label{eq:st}
    \frac{d\| s(t) \|_2}{dt} \le -\gamma \|s(t)\|_2 / 2 + 2 \gamma \err \|f(t)\|_2.
\end{align}

Similarly, we can provide a bound for  $\| \frac{d f(t)}{dt} \|_2$,
\begin{align*}
     \| \frac{d f(t) }{dt} \| 
    =& 2 \| \nabla \flatproj(\phi(w,t)) [\frac{d \phi(w,t)}{dt}] \partial \phi(w,t)[v] 
    +  \flatproj(\phi(w,t))  \frac{d \partial \phi(w,t)[v]}{dt} \|_2\\
    =& -2 \|\nabla \flatproj(\phi(w,t)) [\sharpproj(\phi(w,t)) \nabla L(\phi(w,t))] \partial \phi(w,t)[v]  \\
    & +  \flatproj(\phi(w,t))  \frac{d \partial \phi(w,t)[v]}{dt}  \|_2.
\end{align*}
By~\Cref{lem:phiexists,assum:technical,lem:blockdiag}, the first term satisfies that,
\begin{align*}
    &\| \nabla \flatproj(\phi(w,t)) [\sharpproj(\phi(w,t)) \nabla L(\phi(w,t))] \partial \phi(w,t)[v]  \| \\  =& \|  \nabla \sharpproj(\phi(w,t)) [\sharpproj(\phi(w,t)) \nabla L(\phi(w,t))] \partial \phi(w,t)[v] \| \\\le& \gamma \err \exp(-\gamma t /2) \| f(t) \|_2 \| s(t) + f(t) \|_2 .
\end{align*}
By~\Cref{lem:phiexists,assum:technical,lem:blockdiag}, the second term satisfies that,
\begin{align*}
    & \|   \flatproj(\phi(w,t))  \frac{d \partial \phi(w,t)[v]}{dt} \| \\
    =& \| \flatproj(\phi(w,t)) \partial \flatproj(\phi(w,t))[\partial \phi(w,t)[v]] \nabla L(\phi(w,t))  \rangle  \| \\
    =& \|\flatproj(\phi(w,t))  \partial \flatproj(\phi(w,t))[\partial \phi(w,t)[v]] \sharpproj(\phi(w,t)) \nabla L(\phi(w,t))  \rangle \| \\
    \le& \gamma \err  \exp(-\gamma t /2) \| f(t) \|_2 \| s(t) + f(t) \|_2.
\end{align*}
Hence we can conclude that
\begin{align}
\label{eq:ft}
    \| \frac{d f(t)}{dt}  \|_2  \le  2 \gamma \err  \exp(-\gamma t /2)  (\|f(t)\|_2 + \| s(t)\|_2).
\end{align}

By~\Cref{eq:st}, it holds that
\begin{align*}
    \frac{d (\exp(\gamma t /2)\| s(t)\|_2)}{dt} = \gamma \exp(\gamma t / 2) \| s(t)\|_2 / 2 + \exp(\gamma t / 2) \frac{d \| s(t)\|_2}{dt} \le 2\gamma \err  \exp(\gamma t / 2) \| f(t) \|_2.
\end{align*}

Integrating the above equation from $0$ to $t$, and we have 
\begin{align}
\label{eq:connectsf}
    \| s(t) \|_2 \le 2\gamma \err \int_{0}^{t}  \exp(\gamma (\tau - t) / 2) \| f(\tau) \|_2 d \tau.
\end{align}

By~\Cref{eq:ft,eq:connectsf}, we have that
\begin{align*}
     \Big | \frac{d\| f(t) \|_2}{dt} \Big |   \le 2 \gamma \err  \exp(-\gamma t /2)  \|f(t)\|_2 +  4\gamma^2 \err^2 \int_{0}^{t}  \exp(\gamma (\tau - 2t)/ 2) f(\tau) d \tau.
\end{align*}

This suggests that
\begin{align*}
    \|f(T)\|_2 \le 1 + 2 \gamma \err \int_{0}^T  \exp(-\gamma t /2)  \|f(t)\|_2 dt  +  4\gamma^2 \err^2 \int_{0}^T \int_{0}^{t}  \exp(\gamma (\tau - 2t) / 2) f(\tau) d \tau dt.
\end{align*}

Define $M(t) = \sup_{0 \le \tau \le t} f(t)$, then it holds that
\begin{align*}
    \| M(T) \|_2 &\le 1 + \| M(T) \|_2 \left(2 \gamma \err \int_{0}^T  \exp(-\gamma t /2) dt +  4\gamma^2 \err^2 \int_{0}^T \exp(-\gamma t/2)  \int_{0}^{t}  \exp(\gamma  (\tau - t)/ 2)  d \tau dt \right). \\
    &\le 1 +  \| M(T) \|_2 \left( 4 \err + 16 \err^2 \right).
\end{align*}

This implies that $\forall t, \|f(t)\|_2 \le \| M(t) \|_2 \le \frac{1}{ 1- 4 \err - 16 \err^2} \le 1 + 5 \err$. By~\Cref{eq:connectsf}, this suggests that $\| s(t)\|_2 \le 4 \err ( 1 + 5 \err) \le 5 \err$. Finally, returning to~\Cref{eq:ft}, we have that
\begin{align*}
     \| \frac{f(t)}{dt} \|_2 &\le  2 \gamma \err  \exp(-\gamma t /2)  (\|f(t)\|_2 + \| s(t)\|_2) \\
     &\le  2 \gamma \err  \exp(-\gamma t /2)  (1 + 10 \err).
\end{align*}

Hence $\| f(t) - f(0)\|_2 \le 2 \int_{0}^{\infty} 2 \gamma \err  \exp(-\gamma t /2)  (1 + 10 \err)\le 2 \err (1 + 10 \err) \le 5 \err$.

We have that
\begin{align*}
     \frac{\| \flatproj[\Phi(w)] \partial \Phi(w) \nabla L(w) -  \flatproj(w) \nabla L(w) \|_2}{ \| \flatproj(w) \nabla L(w)  \|_2} = \lim_{t \to \infty} \| f(t) - f(0) \|_2 \in [0, 5\err],
\end{align*}
and 
\begin{align*}
     \frac{\| \sharpproj[\Phi(w)] \partial \Phi(w) \nabla L(w) \|_2}{ \| \flatproj(w) \nabla L(w)  \|_2} = \lim_{t \to \infty} \| s(t) \|_2 \in [0, 5\err],
\end{align*}
The proof is then complete.
\end{proof}

The following lemma generalizes~\Cref{lem:boundpartialphi} to general direction instead of $\nabla L(w)$.

\begin{lemma}
\label{lem:boundpartialphigen}
Under~\Cref{assum:river,assum:technical}, for any $w$ satisfying that $\mathcal{B}(w, \frac{2 \normbound}{\gamma}) \subset U$,  it holds that 
\begin{align*}
     \frac{\| \flatproj[\Phi(w)] \partial \Phi(w) u -  \flatproj(w) u \|_2}{ \| \flatproj(w) \nabla L(w)  \|_2} &\le 5\err,\\
     \frac{\| \sharpproj[\Phi(w)] \partial \Phi(w) u  \|_2}{ \| \flatproj(w) u \|_2} &\le 5\err.
\end{align*}
\end{lemma}

\begin{proof}
    We only need to notice that $\flatproj(w) u $ aligns with $v_d(\nabla^2 L(w))$. Hence, it always holds that
    \begin{align*}
         \flatproj(w) u &= \flatproj(w) \nabla L(w) \frac{\langle \flatproj(w) \nabla L(w), \flatproj(w) u\rangle}{\| \flatproj(w) \nabla L(w) \|_2^2}, \\
         \partial \Phi(w) u &=  \partial \Phi(w) \flatproj(w) u = \partial \Phi(w) \nabla L(w)  \frac{\langle \flatproj(w) \nabla L(w), \flatproj(w) u\rangle}{\| \flatproj(w) \nabla L(w) \|_2^2}.
    \end{align*}
The proof is then complete.
\end{proof}

The following lemma states that the angle between the gradient and the tangent direction is small for any point on the river.

\begin{lemma}
\label{lem:alignspace}
For any $w \in \river$, it holds that 
\begin{align*}
    \| \riverproj(w) \nabla L(w)  - \nabla L(w) \|_2 \le 4\err  \| \riverproj(w) \nabla L(w)   \|_2.
\end{align*}
\end{lemma}
\begin{proof}

Assume $w = x(T)$, we will denote $\riverproj(w) \nabla L(w) $ by $v$.

It holds that
    \begin{align*}
        \nabla \big( \sharpproj(w) \nabla L(w) \big)[v] = 0,
    \end{align*}
    which can be simplified to
    \begin{align*}
        \sharpproj(w) \nabla^2 L(w) v + \nabla \sharpproj(w)[v] \nabla L(w) = 0.
    \end{align*}
    The first term satisfies that $\|\sharpproj(w) \nabla^2 L(w) v \|_2 \ge \gamma \| \sharpproj(w) v\|$ and the second term satisfies that $\|  \nabla \sharpproj(w)[v] \nabla L(w)\|_2 \le \gamma \err \| v\|$. This then suggests $\| \sharpproj(w) v \|_2 \le \err \| v \|_2$. 
    
    Therefore $\| \flatproj(w) v\|_2  \ge (1 - \err) \| v \|_2$. As 
    \begin{align*}
        v = \frac{d x(t)}{dt} \mid_{t = T} = -\riverproj \left(w\right) \nabla L(w) 
    \end{align*}

   We know that $\Big | v_d^\top \riverproj(w) v_d \Big | \ge (1 - \err) \|  \riverproj \left(w\right) v_d\|_2$, which suggests that $\Big | v_d^\top \frac{ \riverproj(w) v_d} {\|  \riverproj(w) v_d \|_2} \Big | \ge ( 1- \err)$. Hence we can conclude that $\| \riverproj(w) v_d \|_2 \ge (1 - \err)$. Hence, we know that
   \begin{align*}
       \| v + \nabla L(w) \|_2 \le \sqrt{1 - (1- \err)^2} \| \nabla L(w) \|_2\le 2 \err \| \nabla L(w) \|_2 \le \frac{2 \err}{1 - \err} \| v \|_2 \le 4 \err \| v\|_2.
   \end{align*}
   This concludes the proof.
\end{proof}

The next lemma states that $\flatproj(w) \nabla L(w)$ and $\nabla L(\Phi(w))$ is always close.

\begin{lemma}
\label{lem:simgrad}
Under~\Cref{assum:river,assum:technical}, for any $w$ satisfying that $\mathcal{B}(w, \frac{2 \normbound}{\gamma}) \subset U$,  it holds that 
\begin{align*}
    \| \flatproj(w) \nabla L(w) - \nabla L(\Phi(w)) \|_2 \le (\gamma \err  + \flatgamma) \| w - \Phi(w)\|_2.
\end{align*}
\end{lemma}

\begin{proof}

By~\Cref{lem:phiexists}, the line segment from $\Phi(w)$ to $w$ lies in $U$. By~\Cref{assum:technical,lem:phiexists},
\begin{align*}
    &\| \flatproj(w) \nabla L(w) - \flatproj(\Phi(w)) \nabla L(\Phi(w)) \|_2 \\
    =& \| \int_{0}^1 \nabla \flatproj\left( \Phi\left(w\right) + t\left(w - \Phi\left(w\right)\right)\right)\left[w - \Phi(w)\right]\nabla L( \Phi\left(w\right) + t\left(w - \Phi\left(w\right)\right)) dt \\
    &+ \int_{0}^1  \flatproj\left( \Phi\left(w\right) + t\left(w - \Phi\left(w\right)\right)\right) \nabla^2 L( \Phi\left(w\right) + t\left(w - \Phi\left(w\right)\right)) (w - \Phi(w)) dt \|_2 \\
    \le& (\gamma \err  + \flatgamma) \| w - \Phi(w)\|_2.
\end{align*}
This concludes the proof.
\end{proof}

The final theorem states that when $w$ is near the river, the movement of its projection has a similar value as the inherent speed at the river.

\begin{lemma}
\label{lem:similarmove}
    Under~\Cref{assum:river,assum:technical}, when $\| w- \Phi(w) \|_2 \le \frac{10 \err \| \flatgrad\|_2}{\gamma + \flatgamma}$,
    \begin{align*}
    \|  \flatproj(w) \nabla L(w) +  \frac{d x(\tau)}{d\tau} \mid_{\tau = T} \|_2 &\le 16\err \|\frac{d x(\tau)}{d\tau} \mid_{\tau = T}\|_2.\\ 
        \| \partial \Phi(w) \nabla L(w) +  \frac{d x(\tau)}{d\tau} \mid_{\tau = T} \|_2 &\le 30\err \|\frac{d x(\tau)}{d\tau} \mid_{\tau = T}\|_2.
    \end{align*}
\end{lemma}

\begin{proof}

By~\Cref{lem:boundpartialphi}
\begin{align*}
    \| \partial \Phi(w) \nabla L(w) - \flatproj(w) \nabla L(w) \|_2 \le 10\err \| \flatproj(w) \nabla L(w) \|_2.
\end{align*}

    Combining~\Cref{lem:simgrad} and $\| w- \Phi(w) \|_2 \le \frac{10 \err \| \flatgrad\|_2}{\gamma + \flatgamma}$, we have that
    \begin{align*}
        \|  \flatproj(w) \nabla L(w) - \nabla L(\Phi(w)) \|_2 \le (\gamma \err + \flatgamma) \| w- \Phi(w) \| \le  10\err \| \flatproj(w) \nabla L(w) \|_2. 
    \end{align*}

By~\Cref{lem:alignspace}, let $v =  \frac{d x(\tau)}{d\tau} \mid_{\tau = T}$,
\begin{align*}
    \| v  +  \nabla L(\Phi(w)) \|_2 \le 4\err  \|v   \|_2.
\end{align*}

Combining the three inequalities, we have that
\begin{align*}
    \| \flatproj(w) \nabla L(w) -\nabla L(\Phi(w))  \|_2 &\le \frac{10 \err}{1 - 10 \err}  \| \nabla L(\Phi(w)) \|_2 \le \frac{ 1 + 4\err }{1 - 10\err} 10 \err \| v\|_2 \le 12 \err \| v \|_2. 
\end{align*}

This suggests that
\begin{align*}
    \| \flatproj(w) \nabla L(w) - v\|_2 \le | \flatproj(w) \nabla L(w) -\nabla L(\Phi(w))  \|_2 + \| v  +  \nabla L(\Phi(w)) \|_2 \le 16 \err \| v\|_2.
\end{align*}

Hence 
\begin{align*}
    &\|  v +  \partial \Phi(w) \nabla L(w)  \|_2 \\ \le &
     \| \flatproj(w) \nabla L(w) - v\|_2 + \|  \flatproj(w) \nabla L(w) - \nabla L(\Phi(w)) \|_2 \\
     \le& 16 \err \| v\|_2 + 10 \err \|\flatproj(w) \nabla L(w) \|_2 \\
     \le& 30 \err \| v\|_2.
\end{align*}
This concludes the proof.
\end{proof}

\subsection{Proof of \Cref{thm:gftracksriver}}
\label{app:gf}
We will consider the following gradient flow:
\begin{align}
\label{eq:gf}
    d w(t) = - \nabla L(w(t)) dt, w(0)  \in V.
\end{align}

We will first prove that along the gradient flow trajectory, it holds that $\| \sharpproj(w) \nabla L(w)\|_2$ is bounded. 

\begin{lemma}
\label{lem:gfsg}
    Under~\Cref{assum:river,assum:technical}, along the gradient flow~\Cref{eq:gf}, it holds that for $t \ge {2\log(2 \normbound / (\err\normbound_{\min})) }/{\gamma}$, $\| \sharpgrad\|_2 \le 2 \err \| \flatgrad\|_2$.
\end{lemma}
\begin{proof}
    We will first compute how fast $\| \sharpgrad \|_2$ can change
    \begin{align*}
        \frac{d \|\sharpgrad\|_2^2}{dt} =& -2 \langle \sharpgrad, \partial \sharpproj(w(t))[\nabla L(w(t))] \nabla L(w(t)) \rangle \\ &-2 \langle \sharpgrad, \sharpproj(w(t)) \nabla^2 L(w(t)) \nabla L(w(t))\rangle \\
    \end{align*}
    By~\Cref{lem:blockdiag,assum:technical}, the first term satisfies,
    \begin{align*}
        &-2\langle \sharpgrad, \partial \sharpproj(w(t))[\nabla L(w(t))] \nabla L(w(t)) \rangle \\
        =& -2\langle \sharpgrad, \partial \sharpproj(w(t))[\nabla L(w(t))] \flatproj L(w(t)) \nabla L(w(t)) \rangle \\
        \le& 2 \err \gamma \|\sharpgrad\|_2 \|\flatgrad\|_2
    \end{align*}
    By~\Cref{assum:technical}, the second term satisfies,
    \begin{align*}
        &-2 \langle \sharpgrad, \sharpproj(w(t)) \nabla^2 L(w(t)) \nabla L(w(t))\rangle \\
        &\le -2(\gamma + \flatgamma) \| \sharpgrad\|_2^2.
    \end{align*}
    Hence,
    \begin{align}
    \label{eq:sgf}
        \frac{d \|\sharpgrad\|_2}{dt} \le \err \gamma \|\flatgrad\|_2 - (\gamma + \flatgamma) \|\sharpgrad\|_2.
    \end{align}

    We then consider the corresponding $\flatgrad$.

    \begin{align*}
 \frac{d \|\flatgrad\|_2^2}{dt} =& -2 \langle \flatgrad, \partial \flatproj(w(t))[\nabla L(w(t))] \nabla L(w(t)) \rangle \\ &-2 \langle \flatgrad, \flatproj(w(t)) \nabla^2 L(w(t)) \nabla L(w(t))\rangle \\
    \end{align*}

    By~\Cref{lem:blockdiag,assum:technical}, the first term satisfies,
    \begin{align*}
        &-2\langle \flatgrad, \partial \flatproj(w(t))[\nabla L(w(t))] \nabla L(w(t)) \rangle \\
        =& -2\langle \flatgrad, \partial \flatproj(w(t))[\nabla L(w(t))] \sharpproj L(w(t)) \nabla L(w(t)) \rangle \\
        \le& 2 \err \gamma \|\sharpgrad\|_2 \|\sharpgrad\|_2
    \end{align*}
    By~\Cref{assum:technical}, the second term satisfies,
    \begin{align*}
        &-2 \langle \flatgrad, \flatproj(w(t)) \nabla^2 L(w(t)) \nabla L(w(t))\rangle \\
        &\le 2\flatgamma \| \flatgrad\|_2^2.
    \end{align*}

    Hence, we have that
    \begin{align}
    \label{eq:fgf}
         \Big| \frac{d \|\flatgrad\|_2}{dt} \Big| \le \err \gamma \|\sharpgrad\|_2 + \flatgamma \| \flatgrad\|_2.
    \end{align}

    Choose $\alpha_\err = \frac{1 - \sqrt{1 - 4 \err^2}}{2 \err} < 1.5 \err$ as the solution to the quadratic equation $\err \alpha^2 -  \alpha + \err = 0$.

    Then combining~\Cref{eq:sgf,eq:fgf}, it holds that
    \begin{align*}
        &\frac{d \left( \| \sharpgrad\|_2 - \alpha_{\err} \|\flatgrad\|_2\right)} {dt} \\ \le&  \gamma ( -1 + \err \alpha_{\err}) \|\sharpgrad \|_2 + \err \gamma \| \flatproj \|_2\\
        &-\flatgamma \left(| \sharpgrad\|_2 - \alpha_{\err} \|\flatgrad\|_2 \right).
    \end{align*}

    Notice that
    \begin{align*}
         \frac{( -1 + \err \alpha_{\err})}{\err} = \frac{-1}{\alpha_{\err}}
    \end{align*}

    Hence,
    \begin{align*}
         &\frac{d \left( \| \sharpgrad\|_2 - \alpha_{\err} \|\flatgrad\|_2\right)} {dt}\\ \le& -( \gamma(1 - \err \alpha_{\err}) + \flatgamma ) \left( \| \sharpgrad\|_2 - \alpha_{\err} \|\flatgrad\|_2\right).
    \end{align*}

    By~\Cref{lem:expdecay}, this suggests that
    \begin{align*}
         &\| \sharpgrad\|_2 - \alpha_{\err} \|\flatgrad\|_2 \\ \le& \exp(-( \gamma(1 - \err \alpha_{\err}) + \flatgamma )t) \left( \| \sharpproj(w(0)) \nabla L(w(0))\|_2 \right) \\
         \le& \exp(-\gamma t/2)  \| \sharpproj(w(0)) \nabla L(w(0))\|_2.
    \end{align*}

    Hence, 
    \begin{align*}
        \| \sharpgrad\|_2 \le 1.5\err \|\flatgrad\|_2 + \exp(-\gamma t/2) \left( \| \sharpproj(w(0)) \nabla L(w(0))\|_2 \right).
    \end{align*}

\end{proof}

\begin{lemma}
\label{lem:gfclose}
    Under~\Cref{assum:river,assum:technical}, along the gradient flow~\Cref{eq:gf}, it holds that $w(t) \in U$ and $\| w(t)- \Phi(w(t)) \|_2 \le \frac{4 \err \| \flatgrad\|_2 + 2 \exp(-\gamma t/2) \Delta }{\gamma + \flatgamma}$.
\end{lemma}
\begin{proof}
    This is a direct combination of~\Cref{lem:gfsg,lem:phiexists}.
\end{proof}

\begin{lemma}
\label{lem:gfspeedcalc}
 Under~\Cref{assum:river,assum:technical}, along the gradient flow~\Cref{eq:gf}, if $T(t)$ satisfies $x(T(t)) = \Phi(w(t))$, then 
 \begin{align*}
     \frac{d T(t)}{d t} \in [1 - 30 \err, 1 + 30 \err].
 \end{align*}
\end{lemma}

\begin{proof}
As $T(t)$ satisfies $x(T(t)) = \Phi(w(t))$, taking derivative on both sides yield,
\begin{align*}
    \frac{d T(t)}{d t} \frac{d x(\tau)}{d \tau} \mid_{\tau = T(t)} = -\partial \Phi(w(t)) \nabla L(w(t)).
\end{align*}

By~\Cref{lem:gfclose,lem:simgrad}, it holds that 
\begin{align*}
    \| \partial \Phi(w(t)) \nabla L(w(t)) - \frac{d x(\tau)}{d \tau} \mid_{\tau = T(t)}\|_2 \le 30 \err   \frac{d x(\tau)}{d \tau} \mid_{\tau = T(t)}.
\end{align*}

We then have that
\begin{align*}
    | \frac{d T(t)}{d t} - 1 | \le 30 \err,
\end{align*}
which concludes the proof.
\end{proof}

\begin{proof}[Proof of~\Cref{thm:gftracksriver}]
The proof is a direct combination of~\Cref{lem:gfclose,lem:gfspeedcalc}.
\end{proof}

\subsection{Proof of \Cref{thm:gdtracksriver}}
\label{app:gd}
We will consider the following gradient descent:
\begin{align}
\label{eq:gd}  
    w_{k + 1} - w_k = - \eta \nabla L(w_k) , w_0 \in \river.
\end{align}

We will track the changes of $\flatproj(w_k) \nabla L(w_k)$ and $\sharpproj(w_k) \nabla L(w_k)$, for simplicity, we will denote them us $fg(k)$ and $sg(k)$. Further, we will use the following denotation 
\begin{align*}
    w_{k, \tau} = (1 - \tau) w_k + \tau w_{k + 1}
\end{align*}
    We will first prove some lemmas bounding the difference between gradient and projections at different points.

    \begin{lemma}
    \label{lem:stayinU}
    Under~\Cref{assum:river,assum:technical}, when $w_k \in V$, $\forall \tau \in (0,1)$, $w_{k, \tau} \in U$.
    \end{lemma}

    \begin{proof}
        It holds that,
        \begin{align*}
            \| w_k - w_{k, \tau} \|_2 \le \eta \normbound \le \frac{\normbound}{2\gamma}.
        \end{align*}
    \end{proof}

    \begin{lemma}
    \label{lem:moveps}
         Under~\Cref{assum:river,assum:technical}, when $w_k \in V$, $\forall \tau, \tau' \in [0, 1]$, it holds that 
        \begin{align*}
            \|  \sharpproj(w_{k, \tau})  - \sharpproj(w_{k, \tau'}) \|_2 \le \eta \gamma \err.
        \end{align*} 
    \end{lemma}

    \begin{proof}
        According to~\Cref{lem:stayinU}, it holds that $w_{k, \tau}, w_{k, \tau'} \in U$.
        Assume without loss of generality $\tau > \tau'$,
        \begin{align*}
            \| \sharpproj(w_{k, \tau})  - \sharpproj(w_{k, \tau'}) \| =&
            \| \int_{\tau}^{\tau'} \nabla \sharpproj(w_{k, \tau''})[\eta \nabla L(w)] d\tau''  \|_2 \\
            \le& \int_{\tau'}^\tau \|   \nabla \sharpproj(w_{k, \tau''})[\eta \nabla L(w)]  \|_2 d\tau''  \\
            \le & \eta \gamma \err.
        \end{align*}
    \end{proof}

    \begin{lemma}
    \label{lem:moveg}
        Under~\Cref{assum:river,assum:technical}, when $w_k \in V$, $\forall \tau \in (0, 1)$, it holds that 
        \begin{align*}
            \|  \sharpproj(w_{k, \tau}) \nabla L(w_{k, \tau}) - \sharpproj(w_{k, \tau}) \nabla L(w_k) \|_2 \le \eta \maxgamma \|  sg(k) \|_2 + 2\eta^2 \maxgamma \gamma \err \| \nabla L(w_k)\|_2.
        \end{align*}
    \end{lemma}

    \begin{proof}
        According to~\Cref{lem:stayinU}, it holds that $w_{k, \tau}, w_{k, \tau'} \in U$.
        Define $g(\tau') = \| \sharpproj(w_{k, \tau}) \nabla L(w_{k, \tau'}) - \sharpproj(w_{k, \tau}) \nabla L(w_k) \|_2  $, then by  Lagrange's Mean Value Theorem, there exists $\tau'$, such that
        \begin{align*}
            &\|  \sharpproj(w_{k, \tau}) \nabla L(w_{k, \tau}) - \sharpproj(w_{k, \tau}) \nabla L(w_k) \|_2 = g(\tau) - g(0) \\
            =& \tau g'(\tau') = \tau \frac{ d \|  \sharpproj(w_{k, \tau}) \nabla L(w_{k, \tau'}) - \sharpproj(w_{k, \tau}) \nabla L(w_k) \|_2}{d\tau'} \\
            \le& \| \frac{ d   \sharpproj(w_{k, \tau}) \nabla L(w_{k, \tau'}) - \sharpproj(w_{k, \tau}) \nabla L(w_k) }{d\tau'} \|_2 \\
            =& \eta \| \sharpproj(w_{k, \tau}) \nabla^2 L(w_{k, \tau'}) \nabla L(w_k) \|_2 \\
             \\
            \le & \eta  \| \sharpproj(w_{k, \tau}) \nabla^2 L(w_{k, \tau'}) \sharpproj (w_{k, \tau'}) \nabla L(w_k) \|_2  + \eta \| \sharpproj(w_{k, \tau}) \nabla^2 L(w_{k, \tau'}) \flatproj (w_{k, \tau'}) \nabla L(w_k) \|_2 \\
            \le & \eta \maxgamma \|  \sharpproj (w_{k, \tau'}) \nabla L(w_k) \|_2 + \eta \| (\sharpproj(w_{k, \tau}) - \sharpproj(w_{k, \tau'})) \nabla^2 L(w_{k, \tau'}) \flatproj (w_{k, \tau'}) \nabla L(w_k) \|_2.
        \end{align*}
        By~\Cref{lem:moveps}, it holds that
        \begin{align*}
            \maxgamma \|  \sharpproj (w_{k, \tau'}) \nabla L(w_k) \|_2 \le \maxgamma \|  sg(k) \|_2 + \eta \maxgamma \gamma \err \| \nabla L(w_k)\|_2. \\
            \| (\sharpproj(w_{k, \tau}) - \sharpproj(w_{k, \tau'}) \nabla^2 L(w_{k, \tau'}) \flatproj (w_{k, \tau'}) \nabla L(w_k) \|_2 \le
            \eta \gamma  \flatgamma \err \| \nabla L(w_k)\|_2.
        \end{align*}
        Summing up and the proof is complete.
    \end{proof}

    \begin{lemma}
    \label{lem:movegps}
        $\forall \tau \in (0, 1)$, it holds that 
        \begin{align*}
            \|  \sharpproj(w_{k, \tau}) \nabla L(w_{k, \tau}) -sg(k) \|_2 \le  \|  sg(k) \|_2 + 3\eta  \gamma \err \| \nabla L(w_k)\|_2.
        \end{align*}
    \end{lemma}

    \begin{proof}
        This is a direct combination of~\Cref{lem:moveg,lem:moveps}, with
        \begin{align*}
            &\|  \sharpproj(w_{k, \tau}) \nabla L(w_{k, \tau}) - \sharpproj(w_k) \nabla L(w_k) \|_2 \\
            \le& \|  \sharpproj(w_{k, \tau}) \nabla L(w_{k, \tau}) - \sharpproj(w_{k, \tau}) \nabla L(w_k) \|_2  + \| (\sharpproj(w_{k, \tau}) - \sharpproj(w_k)) \nabla L(w_k) \|_2.
        \end{align*}
        The proof is then complete.
    \end{proof}

    \begin{lemma}
    \label{lem:movepfg}
        $\forall \tau \in (0, 1)$, it holds that 
        \begin{align*}
            \|  \flatproj(w_{k, \tau}) \nabla L(w_{k, \tau}) - \flatproj(w_{k, \tau}) \nabla L(w_k) \|_2 \le \eta \flatgamma \|  fg(k) \|_2 + 2\eta^2 \maxgamma \gamma \err \| \nabla L(w_k)\|_2.
        \end{align*}
    \end{lemma}

    \begin{proof}
        Define $g(\tau') = \| \flatproj(w_{k, \tau}) \nabla L(w_{k, \tau'}) - \flatproj(w_{k, \tau}) \nabla L(w_k) \|_2  $, then by  Lagrange's Mean Value Theorem, there exists $\tau'$, such that
        \begin{align*}
            &\|  \flatproj(w_{k, \tau}) \nabla L(w_{k, \tau}) - \flatproj(w_{k, \tau}) \nabla L(w_k) \|_2 = g(\tau) - g(0) \\
            =& \tau g'(\tau') = \tau \frac{ d \|  \flatproj(w_{k, \tau}) \nabla L(w_{k, \tau'}) - \flatproj(w_{k, \tau}) \nabla L(w_k) \|_2}{d\tau'} \\
            \le& \| \frac{ d   \flatproj(w_{k, \tau}) \nabla L(w_{k, \tau'}) - \flatproj(w_{k, \tau}) \nabla L(w_k) }{d\tau'} \|_2 \\
            =& \eta \| \flatproj(w_{k, \tau}) \nabla^2 L(w_{k, \tau'}) \nabla L(w_k) \|_2 \\
             \\
            \le & \eta  \| \flatproj(w_{k, \tau}) \nabla^2 L(w_{k, \tau'}) \flatproj (w_{k, \tau'}) \nabla L(w_k) \|_2  + \eta \| \flatproj(w_{k, \tau}) \nabla^2 L(w_{k, \tau'}) \sharpproj (w_{k, \tau'}) \nabla L(w_k) \|_2 \\
            \le & \eta \flatgamma \|  \flatproj (w_{k, \tau'}) \nabla L(w_k) \|_2 + \eta \| (\flatproj(w_{k, \tau}) - \flatproj(w_{k, \tau'})) \nabla^2 L(w_{k, \tau'}) \sharpproj (w_{k, \tau'}) \nabla L(w_k) \|_2.
        \end{align*}
        By~\Cref{lem:moveps}, it holds that
        \begin{align*}
            \flatgamma \|  \flatproj (w_{k, \tau'}) \nabla L(w_k) \|_2 \le \flatgamma \| fg(k) \|_2 + \eta \flatgamma \gamma \err \| \nabla L(w_k)\|_2. \\
            \| (\flatproj(w_{k, \tau}) - \flatproj(w_{k, \tau'}) \nabla^2 L(w_{k, \tau'}) \flatproj (w_{k, \tau'}) \nabla L(w_k) \|_2 \le
            \eta \gamma  \maxgamma \err \| \nabla L(w_k)\|_2.
        \end{align*}
        Summing up and the proof is complete.
    \end{proof}

    \begin{lemma}
    \label{lem:movegfs}
        $\forall \tau \in (0, 1)$, it holds that 
        \begin{align*}
            \|  \flatproj(w_{k, \tau}) \nabla L(w_{k, \tau}) -fg(k) \|_2 \le  \|  fg(k) \|_2 + 3\eta  \gamma \err \| \nabla L(w_k)\|_2.
        \end{align*}
    \end{lemma}

    \begin{proof}
        This is a direct combination of~\Cref{lem:movepfg,lem:moveps}, with
        \begin{align*}
            &\|  \flatproj(w_{k, \tau}) \nabla L(w_{k, \tau}) - \flatproj(w_k) \nabla L(w_k) \|_2 \\
            \le& \|  \flatproj(w_{k, \tau}) \nabla L(w_{k, \tau}) - \flatproj(w_{k, \tau}) \nabla L(w_k) \|_2  + \| (\flatproj(w_{k, \tau}) - \flatproj(w_k)) \nabla L(w_k) \|_2.
        \end{align*}
        The proof is then complete.
    \end{proof}

We will prove a discrete version of~\Cref{lem:gfsg}.

\begin{lemma}
\label{lem:gdsg}
    Under~\Cref{assum:river,assum:technical}, when $\eta < 1 / \maxgamma$, along the gradient flow~\Cref{eq:gd}, it holds that $\| \sharpproj(w_k) \nabla L(w_k)\|_2 \le 10 \err \| \flatproj(w_k) \nabla L(w_k)\|_2$ as long as $w(\tau) \in U, \forall \tau \le k$.
\end{lemma}

\begin{proof}

    We will first consider $sg(k)$, By Lagrange's Mean Value Theorem, there exists $\tau$, such that,
    \begin{align*}
        &\| sg(k + 1) \|_2^2 - \| sg(k)\|_2^2 \\ =& 
        \| \sharpproj(w_{k,1}) \nabla L(w_{k,1})\|_2^2 - \| \sharpproj(w_{k,0}) \nabla L(w_{k,0})\|_2^2 = \frac{d \| \sharpgradtau \|_2^2}{d \tau} \\
        =& -\eta \langle \sharpgradtau, \partial \sharpproj(w_{k, \tau})[\nabla L(w_k)] \nabla L(w_{k, \tau}) + \sharpproj(w_{k, \tau}) \nabla^2 L(w_{k, \tau}) \nabla L(w)\rangle
    \end{align*}

    The first term satisfies that
    \begin{align*}
        &-\eta \langle \sharpgradtau, \partial \sharpproj(w_{k, \tau})[\nabla L(w_k)] \nabla L(w_{k, \tau})\rangle \\ \le& \eta \gamma \err \| \nabla L(w_k)\|_2 \| \sharpgradtau \|_2  \\
        \le& \eta\gamma \err \| \nabla L(w_k)\|_2 ( 2\|  sg(k) \|_2 + 3\eta  \gamma \err \| \nabla L(w_k)\|_2).
    \end{align*}

    The second term satisfies that
    \begin{align*}
        &-\eta \langle \sharpgradtau, \sharpproj(w_{k, \tau}) \nabla^2 L(w_{k, \tau}) \nabla L(w_k)\rangle \\
        &= -\eta\langle \sharpproj(w_{k, \tau}) \nabla L(w_k), \sharpproj(w_{k, \tau}) \nabla^2 L(w_{k, \tau}) \nabla L(w_k)\rangle  + \langle \sharpproj(w_{k,\tau}) (\nabla L(w_k) - \nabla L(w_{k,\tau})), \nabla^2 L(w_{k, \tau}) \nabla L(w_k)\rangle \\
        &\le -\eta (\gamma  + 4 \flatgamma)\| \sharpproj(w_{k, \tau}) \nabla L(w_k)\|_2^2  + \eta \maxgamma \| \sharpproj(w_{k, \tau}) \nabla L(w_k) \|_2 \|\sharpproj(w_{k, \tau}) (\nabla L(w_k) - \nabla L(w_{k,\tau})) \|_2  
    \end{align*}

    As we have that $\| a - b \|^2 \ge \frac{\| a\|^2}{2} - 4\| b\|^2$, by~\Cref{lem:moveps}, it holds that
    \begin{align*}
     &- \| \sharpproj(w_{k, \tau}) \nabla L(w_k)\|_2^2 \\
        =&-\| \sharpproj(w_k) \nabla L(w_k) + \left(  \sharpproj(w_{k, \tau}) - \sharpproj(w_k)\right) \nabla L(w_k)\|_2^2 \\ \le& - \frac{\| \sharpproj(w_k) \nabla L(w_k)\|_2^2}{2} + 4\| \left(  \sharpproj(w_{k, \tau}) - \sharpproj(w_k)\right) \nabla L(w_k) \|_2^2 \\
        \le&  -\frac{\| \sharpproj(w_k) \nabla L(w_k)\|_2^2}{2} + 4  (\eta \gamma \err \| \nabla  L(w_k) \|)^2. \\
        =& - \frac{\| \sharpproj(w_k) \nabla L(w_k)\|_2^2}{2} + 4 (\eta\gamma)^2  (\err \| \nabla  L(w_k) \|)^2. \\
        \le& -\| sg(k)\|_2^2  + 2\eta\gamma \err^2 \| \nabla  L(w_k) \|^2
    \end{align*}

    Hence 
    \begin{align*}
         &-\eta (\gamma  + 4 \flatgamma)\| \sharpproj(w_{k, \tau}) \nabla L(w_k)\|_2^2 \\
         \le&   -\eta (\gamma  + 4 \flatgamma) \| sg(k)\|_2^2  + 2\eta^2 (\gamma  + 4 \flatgamma)\gamma \err^2 \| \nabla  L(w_k) \|^2 \\
         \le&   -\eta (\gamma  + 4 \flatgamma) \| sg(k)\|_2^2  + 2\eta \gamma \err^2 \| \nabla  L(w_k) \|^2
    \end{align*}

    By~\Cref{lem:moveg,lem:moveps} and $\eta \maxgamma^2 \le \gamma / 2$, it holds that
    \begin{align*}
         &\eta \maxgamma \| \sharpproj(w_{k, \tau}) \nabla L(w_k) \|_2 \|\sharpproj(w_{k, \tau}) (\nabla L(w_k) - \nabla L(w_{k,\tau})) \|_2 \\
         \le& \eta^2 \maxgamma^2  (\| sg(k) \|_2 + \eta  \gamma \kappa \|\nabla L(w_k)\|_2) (\| sg(k) \|_2 + 2\eta  \gamma \kappa \|\nabla L(w_k)\|_2)\\
         \le& \eta \gamma  (\| sg(k) \|_2 +  \kappa \|\nabla L(w_k)\|_2/2) (\| sg(k) \|_2 +  \kappa \|\nabla L(w_k)\|_2)/2.
    \end{align*}

    Hence, we can conclude that 
    \begin{align}
        &\|sg(k + 1)\|_2^2 - \|sg(k)\|_2^2 \notag\\
        &-\eta((\gamma  + 4 \flatgamma) \| sg(k)\|_2^2  + 4 \eta  \gamma \err^2 \| \nabla  L(w_k) \|^2 \notag \\
        &+ \eta \gamma   (\| sg(k)\|_2 
 +   \err \| \nabla L(w_k) \|_2 / 2)(\| sg(k)\|_2 
 +   \err \| \nabla L(w_k) \|_2)/ 2 \notag 
    \end{align}
    Let $b = \kappa  \| \nabla L(w_k)\|_2$ and $a = sg(k)$, as $b(2a + 3b/2) -a^2 + 4 b^2 + \frac{1}{2} (a + \frac{b}{2})(a + b) \le -\frac{a^2}{4} + 10b^2$, it holds that
    \begin{align}
    \label{eq:sg}
    \|sg(k + 1)\|_2^2 - \|sg(k)\|_2^2 \le& 
    -\eta \gamma \frac{\| sg(k)\|_2^2 }{4} + 10 \eta \gamma \kappa^2  \| \nabla  L(w_k) \|^2  - 4 \eta \flatgamma \| sg(k)\|_2^2 
    \end{align}

    Similarly, we can control the $fg(k) $ changes. By Lagrange's Mean Value Theorem, there exists $\tau'$, such that,
    \begin{align*}
        &\| fg(k + 1) \|_2^2 - \| fg(k)\|_2^2 \\ =& 
        \| \flatproj(w_{k,1}) \nabla L(w_{k,1})\|_2^2 - \| \flatproj(w_{k,0}) \nabla L(w_{k,0})\|_2^2 = \frac{d \| \flatgradtau \|_2^2}{d \tau} \\
        =& -\eta \langle \flatgradtau, \partial \flatproj (w_{k, \tau})[\nabla L(w_k)] \nabla L(w_{k, \tau}) + \flatproj (w_{k, \tau}) \nabla^2 L(w_{k, \tau}) \nabla L(w)\rangle 
    \end{align*}

    The first term satisfies that 
    \begin{align*}
        &\eta \langle \flatgradtau, \partial \flatproj (w_{k, \tau})[\nabla L(w_k)] \nabla L(w_{k, \tau}) \rangle \\
        \le& \gamma \eta \kappa \| \nabla L(w_k)\|_2 \| \flatgradtau \|_2. \\
        \le&  \gamma \eta \kappa \| \nabla L(w_k)\|_2 ( \| fg(k) \|_2 + \eta \flatgamma \|  fg(k) \|_2 + 2\eta^2 \maxgamma \gamma \err \| \nabla L(w_k)\|_2). \\\le& 4  \gamma \eta \kappa  \| \nabla L(w_k)\|_2^2.
    \end{align*}

    Similarly, the second term satisfies that 
    \begin{align*}
         &\eta \langle \flatgradtau, \flatproj (w_{k, \tau}) \nabla^2 L(w_{k, \tau}) \nabla L(w)\rangle \\
         \le&  \eta \flatgamma \| \flatproj (w_{k, \tau}) \nabla L(w_k) \|_2  \| \flatgradtau \|_2 \\
         \le&  \eta \flatgamma  (\| fg(k) \|_2 + \eta \gamma \kappa \| \nabla L(w_k) \|_2) (\| fg(k) \|_2 + \eta \flatgamma \|  fg(k) \|_2 + 2\eta^2 \maxgamma \gamma \err \| \nabla L(w_k)\|_2) \\
         \le&  2\eta \flatgamma  (\| fg(k) \|_2 + \eta \gamma \kappa \| \nabla L(w_k) \|_2)^2 \\
         \le& 4 \eta \flatgamma \| fg(k) \|_2^2 + 4\eta^2 \gamma^2 \kappa^2  \| \nabla L(w_k) \|_2^2
    \end{align*}

    Summarizing and we have
    \begin{align}
        \label{eq:fg}  
        \|fg(k + 1)\|_2^2 - \|fg(k)\|_2^2 \ge& - 5 \gamma \eta \kappa  \| \nabla L(w_k)\|_2^2 - 4 \eta \flatgamma \| fg(k) \|_2^2
    \end{align}

    Let $a_\kappa$ be the smaller positive solution of
    \begin{align*}
        5\kappa a^2 + (10\kappa^2 + 5\kappa - \frac{1}{4})a + 10 \kappa^2 = 0.
    \end{align*}
    Then $a_\kappa = \frac{(-10\kappa^2 - 5\kappa + \frac{1}{4}) - \sqrt{(-10\kappa^2 - 5\kappa + \frac{1}{4})^2 - 200 \kappa^3}}{10 \kappa}< 100 \kappa^2$.

    Then combining~\Cref{eq:fg,eq:sg}
    \begin{align*}
        &\|sg(k+1)\|_2^2  - a_{\kappa} \|fg(k + 1)\|_2^2 \\
        \le& (1 - 4 \eta\flatgamma) (\|sg(k)\|_2^2  - a_{\kappa} \|fg(k + 1)\|_2^2) - \eta\gamma (\frac{1}{4} + 10\kappa^2 - 5 \kappa a_{\kappa}) \|sg(k)_2\|^2 + \eta \gamma (10 \kappa^2 + 5\kappa a_{\kappa}) \|fg(k)_2\|^2 \\
        =& (1 - 4 \eta\flatgamma - \eta\gamma(\frac{1}{4} + 10\kappa^2 - 5 \kappa a_{\kappa})) ( \|sg(k)\|_2^2  - a_{\kappa} \|fg(k)\|_2^2).
    \end{align*}

    As $\| sg(0) \|_2 ^2 - a_{\kappa} \|fg(0)\|_2^2 < 0$, we have that $\|sg(k)\|_2^2  < a_{\kappa} \|fg(k + 1)\|_2^2  < 100 \kappa^2 \| fg(k)\|_2^2$ for all the $t$.
\end{proof}

Then we can show that gradient descent will also track the river closely.

\begin{lemma}
\label{lem:gdclose}
    Under~\Cref{assum:river,assum:technical}, along the gradient descent~\Cref{eq:gd}, it holds that $w_k \in U$ and $\| w_k- \Phi(w_k) \|_2 \le \frac{10 \err \| \flatproj(w_{k}) \nabla L(w_{k})\|_2}{\gamma + \flatgamma}$.
\end{lemma}
\begin{proof}
    This is a direct combination of~\Cref{lem:gdsg,lem:phiexists}.
\end{proof}

Finally, we will show that the movement of the projection of the gradient flow moves approximately at the same rate as the river, a discrete version of~\Cref{lem:gfclose}.

\begin{lemma}
\label{lem:gdspeedcalc}
    Under~\Cref{assum:river,assum:technical}, along the gradient flow~\Cref{eq:gf}, along the gradient descent~\Cref{eq:gd}, if $T(t)$ satisfies $x(T(t)) = \Phi(w_{[t], t - [t]})$ where $[t]$ is the integer part of $t$, then for any $t$ that is not integer
\begin{align*}
    \frac{d T(t)}{d t} \in [\eta - (30 \kappa + 4\eta\flatgamma)\eta, \eta + (30 \kappa + 4\eta\flatgamma)\eta].
\end{align*}
\end{lemma}

\begin{proof}
Let $[t] = k, t- [t] = \tau$, as $T(t)$ satisfies $x(T(t)) = \Phi(w(k, t - [t]))$, let $v = \frac{d x(\tau)}{d \tau} \mid_{\tau = T(t)}$, taking derivative on both sides yield,
\begin{align*}
    \frac{d T(t)}{d t} v = -\eta \partial \Phi(w_{k, \tau}) \nabla L(w_k).
\end{align*}

As the proof of~\Cref{lem:movepfg}, there exists $\tau'$
\begin{align*}
    &\| \flatproj(w_{k, \tau}) \nabla L(w_k)  - \flatproj(w_{k, \tau}) \nabla L(w_{k, \tau}) \|_2 \\
    \le& \eta  \| \flatproj(w_{k, \tau}) \nabla^2 L(w_{k, \tau'}) \flatproj (w_{k, \tau'}) \nabla L(w_k) \|_2  + \eta \| \flatproj(w_{k, \tau}) \nabla^2 L(w_{k, \tau'}) \sharpproj (w_{k, \tau'}) \nabla L(w_k) \|_2 \\
    \le& \eta \flatgamma \|\flatproj (w_{k, \tau'}) \nabla L(w_k) \|_2 + \eta^2 \gamma \kappa \maxgamma \| \nabla L(w_k)\|_2
\end{align*}

By~\Cref{lem:moveps}, it holds that
\begin{align}
&\| \flatproj(w_{k, \tau}) \nabla L(w_k)  - \flatproj(w_{k, \tau}) \nabla L(w_{k, \tau}) \|_2 \notag \\
\le& \eta \flatgamma \|\flatproj (w_{k, \tau}) \nabla L(w_k) \|_2 + 2\eta^2 \gamma \maxgamma \kappa \| \nabla L(w_k)\|_2 \notag \\
\le& \eta \flatgamma \|\flatproj (w_{k, \tau}) \nabla L(w_k) \|_2 + \kappa \| \nabla L(w_k)\|_2\label{eq:fpg1fpg} 
\end{align}

By~\Cref{lem:gdsg}, 
\begin{align*}
    \| \nabla L(w_k)\|_2 &\le \frac{1}{1 - 10 \kappa} \| \flatproj(w_k) \nabla L(w_k)\|_2 \\
    &\le \frac{1}{1 - 10 \kappa} \left(   \|\flatproj (w_{k, \tau}) \nabla L(w_k) \|_2 + \eta \gamma \kappa \| \nabla L(w_k) \|_2 \right)
\end{align*}

This shows that 
\begin{align*}
    \| \nabla L(w_k)\|_2 &\le  \frac{1}{1 - 10 \kappa - \eta \gamma \kappa}  \|\flatproj (w_{k, \tau}) \nabla L(w_k) \|_2 \le (1 + 12 \kappa) \|\flatproj (w_{k, \tau}) \nabla L(w_k) \|_2 
\end{align*}

Combining with~\Cref{eq:fpg1fpg}, we have that
\begin{align}
    &\| \flatproj(w_{k, \tau}) \nabla L(w_k)  - \flatproj(w_{k, \tau}) \nabla L(w_{k, \tau}) \|_2  \notag\\
    \le& \eta \flatgamma \|\flatproj (w_{k, \tau}) \nabla L(w_k) \|_2 + \kappa(1 + 12 \kappa) \|\flatproj (w_{k, \tau}) \nabla L(w_k) \|_2  \notag\\
    \le& (\eta \flatgamma + 2\kappa) \|\flatproj (w_{k, \tau}) \nabla L(w_k) \|_2 \notag
\end{align}

This shows that
\begin{align}
    \| \flatproj(w_{k, \tau}) \nabla L(w_k)  - \flatproj(w_{k, \tau}) \nabla L(w_{k, \tau}) \|_2 \le& \frac{(\eta \flatgamma + 2\kappa) }{1 - (\eta \flatgamma + 2\kappa)}  \|\flatproj (w_{k, \tau}) \nabla L(w_{k, \tau}) \|_2 \notag\\ \le&(2\eta \flatgamma + 3\kappa) \|\flatproj (w_{k, \tau}) \nabla L(w_k) \|_2 \label{eq:fpg2fpg}
\end{align}

By~\Cref{lem:similarmove}
\begin{align}
\label{eq:fpg1river}
   \|  \flatproj(w_{k, \tau}) \nabla L(w_{k, \tau}) + v \|_2 \le 16 \kappa \| v \|_2 
\end{align}

Combining~\Cref{eq:fpg1river,eq:fpg2fpg},
\begin{align}
    &\| \flatproj(w_{k, \tau}) \nabla L(w_k) + v \|_2 \notag \\
    \le&  \| \flatproj(w_{k, \tau}) \nabla L(w_k) -  - \flatproj(w_{k, \tau}) \nabla L(w_{k, \tau})\|_2 +  \| \flatproj(w_{k, \tau}) \nabla L(w_{k,\tau}) + v \|_2 \notag \\
    \le& (2\eta \flatgamma + 3\kappa) \|\flatproj (w_{k, \tau}) \nabla L(w_k) \|_2 + 16\kappa \|v\|_2 \notag \\
    \le& ((2\eta \flatgamma + 3\kappa)( 1 + 16 \kappa) + 16\kappa)\|v\|_2 \notag  \\
    \le& ( 19\kappa + 3\eta\flatgamma) \| v\|_2. \label{eq:fpg2river}
\end{align}

By~\Cref{lem:boundpartialphigen}
\begin{align}
\label{eq:phg1fpg}
     \| \partial \Phi(w_{k, \tau})  \nabla L(w_k) - \flatproj(w_{k, \tau}) \nabla L(w_k) \|_2 \le 10 \kappa \|\flatproj(w_{k, \tau}) \nabla L(w_k) \|_2.
\end{align}

Combining~\Cref{eq:fpg2river,eq:phg1fpg}, it holds that
\begin{align*}
&\| \partial \Phi(w_{k, \tau})  \nabla L(w_k)  + v\|_2 \\
\le& \| \partial \Phi(w_{k, \tau})  \nabla L(w_k) - \flatproj(w_{k, \tau}) \nabla L(w_k) \|_2 + \| \flatproj(w_{k, \tau}) \nabla L(w_k) + v \|_2 \\
\le& 10 \kappa \|\flatproj(w_{k, \tau}) \nabla L(w_k) \|_2 + ( 19 \kappa + 3\eta\flatgamma) \| v\|_2 \\
\le& (30 \kappa + 4\eta\flatgamma) \|v\|_2.
\end{align*}

Hence 
\begin{align*}
    \frac{d T(t)}{d t} \in [\eta - (30 \kappa + 4\eta\flatgamma)\eta, \eta + (30 \kappa + 4\eta\flatgamma)\eta].
\end{align*}
This concludes the proof.
\end{proof}

\begin{proof}[Proof of~\Cref{thm:gdtracksriver}]
The proof is a direct combination of~\Cref{lem:gdclose,lem:gdspeedcalc}.
\end{proof}

\subsection{Proof of~\Cref{thm:sgd-main}}
\label{app:sgd}

We will first show that under~\Cref{assum:straight}, the loss is separable within $U$.

\begin{lemma}
\label{lem:straightriver}
Under~\Cref{assum:river,assum:straight,assum:technical}, the river is a straight line parallel to $v_d$.
\end{lemma}

\begin{proof}
    In this case, the $\kappa$ in~\Cref{assum:technical} is $0$ and this is a direct corollary of of~\Cref{lem:similarmove}.
\end{proof}

\begin{lemma}
\label{lem:separable}
Under~\Cref{assum:river,assum:straight,assum:technical}, there exists functions $g$ and $h$, such that for any $w \in U$ satisfying that $\mathcal{B}(w, \frac{2 \normbound}{\gamma}) \subset U$, it holds that
\begin{align*}
    L(w) = g(\Phi(w)) + h(w - \Phi(w)).
\end{align*}
Furthermore, $h$ is a $\gamma$-strongly convex function when constrained on the range of $\sharpproj$. 
\end{lemma}

\begin{proof}
    We will choose $g$ as the constraint of $L$ on $\river$.  Now $w - \Phi(w)$ will always fall in the range of $\sharpproj$. Consider any $y$ in the range of $\sharpproj$ and as $\nabla v_d(w)[v] = 0$, we have that
    \begin{align*}
        y^T \nabla^2 L(w) v_d = 0.
    \end{align*}
    This then suggest that
    \begin{align*}
        \nabla[ \langle \nabla L(w), y\rangle][v_d] = 0.
    \end{align*}
    We then have for any $a \in \river$, by~\Cref{lem:straightriver},
    \begin{align*}
        L(w) - L(\Phi(w))  &= \int_{0}^{1} \langle (w - \Phi(w)), \nabla L(\Phi(w) + \tau(w - \Phi(w))) \rangle d \tau \\
         &= \int_{0}^{1} \langle (w - \Phi(w)), \nabla L(a + \tau(w - \Phi(w))) \rangle d \tau.
    \end{align*}

    We will then define $h(w - \Phi(w)) =  \int_{0}^{1} \langle (w - \Phi(w)), \nabla L(a + \tau(w - \Phi(w))) \rangle d \tau$. and this concludes the proof.

    Now, as $h(w - \Phi(w)) = L(w) - L(\Phi(w))$, $\nabla^2 h(y)$ when constrained on the range of $\sharpproj$ has an eigenvalue greater than $\gamma$.
\end{proof}

We will first consider the mixing dynamics of the current SGD iterates on a strongly convex loss $h$ with a minimizer at $0$.
\begin{align}
\label{eq:sgd-h}
    y(k + 1) = y_k - \eta \nabla h(y_k) - \eta \noise_k, y(0) = 0,{\noise_k} \sim \Normal(0, \sigma^2 \identity)
\end{align}

We will define a coupling process $\tilde y_k$ as 
\begin{align}
    \label{eq:sgd-h-gauss}
    \tilde y(k + 1) = \tilde y_k - \eta H \tilde y_k - \eta \noise_k, w(0) = 0, \noise_k \sim \Normal(0, \sigma^2 \identity), \tilde y(0) = 0.
\end{align}
Here $H = \nabla^2 h(0)$ is positive definite.

\begin{assumption}[Regularity of $h$]
\label{assum:regular}
We will assume the following for the function $h$, constant $\delta \in (0,1]$, learning rate $\eta$.
    \begin{enumerate}
        \item The smallest eigenvalue of $\nabla^2 h(y)$ within $\mathcal{B}(0,r)$ is at least $\gamma > 0$ and the largest eigenvalue for $H$ is at most $\maxgamma$.
        \item $\forall y, h(y) \in [0, M]$.
        \item  $\forall y \in \mathcal{B}(0,r), \| \nabla h(y) \|_2 \le \normbound, \| \nabla^3 h(y) \|_2 \le \rho$.
        \item $T > 1 / \gamma$.
        \item $\eta < 1/(2\maxgamma)$.
        \item $\eta \rho^2 \sigma^2 \le {\gamma^3}/{(1600  d \log(8\gamma T/\delta))}$.
        \item $10\frac{\sqrt{\eta} \sigma}{\sqrt{\gamma}} \sqrt{d \log(8 \gamma T / \delta)} + 400 \eta \rho \sigma^2 d \log(8 \gamma T / \delta) / \gamma^2 \le r$.
    \end{enumerate}
\end{assumption}

We will first show that $\tilde y_k$ will be bounded with a high probability for $T/\eta$ steps.

\begin{lemma}
\label{lem:normboundsgd}
    For any $\delta \in (0,1]$, with probability $1 - \delta$, for $\tilde y_k$ defined in~\Cref{eq:sgd-h-gauss}, under~\Cref{assum:regular}, it holds that for any $k \le T/\eta$, 
    \begin{align*}
   \| \tilde y_k \|_2 \le \frac{10 \sqrt{\eta} \sigma}{\sqrt{\gamma}} \sqrt{d \log(8\gamma T/\delta)}.
    \end{align*}
\end{lemma}

\begin{proof}

For integer $I = \lceil \gamma T \rceil$. We first have that for $i \le I$
\begin{align*}
    \tilde y_{i \lceil\frac{1}{\eta \gamma} \rceil} = (1 - \eta \gamma)^{\lceil\frac{1}{\eta \gamma} \rceil} \tilde y_{(i-1)\lceil\frac{1}{\eta \gamma} \rceil }  + \eta \sum_{\tau = 0}^{\lceil\frac{1}{\eta \gamma} \rceil}  (1 - \eta \gamma)^{\lceil\frac{1}{\eta \gamma} \rceil - t}  \noise_{(i-1)\lceil\frac{1}{\eta \gamma} \rceil + \tau}.
\end{align*}

Denote $\bar \noise_{k} = \eta \sum_{\tau = 0}^{\lceil\frac{1}{\eta \gamma} \rceil}  (1 - \eta \gamma)^{\lceil\frac{1}{\eta \gamma} \rceil - t}  \noise_{(i-1)\lceil\frac{1}{\eta \gamma} \rceil + \tau}$, then $\noise_k$ is a normal vector with variance
\begin{align*}
     \eta^2 \sum_{\tau = 0}^{\lceil\frac{1}{\eta \gamma} \rceil} (1 - \eta \gamma)^{2(\lceil\frac{1}{\eta \gamma} \rceil - t)} \sigma^2 \identity \le \frac{\eta \sigma^2}{2\gamma - \eta \gamma^2} \le  \frac{\eta \sigma^2}{\gamma}.
\end{align*}

Further, denote $Y_i = y_{i \lceil\frac{1}{\eta \gamma} \rceil}$ and $e_{\gamma} = (1 - \eta \gamma)^{\lceil\frac{1}{\eta \gamma} \rceil} < \frac{1}{e}$, then
\begin{align*}
    Y_i = e_{\gamma} Y_{i - 1} + \bar \noise_{i - 1} = \sum_{j \le i - 1} e_{\gamma}^{j - 1} \noise_{i - j}
\end{align*}

Then each variable $Y_i$ is also a Gaussian variable with variance smaller than 
\begin{align*}
     \sum_{j \le i - 1} e_{\gamma}^{2(j - 1)} \E[  \bar \noise_{i - j}^T \bar \noise_{i - j}] \le \frac{1}{1-1/e^2}\frac{\eta \sigma^2} {\gamma} \identity \le \frac{2\eta \sigma^2}{\gamma} \identity .
\end{align*}

Hence, by~\Cref{lem:normbound}, for each $i$, it holds that
\begin{align*}
    \prob(\Big | Y_i \Big| > \frac{2 \sqrt{\eta} \sigma}{\sqrt{\gamma}} \sqrt{d \log(4I/\delta)})) < \delta/2I.
\end{align*}

Using union bound,
\begin{align*}
    \prob(\exists i \le I, \Big | Y_i \Big| > \frac{2 \sqrt{\eta} \sigma}{\sqrt{\gamma}} \sqrt{d \log(4I/\delta)}) < \delta/2.
\end{align*}

We now proceed to bound the distance of $y_k$ compared with close $Y_i$, without loss of generality, considering $i = 0$, we will define a new process called $m_k$ satisfying that
\begin{align*}
    m_k = \sum_{j \le k} (1 - \eta \gamma)^{\lceil\frac{1}{\eta \gamma} \rceil - j} \noise_j. 
\end{align*}
Then $m_k$ is a martingale and each $m_k$ is a Gaussian vector. In particular, $m_{\lceil\frac{1}{\eta \gamma}\rceil} = \bar \noise_1$. This further suggests that $\| m_k \|^2$ is a super martingale
\begin{align*}
    \E[\| m_k \|_2^2 \mid m_{k - 1}] \ge \| m_{k - 1} \|_2^2.
\end{align*}

By Doob's lemma (\Cref{lem:doob})
\begin{align*}
    \prob(\sup_{k \le \lceil\frac{1}{\eta \gamma} \rceil} \| m_k \|_2^2 > C^2) &\le
    \prob(\sup_{k \le \lceil\frac{1}{\eta \gamma} \rceil} \exp(\lambda  \| m_k \|_2^2 ) > \exp(\lambda C^2))\\
    &\le \E[{\exp(\lambda \| m_{\lceil\frac{1}{\eta \gamma} \rceil } \|_2^2 - \lambda C^2)}]\\
    &= \E[{\exp(\lambda \| \bar g_1 \|_2^2 - \lambda C^2)}].
\end{align*}

Following the same line of proof as~\Cref{lem:normbound}, we have that 
\begin{align*}
 \prob(\sup_{k \le \lceil\frac{1}{\eta \gamma} \rceil} \| m_k \|_2 > \frac{2 \sqrt{\eta} \sigma}{\sqrt{\gamma}} \sqrt{d \log(4K/\delta)}) \le \delta / 2K
\end{align*}

We further note that $\Big | y_k - Y_0 \Big | \le (1 - \eta \gamma)^{-\lceil\frac{1}{\eta \gamma} \rceil} m_k \le 4 m_k$. We have that for any $k < K$
\begin{align*}
 \prob(\sup_{k \le \lceil\frac{1}{\eta \gamma} \rceil} \| y_k - Y_0 \|_2 > \frac{8 \sqrt{\eta} \sigma}{\sqrt{\gamma}} \sqrt{d \log(4K/\delta)}) \le \delta / 2K
\end{align*}

Combining with the bound on $Y_k$, we have that
\begin{align*}
    \prob(\sup_{0 \le k \le T} |y_k| > \frac{10 \sqrt{\eta} \sigma}{\sqrt{\gamma}} \sqrt{d \log(8\gamma T/\delta)}))  \le \delta.
\end{align*}
The proof is then complete.
\end{proof}

The following lemma states that $y_k$ and $\tilde y_k$ are close with high probability.
\begin{lemma}
\label{lem:track}
    Assume function $h(y)$ is $\gamma$-strong convex in $\mathcal{B}(0,r)$ and has a minimizer at $0$,then for $\delta \in (0,1)$, under~\Cref{assum:regular}, it holds that with probability $1 - \delta$,
    \begin{align*}
        \forall k < T /\eta, \| \tilde y_k - y_k \|_2  \le   400{\eta \rho \sigma^2 d\log(8\gamma T/\delta)}/{\gamma^2}, 
        &y_k \in \mathcal{B}(0, r), \tilde y_k \in \mathcal{B}(0,r)
    \end{align*}
\end{lemma}

\begin{proof}
By~\Cref{lem:normboundsgd}, with probability $1 - \delta$, 
\begin{align*}
   \forall k < T /\eta,  \| \tilde y_k \|_2^2 \le \frac{100 {\eta} \sigma^2}{{\gamma}} {d \log(8\gamma T/\delta)} 
\end{align*}

We will use $C$ as a shorthand for $100 {d \log(8\gamma T/\delta)}$. Under such scenario, define $\nu_k =  \tilde y_k - y_k $, we will prove by induction for $k \le T/\eta$ that 
\begin{align}
\label{eq:ind}
    \| \nu_k \|_2 \le 4{\eta \rho \sigma^2 C}/{\gamma^2} , y_k \in \mathcal{B}(0,r).
\end{align}

Clearly $\nu_0 = 0$, satisfies the induction hypothesis.
Assuming~\Cref{eq:ind} hold for $t$, then 
\begin{align*}
    y(k + 1) &=  y_k - \eta \nabla L(y_k) - \eta \noise_k \\
    &= y_k - \eta \nabla^2 L(0) y_k + e_k - \eta \noise_k \\
    &= \tilde y_k( 1 - \eta\nabla^2 L(0)) - \eta \noise_k + \nu_k( 1 - \eta\nabla^2 L(0)) + e_k.
\end{align*}

Here $\| e_k \| = \| -\eta ( \nabla L(y_k) -  \eta \nabla^2 L(0) y_k) \|_2 \le \eta \rho \|y_k\|_2^2 \le 2 \eta \rho (\|\tilde y_k \|_2^2 + \| \nu_k\|_2^2)$. Hence we have that
\begin{align*}
    \| \nu_{k + 1} \|_2 \le (1 - \eta \gamma) \| \nu_k \|_2 + 2 \eta \rho (\|\tilde y_k \|_2^2 + \| \nu_k\|_2^2). 
\end{align*}

As $\| \nu_k \| \le 4\frac{\eta \rho \sigma^2 C}{\gamma^2} \le \frac{\gamma}{4 \rho}$, we have that $2 \rho \| \nu_k \|_2^2 \le \eta \gamma \| \nu_k\|_2^2 / 2$. 

Hence 
\begin{align*}
    \| \nu_{k + 1} \|_2 &\le (1 - \eta \gamma/2)  \| \nu_k \|_2 + 2 \eta \rho \frac{\eta \sigma^2 C}{\gamma}  \\
    &= (1 - \eta \gamma/2)  \| \nu_k \|_2  + 2  \frac{\eta^2 \rho \sigma^2 C}{\gamma}
\end{align*}

By induction $\| \nu_k \|_2 \le 4{\eta \rho \sigma^2 C}/{\gamma^2}$. It is then easy to check $\| \nu_{k + 1} \|_2 \le  4{\eta \rho \sigma^2 C}/{\gamma^2}$.
\end{proof}

The following lemma tracks the changes of $\E[h(y_k)]$.
\begin{lemma}
\label{lem:tracksgd}
    Assume function $h(y)$ is $\gamma$-strong convex in $\mathcal{B}(0,r)$ and has a minimizer at $0$,then for $\delta \in (0,1)$, denote $100 \log(8\gamma T / \delta) $ as $C$, under~\Cref{assum:regular}, it holds that $\forall t \in [1/\eta \gamma, T /\eta]$,
    \begin{align*}
    \Big | \E[h(\tilde y_k)] -  \eta \sigma^2d / 2\Big|\le \eta^2 \sigma^2 \mathrm{Tr}(H) + \frac{ \normbound \rho C}{\gamma^2} d \eta  \sigma^2 + \rho \left( \frac{ d \eta \sigma^2  C}{{\gamma}} \right)^{3/2} + 2\delta M + \delta \eta \sigma^2 d / 2
    \end{align*}
\end{lemma}

\begin{proof}
    By~\Cref{lem:track}, with probability $1 - \delta$, 
    \begin{align*}
        \| \tilde y_k - y_k \|_2  \le   4{\eta \rho d\sigma^2 C }/{\gamma^2}, y_k\in \mathcal{B}(0,r) , \tilde y_k \in \mathcal{B}(0,r).
    \end{align*}
    Define this event as $\mathcal{E}_1$.

    Hence 
    \begin{align*}
        &\Big| \E[h(y_k)] - \E[ h(\tilde y_k)] \Big|\\
        \le& \Big| \E[h(\tilde y_k) - h(y_k) \mid \mathcal{E}_1 ]\prob(\mathcal{E}_1  ) + \E[h(\tilde y_k) - h(y_k) \mid \mathcal{E}_1^c ]\prob(\mathcal{E}_1^c) \Big | \\
        \le& 4\frac{ \normbound \rho C}{\gamma^2} \eta  d \sigma^2 + \delta M.
    \end{align*}

    For $\| y \|_2 < r$, it holds that
    \begin{align*}
        \| h(y) - y^T \nabla^2 h(0) y \|_2 \le \rho \| y \|_2^3.
    \end{align*}

    By~\Cref{lem:normboundsgd}, with probability $1 - \delta$, for $\eta < 1/\maxgamma$, it holds that for any $k \le T/\eta$, 
    \begin{align*}
   \| \tilde y_k \|_2^2 \le \frac{ d\eta \sigma^2  C}{{\gamma}} < r^2.
    \end{align*}
Define this event as $\mathcal{E}_2$.

    Denote $H = \nabla^2 h(0)$, we have that,
    \begin{align*}
        &\Big| \E[h(\tilde y_k)] - \E[(\tilde y_k)^T H \tilde y_k] \Big|\\
    =&  \Big| \E[h(\tilde y_k) \mid \mathcal{E}_2] \prob(\mathcal{E}_2) - \E[(\tilde y_k)^T H \tilde y_k] +   \E[h(\tilde y_k) \mid \mathcal{E}_2^c]\prob(\mathcal{E}_2^c)\Big | \\
    \le&  \delta M + \Big|\E[h(\tilde y_k) \mid \mathcal{E}_2] \prob(\mathcal{E}_2) -   \E[(\tilde y_k)^T H \tilde y_k] \Big| \\
    \le&  \delta M + \rho  \left( \frac{ \eta d \sigma^2  C}{{\gamma}} \right)^{3/2}+ \E[(\tilde y_k)^T H \tilde y_k \mid \mathcal{E}_2^c]
    \end{align*}

    Combining the both and we have that
    \begin{align*}
          \Big| \E[h(y_k)] - \E[(\tilde y_k)^T H \tilde y_k]  \Big| \le 4\frac{ \normbound \rho C}{\gamma^2} \eta  d \sigma^2 + \rho \left( \frac{ \eta d \sigma^2  C}{{\gamma}} \right)^{3/2} + 2\delta M + \E[(\tilde y_k)^T H \tilde y_k \mid \mathcal{E}_2^c].
    \end{align*}

    Here the covariance of $\tilde y_k$, denoted as $\Sigma_k$ satisfies that
    \begin{align*}
        \Sigma_{k + 1} = (\identity - \eta H)^2 \Sigma_{k} + \sigma^2 \eta^2 \identity.
    \end{align*}

    Therefore 
    \begin{align*}
        &\Sigma_k - \eta \sigma^2 (2 \eta H - \eta^2 H^2)^{-1}  = (\identity - \eta H)^2 (\Sigma_{k - 1} - \eta \sigma^2 (2 \eta H - \eta^2 H^2)^{-1}). \\
        &\Sigma_k = \sigma^2 (2H - \eta H^2)^{-1} (\identity - (\identity - \eta H)^{2k}).
    \end{align*}

    Hence assuming the eigenvalues of $H$ is $\gamma_1, \dots, \gamma_d$
    \begin{align*}
        \E[(\tilde y_k)^T H \tilde y_k] = \mathrm{Tr}(\Sigma_k H) = \eta \sigma^2 \sum_{i = 1}^d \frac{1}{2 - \eta \gamma_i } (1 - (1 - \eta \gamma_i)^{2k}).
    \end{align*}

    When $t \ge \frac{1}{\eta\gamma_i}, \eta \gamma_i < 1/2$, it holds that
    \begin{align*}
        &\Big|  \eta \sigma^2 \frac{1}{2 - \eta \gamma_i} (1 - (1 - \eta \gamma_i)^{2k}) - \eta \sigma^2 /2 \Big | \\
        =& \eta \sigma^2 \frac{(1 - (1 - \eta \gamma_i)^{2k})}{2 - \eta \gamma_i} - \eta \sigma^2 /2 \\
        \le& \eta \sigma^2 \left( \frac{1}{2 - \eta \gamma_i} - \frac{1}{2} \right) \\
        \le& \eta^2 \sigma^2 \gamma_i / 2.
    \end{align*}

    Hence 
    \begin{align*}
        \Big | \E[(\tilde y_k)^T H \tilde y_k] - d \eta \sigma^2 / 2 \Big | \le \mathrm{Tr}(H)\eta^2 \sigma^2 / 2.
    \end{align*}
    Further, let $u_k = \Sigma_k^{-1/2} \tilde y_k$, under $\mathcal{E}_2^c$, we have that
    \begin{align*}
        \| u_k \|_2^2 \ge \lambda_{min}(\Sigma_k^{-1}) \| \tilde y_k \|_2^2 = \lambda_{min}\left( (2H - \eta H^2) (\identity - (\identity - \eta H)^{2k})^{-1}  \right) \| \tilde y_k \|_2^2 / \sigma^2 \ge d \sigma^2 C.
    \end{align*}

    As $u_k$ is isometric Gaussian,
    \begin{align*}
        \E[(\tilde y_k)^T H \tilde y_k \mid \mathcal{E}_2^c] &\le \E[u_k^T (\Sigma_k^{1/2})^T H \Sigma_k^{1/2} u_k \mid \| u_k \|_2^2 \ge d \sigma^2 C] \\
        &= \E[u_k^T (\Sigma_k^{1/2})^T H \Sigma_k^{1/2} u_k] \frac{\E[\|u_k\|_2^2 \mid \| u_k \|_2^2 \ge d \sigma^2 C]}{\E[\|u_k\|^2]} \\
        &\le d \eta \sigma^2 \frac{\E[\|u_k\|_2^2 \mid \| u_k \|_2^2 \ge d \sigma^2 C]}{\E[\|u_k\|^2]}  \end{align*}

    Plugging in the density function of $\|u_k\|_2$, we have that
    \begin{align*}
    \frac{\E[\|u_k\|_2^2 \mid \| u_k \|_2^2 \ge d \sigma^2 C]}{\E[\|u_k\|_2^2]} &= \frac{\int_{\sqrt{dC} \sigma }^{\infty} r^{d + 1 } e^{-r^2/(2\sigma^2)} dr}{\int_{0}^{\infty} r^{d + 1 } e^{-r^2/(2\sigma^2)} dr} 
    \end{align*}

    Let $r' = \sqrt{\frac{d}{d + 1}} r$, then 
    \begin{align*}
        \int_{\sqrt{dC} \sigma}^{\infty} r^{d + 1 } e^{-r^2/\sigma^2} dr &= (\frac{d + 1}{d })^{\frac{d + 2}{2}} \int_{d \sigma^2 C}^{\infty} r'^{d + 1} e^{-(r')^2 (d + 1)/ (2 d \sigma^2)} dr' \\
        &\le 4 \int_{\sqrt{dC} \sigma}^{\infty} r'^{d + 1} e^{-(r')^2 / (2\sigma^2)}  e^{-(r')^2/(2d \sigma^2)}  dr' \\
        &\le 4 e^{-C / 2} \int_{0}^{\infty} r^{d + 1 } e^{-r^2/(2\sigma^2)} dr.
    \end{align*}
    Hence, we have that 
    \begin{align*}
        \E[(\tilde y_k)^T H \tilde y_k \mid \mathcal{E}_2^c] \le 4 e^{-C / 2} d \eta \sigma^2 \le \delta d \eta \sigma^2 / 2.
    \end{align*}

    Putting together, we have that,
    \begin{align*}
        \Big | \E[h(\tilde y_k)] -  \eta \sigma^2d / 2\Big|\le \eta^2 \sigma^2 \mathrm{Tr}(H) + \frac{ \normbound \rho C}{\gamma^2} d \eta  \sigma^2 + \rho \left( \frac{ d \eta \sigma^2  C}{{\gamma}} \right)^{3/2} + 2\delta M + \delta \eta \sigma^2 d / 2.
    \end{align*}
    The proof is then complete.
\end{proof}

We will now state the complete version of~\Cref{thm:sgd-main}.

\begin{assumption}[Sufficient Small Learning Rate]
\label{assum:regular-L}
We will assume the following for constant $\delta \in (0,1]$ and learning rate $\eta$:
    \begin{enumerate}
        \item $\eta < 1/(2\maxgamma)$.
        \item $\eta \le {\gamma^3}/{(1600 \rho^2 \sigma^2 d \log(8\gamma T/\delta))}$.
        \item $10\frac{\sqrt{\eta} \sigma}{\sqrt{\gamma}} \sqrt{d \log(8 \gamma T / \delta)} + 400 \eta \rho \sigma^2 d \log(8 \gamma T / \delta) / \gamma^2 \le r$.
    \end{enumerate}
\end{assumption}

\begin{theorem}[Complete version of~\Cref{thm:sgd-main}]
\label{thm:sgd}
If a loss $L$ is a river valley (\Cref{def:river-valley}) and satisfies~\Cref{assum:straight,assum:noise}, for any constants $\delta \in (0,1)$ and $T > 1/\gamma$, for sufficiently small learning rate $\eta$ satisfying~\Cref{assum:regular-L}, there exists a time shift $T_0$ depending on $w$ and $\eta$, the SGD iterates (defined in~\Cref{eq:sgd}) with $\eta_k = \eta$ satisfies that for any integer $k \in [1 /{\eta\gamma}, \maxT / \eta]$, there exists a $\tilde t \in [(1 - \epsilon_{t}) \eta k, (1 + \epsilon_t) \eta k]$ satisfying that,
\begin{align*}
\E[L\left( \tilde w (k)\right)] - L( x( T_0  + \tilde t))& =  {(d - 1) \eta \sigma^2}/{2} + \epsilon_{L} \end{align*}
where $\epsilon_{t} = 4 \eta \flatgamma$ and $|\epsilon_{L}| \le \tau \eta^2  \sigma^2+ \rho (C {d\eta \sigma^2}/{\gamma})^{3/2} + C \kappa' d \eta \sigma^2 +  \delta (2 M + \eta \sigma^2 d) \ll (d - 1) \eta \sigma^2$ with $C = 200 \log(64 \gamma T/\delta)$. 
\end{theorem}

\begin{proof}
    By~\Cref{lem:separable}, we can write 
    \begin{align*}
        L(w) = h(w - \Phi(w)) + L(\Phi(w)).
    \end{align*}

    Hence we can separate the dynamics of~\Cref{eq:sgd} into two parts, namely $w = \Phi(w) + (w - \Phi(w))$. It is easy to check that when constrained on range of $\sharpproj$, $h(y)$ satisfies~\Cref{assum:regular}. Hence, we can use~\Cref{lem:tracksgd} to control $h(w_k - \Phi(w_k))$.
    For $\Phi(w_k)$, the iterates is running a gradient descent with learning rate $\eta$ on $\river$ and we can use proof analogous to the proof  of~\Cref{thm:gdtracksriver} to show that if $\Phi(w_k) = x(\tilde T(k, \eta))$, then there exists $T_0$, such that 
    \begin{align*}
        \tilde T(k, \eta) \in [T_0 + (1 - 4 \eta \flatgamma)\eta k, T_0 +  (1 + 4 \eta\flatgamma)\eta k].
    \end{align*}
    This completes the proof.
\end{proof}

\subsection{Proof of~\Cref{thm:decay-main}}
\label{app:decay}

We will first state the complete version of~\Cref{thm:decay-main}.

\begin{theorem}
\label{thm:decay}
Under the setting of~\Cref{thm:sgd}, the SGD iterates (defined in~\Cref{eq:sgd}) with the learning rate schedule defined in~\Cref{eq:decaylr} satisfies that for any integer $k \in [k_s, 1.1 k_s]$, there exists a $\tilde t \in [(1 - \epsilon_{t}) T(t), (1 + \epsilon_t) T(t)]$ satisfying that, 
\begin{align*}
\E[L\left( \tilde w (k)\right)] - L( x(T_0  + \tilde t))& \le  {(d - 1) {\color{red}{ \eta_k}} \sigma^2}/{2} + \epsilon_{L} 
\end{align*}
with $T(t) = T + \sum\limits_{i = k_s}^{k} \eta_i$.
\end{theorem}

\begin{proof}
    The proof is analogous to~\Cref{thm:sgd} and~\Cref{lem:tracksgd}. We will omit the detail derivation and only focus on deriving the variance of corresponding $\tilde y_k$.
    \begin{align*}
        \tilde y(k + 1) = \tilde y_k - \eta_k H \tilde y_k - \eta_k \noise_k, w(0) = 0,\mathbf{\noise_k} \sim \Normal(0, \sigma^2 \identity), \tilde y(0) = 0.
    \end{align*}
    
    Here the covariance of $\tilde y_k$, denoted as $\Sigma_k$ satisfies that
    \begin{align*}
        \Sigma_{k + 1} = (\identity - \eta_k H)^2 \Sigma_{k} + \sigma^2 \eta_k^2 \identity.
    \end{align*}

    If we consider $i$-th eigenvector of $H$ as $v_i$, and denote $\sigma_{k,i} = v_i^\top \Sigma v_i$.

    Analogous to the proof of~\Cref{thm:sgd}, $\Big | \sigma_{k_s, i}  - \frac{\eta \sigma^2}{\gamma_i} \Big | \le \frac{4 \eta^2 \sigma^2}{\gamma_i}$.

    We further have that
    \begin{align*}
        \sigma_{k_s + r + 1, i} &= (1 - \frac{\eta}{2 + r \eta \gamma} \gamma_i)^2  \sigma_{k_s + r , i} + \sigma^2 \frac{\eta^2}{(2 + r \eta \gamma)^2}.
    \end{align*}

    Then by induction, we can prove that for $r \ge 0$
    \begin{align*}
        \sigma_{k_s + r + 1, i}  \le  \frac{\sigma^2}{\gamma_i} \frac{\eta}{2 + r \eta \gamma} + \frac{4 \eta^2 \sigma^2}{\gamma_i} = \frac{\sigma^2}{\gamma_i} \eta_{t_s + r + 1} + \frac{4 \eta^2 \sigma^2}{\gamma_i}.
    \end{align*}

    The rest follows the proof of~\Cref{thm:sgd}.
\end{proof}

\subsection{Proof of~\Cref{lem:sharp,thm:toy}}

In this section, we will denote $\frac{\exp(\Theta_{i, j})}{\sum_{j = 1}^m {\exp(\Theta_{i, j})}}$ as $\mathcal{Q}_{i,j}$

\begin{lemma}
\label{lem:helpsharp}

The loss defined $L$ in~\Cref{eq:toy} satisfies that
\begin{align*}
        (\nabla L(\Theta))_{(i,j)} &= \mathcal{P}_{i,j}  - \mathcal{Q}_{i,j}.\\
    (\nabla^2 L(\Theta))_{(i,j), (i',j')} &= \one(i = i') (\mathcal{Q}_{i,j} \one(j = j') - \mathcal{Q}_{i,j} \mathcal{Q}_{i, j'}).
\end{align*}
\end{lemma}

\begin{proof}
    The loss satisfies that
    \begin{align*}
        L(\Theta) = \sum_{i = 1}^n (\sum_{j = 1}^m \mathcal{P}_{i,j} \Theta_{i,j}) - \log(\sum_{j' = 1}^m \mathcal{P}_{i,j'}).
    \end{align*}
    Hence,
    \begin{align*}
        (\nabla L(\Theta))_{(i,j)} = \mathcal{P}_{i,j}  - \mathcal{Q}_{i,j}.
    \end{align*}
    Taking differentiation for another time yields the desired result.
\end{proof}

\begin{proof}[Proof of~\Cref{lem:sharp}]
    This can be done by directly summing diagonal entries in~\Cref{lem:helpsharp}.
\end{proof}

\begin{assumption}
\label{assum:P}
    We will assume there exists constant $\gamma$ and positive integer $n' < n$ such that $\mathcal{P}$ satisfies the following assumption,
    \begin{enumerate}
        \item For any $i \le n'$, $\forall j, \mathcal{P}_{i,j} > 8\gamma$.
        \item For any $i > n'$, there exists $j_i$, $\mathcal{P}_{i,j'} > 1 - \gamma$.
    \end{enumerate}
\end{assumption}

\begin{theorem}
\label{thm:toy}
Under~\Cref{assum:P}, a generalized river with dimension $n'm + (n - n')$ exists in the loss landscape defined by $L$ in~\Cref{eq:toy}. 
\end{theorem}

\begin{proof}
    According to~\Cref{lem:helpsharp}, the Hessian for $L$ is block-diagonal. Now fixing a city $i$, we will analyze the eigenvalue distribution in this block. Let $q = [\mathcal{Q}_{i,j'}]_{j' \in [m]})$, then this block is $\mathrm{diag}(q) - qq^T$. 
    
    For all non-zero eigenvalue $\lambda$ for this block, there exists $v$ such that
    \begin{align*}
        \mathrm{diag}(q) v - q^T v q = \lambda v.
    \end{align*}

    Hence, we have that
    \begin{align*}
        v_j = \frac{q_j q^T v}{q_j - \lambda}
    \end{align*}

    This implies $\sum_{j = 1}^m \frac{q_j^2}{q_j - c} = 1$. We then have $\lambda \ge 0$ and there exists only one eigenvector corresponding to  $\lambda = 0$. For the rest nonzero eigenvalue, we have that $\lambda > \min q_i$.

    Now if we consider the manifold $\mathcal{M}$ defined as 
    \begin{align*}
        \mathcal{M} = \{ \Theta \mid \forall i \le n', \mathcal{Q}_{i,j} = \mathcal{P}_{i,j}; \forall i \ge n', \mathcal{Q}_{i,j_i} > 1 -  \gamma\}.
    \end{align*}

    Then for all $\Theta \in \mathcal{M}$, we have that the gradient is zero for all dimensions $(i, j)$ with $i \le n'$. Further, we know all the nonzero eigenvalues for these dimensions are at least $8 \gamma$ by~\Cref{assum:P}. For the rest of dimensions $(i, j)$ with $i > n'$, by~\Cref{lem:sharp}, the largest eigenvalue is bounded by $1 - (1 - \gamma)^2 < 2 \gamma$. This shows that the gradient falls in the eigenspace spanned by the last $n'm + (n - n')$ eigenvectors, which concludes the proof.
\end{proof}

\subsection{Technical Lemma}

\begin{lemma}
\label{lem:expdecay}
If a function $F(t)$ satisfies that
\begin{align*}
    \frac{d F(t)}{dt} \le -A F(t),
\end{align*}
then $F(t) \le e^{-At} F(0)$.
\end{lemma}

\begin{proof}[Proof of \Cref{lem:expdecay}]
    Consider $G(t) = F(t) e^{At}$, then
    \begin{align*}
        d G(t) = e^{At} dF(t) + A e^{At} F(t) \le 0.
    \end{align*}
    Hence $G(t) \le G(0)$.
\end{proof}

\begin{lemma}
\label{lem:normbound}
If a random vector $g \sim \Normal(0, \Sigma)$ and $\delta \in (0, 1)$, then it holds that
\begin{align*}
    \prob( \| g \|_2  \ge 2 \sqrt{ \mathrm{Tr}(\Sigma)}\sqrt{ \log(2/\delta)} )  \le \delta
\end{align*}
\end{lemma}

\begin{proof}[Proof of \Cref{lem:normbound}]
    Assume $\Sigma = Q \Lambda^2 Q^T$ with $Q$ being an orthonormal matrix and $\Lambda$ being diagonal with diagonal $\lambda_i$ for $i \in [d]$, further let $g'$ being a standard gaussian random vector, then $g$ follows the same distribution as that of $\Lambda Q^T g'$, which is further identical to $\Lambda g'$.

    \begin{align*}
        &\prob( \| g \|_2  \ge   C) \\
        =&  \prob( \| \Lambda g' \|_2  \ge   C) \\
        =& \prob( \sum_{i = 1}^d \Lambda_i^2 (g'_i)^2 \ge C^2) \\
        \le& \E[\exp( t\sum_{i = 1}^d \Lambda_i^2 (g'_i)^2 - tC^2)].
    \end{align*}

    It is well known that the moment-generating function of $(g'_i)^2$ is
    \begin{align*}
        E[\exp(t\Lambda_i^2 (g'_i)^2)] = \frac{1}{\sqrt{1 - 2t \Lambda_i^2}}.
    \end{align*}

    Hence
    \begin{align*}
        \prob( \| g \|_2  \ge   C) \le& e^{-tC^2} \prod_{i = 1}^d \frac{1}{\sqrt{1 - 2t \Lambda_i^2}} \\
        \le&  \frac{e^{-tC^2}}{\sqrt{1 - 2t \mathrm{Tr}(\Sigma)}}.
    \end{align*}

    With $t = \frac{1}{4\mathrm{Tr}(\Sigma)}$, it holds that
    \begin{align*}
    \prob( \| g \|_2  \ge   C) \le& 2 e^{-\frac{C^2}{4\mathrm{Tr}(\Sigma)}}.
    \end{align*}
    This concludes the proof.
\end{proof}

\begin{lemma}[Doob's Inequality]
\label{lem:doob}
    Let $X_1, \dots, X_n$ as a positive submartingale adapted to filtration $\mathcal{F}_1, \ldots, \mathcal{F}_n$, which means $X_i \le \E[X_{i + 1} \mid  \mathcal{F}_i]$, then 
    \begin{align*}
        \prob( \sup_{i \le n} X_i > C) \le \frac{\E[X_n]}{C}.
    \end{align*}
\end{lemma}

%% file: appendix/experiments.tex
\section{Omitted Experiments Details}
\label{app:exp}
We train LLaMA models with 4 parameter sizes using the Levanter framework for our study on $\wsds$. For our theoretical study, we pretrain a 124M GPT-2 using the nanoGPT framework with a learning rate 6e-4 and train it with a batch size of 0.5M for 100k steps with warmup steps of 2k. 

We hereby provide all the hyperparameters we used for the LLaMA and GPT-2 models training.

\begin{table}[h!]
\centering
\begin{tabular}{cccccc}
\toprule
\textbf{Model} & \textbf{Hidden Dim} & \textbf{Intermediate Dim} & \textbf{Num Layers} & \textbf{Num Heads} & \textbf{Peak LR} \\
\midrule
0.1B LLaMa & 768 & 3072 & 12 & 12 & 6e-4 \\
0.3B LLaMa & 1024 & 2048 & 24 & 16 & 6e-4 \\
0.6B LLaMa & 1536 & 6144 & 24 & 32 & 4e-4 \\
1.2B LLaMa & 2048 & 8096 & 16 & 32 & 4e-4 \\
0.1B GPT-2 & 768 & 3072 & 12 & 12 & 6e-4 \\
\bottomrule
\end{tabular}
\caption{Specifications for Different Sizes of LLaMa Models}
\label{tab:llama_models}
\end{table}

We decay the model for the last $10\%$ of the training runs with one exception for 0.3B model using $\wsd$ method near 25k steps to avoid loss spikes. We outline the decaying and resuming point (the unit is 1k steps) we choose here:

\begin{table}[h!]
\centering
\begin{tabular}{cccccc}
\toprule
\textbf{Model} & \textbf{1st Decay Starts/Resume} & \textbf{2nd Decay Starts/Resume} & \textbf{3rd  Decay Starts} \\
\midrule
0.1B LLaMa & 11.25 / 12.5 & 22.5 / 25 & 48.75/ 53.75 \\
0.3B LLaMa & 11.25 / 12.5 & 22.5 / 25 & 48.75/ 53.75 \\
0.6B LLaMa & 11.25 / 12.5 & 22.5 / 25 & 48.75/ 53.75 \\
1.2B LLaMa & 11.25 /12.5 & 22.5 / 25 & 48.75/ 53.75 \\
\bottomrule
\end{tabular}
\caption{Specifications for Decaying Steps for $\wsds$ Method}
\label{tab:llama_models_wsds}
\end{table}

\begin{table}[h!]
\centering
\begin{tabular}{cccccc}
\toprule
\textbf{Model} & \textbf{1st Decay Starts/Ends} & \textbf{2nd Decay Starts/Ends} & \textbf{3rd Decay Starts} & \textbf{Total Steps} \\
\midrule
0.1B LLaMa & 11.25 / 12.5 & 22.5 / 25 & 45/ 50 & 53.75 \\
0.3B LLaMa & 11.25 / 12.5 & 22 / 25 & 45/ 50 & 54 \\
0.6B LLaMa & 11.25 / 12.5 & 22.5 / 25 &  45/ 50  & 53.75 \\
1.2B LLaMa & 11.25 /12.5 & 22.5 / 25 &  45/ 50  &  53.75\\
\bottomrule
\end{tabular}
\caption{Specifications for Decaying Steps for $\wsd$ Method}
\label{tab:llama_models_wsd}
\end{table}